\documentclass[11pt,letterpaper]{article}

\usepackage{amssymb,amsmath,amsfonts}
\usepackage{graphicx,xcolor,enumitem}
\usepackage{epsfig}
\usepackage{amsthm}
\usepackage{bm}
\usepackage{subcaption}
\usepackage{nicematrix}
\usepackage[round]{natbib}
\usepackage{csquotes}

\usepackage{multirow}
\usepackage{algorithm, algorithmic}
\usepackage{float}
\usepackage{diagbox}

\usepackage{booktabs}
\usepackage{threeparttable}

\renewcommand{\mkbegdispquote}[2]{\itshape}
\usepackage[twoside, hmarginratio=1:1, vmarginratio=1:1, left=1in,top=1in]{geometry}

\RequirePackage[breaklinks=true, hidelinks]{hyperref}
\usepackage{breakcites}

\DeclareMathOperator*{\esssup}{ess\,sup}

\newcommand{\ang}[1]{\langle  #1 \rangle }  

\newcommand{\cA}{\mathcal{A}}
\newcommand{\cB}{\mathcal{B}}
\newcommand{\cD}{\mathcal{D}}
\newcommand{\cW}{\mathcal{W}}
\newcommand{\E}{\mathbb{E}}

\newcommand{\bF}{\mathbb{F}}

\newcommand{\R}{\mathbb{R}}
\newcommand{\p}{\mathbb{P}}

\newcommand{\cL}{{\mathcal L}}

\newcommand{\cS}{{\mathcal S}}
\newcommand{\cE}{{\mathcal E}}
\newcommand{\bD}{{\mathbb D}}
\newcommand{\cK}{{\mathcal K}}
\newcommand{\cR}{{\mathcal R}}

\newcommand{\cF}{{\mathcal F}}

\newcommand{\cP}{{\mathcal P}}
\newcommand{\cU}{{\mathcal U}}

\newcommand{\cM}{{\mathcal M}}

\newcommand{\tr}{{\rm{tr}}}

\newtheorem{theorem}{Theorem}

\newtheorem{assumption}[theorem]{Assumption}

\newtheorem{definition}[theorem]{Definition}

\newtheorem{lemma}[theorem]{Lemma}

\newtheorem{proposition}[theorem]{Proposition}

\theoremstyle{definition}

\numberwithin{equation}{section}
\numberwithin{theorem}{section}

\begin{document}

\title{Continuous-time Online Learning via Mean-Field Neural Networks: Regret Analysis in Diffusion Environments\footnote{Erhan Bayraktar is partially supported by the National Science Foundation under grant DMS-2507940 and by the Susan M. Smith chair. Bingyan Han is partially supported by the National Natural Science Foundation of China (grant 12401621). }}

\author{
	Erhan Bayraktar\thanks{Department of Mathematics, University of Michigan, Ann Arbor, Email: erhan@umich.edu.}
	\and Bingyan Han\thanks{Thrust of Financial Technology, The Hong Kong University of Science and Technology (Guangzhou), Email: bingyanhan@hkust-gz.edu.cn.}
	\and Ziqing Zhang\thanks{Department of Mathematics, University of Michigan, Ann Arbor, Email: ziqingzh@umich.edu.}
}

\date{April 13, 2026}
\maketitle

\begin{abstract}
We study continuous-time online learning where data are generated by a diffusion process with unknown coefficients. The learner employs a two-layer neural network, continuously updating its parameters in a non-anticipative manner. The mean-field limit of the learning dynamics corresponds to a stochastic Wasserstein gradient flow adapted to the data filtration. We establish regret bounds for both the mean-field limit and finite-particle system. Our analysis leverages the logarithmic Sobolev inequality, Polyak-Lojasiewicz condition, Malliavin calculus, and uniform-in-time propagation of chaos. Under displacement convexity, we obtain a constant static regret bound. In the general non-convex setting, we derive explicit linear regret bounds characterizing the effects of data variation, entropic exploration, and quadratic regularization. Finally, our simulations demonstrate the outperformance of the online approach and the impact of network width and regularization parameters.
	\\[2ex] 
	\noindent{\textbf {Keywords}: Online learning, logarithmic Sobolev inequality, Polyak-Lojasiewicz condition, Malliavin calculus, propagation of chaos.}
	\\[2ex]
	\noindent{\textbf {Mathematics Subject Classification:} } 60H30, 60J60, 49Q22, 68T07.
\end{abstract}

\section{Introduction}
	

Modern machine learning increasingly relies on massive datasets. Driven by the continuous generation of digital information, these data are inherently sequential and timestamped, a characteristic ubiquitous across fields such as engineering and finance. Concurrently, several learning frameworks, including large language models (LLMs), have scaled to require billions of parameters.

A fundamental requirement when processing temporal data is non-anticipativeness. Models must not use future information to make current decisions. Violating this principle introduces lookahead bias and artificially inflates performance. For example, \cite{sarkar2025lookahead} evaluate the task of predicting a firm's risk factors using information strictly up to November 2019. They find evidence of temporal leakage, as models like Llama systematically generate outputs related to the COVID-19 pandemic. This highlights that training large neural networks on temporal data should incorporate non-anticipativeness. To address this, \cite{he2025chronologically} and \cite{sarkar2025lookahead} train a family of models timestamped with their training data cutoff dates. However, this approach has two major drawbacks. First, training a single large model is already computationally expensive, making a family of models prohibitive. Second, model timestamps are often coarse. If a separate model is trained annually, a prediction task in August must rely on a model trained only up to the previous year, completely ignoring eight months of recent data.

Motivated by these practical challenges, we propose an online learning approach to train overparameterized neural networks on continuous-time data streams.  This method can be computationally more efficient because it does not solve the instantaneous optimization problem to exact optimality, yet it strictly maintains non-anticipativeness. We frame our problem setting to balance mathematical tractability with practical generality. First, we focus on two-layer neural networks and their mean-field limits. Although the architecture is shallow, extensive literature \citep{mei2018mean,chizat2018global} demonstrates it is sufficiently rich to provide insights into training dynamics, and we do not restrict the nonlinearity to specific architectures. Second, we assume that data arrive in continuous time, which reflects the real-time nature of modern digital data streams. More importantly, it allows us to leverage powerful tools from optimal transport and stochastic analysis to derive theoretical guarantees. Finally, while discrete-time online learning is well studied \citep{hazan2016introduction, cesa2006prediction}, its continuous-time counterpart remains far less understood, particularly for overparameterized neural networks.

Our theoretical contributions are as follows. The evolution of the parameter distribution $\rho_t$ in the mean-field limit satisfies a nonlinear Fokker-Planck equation, which can be interpreted as the Wasserstein gradient flow of the free energy functional. A key distinction in our formulation is that $\rho_t$ itself is a stochastic process adapted to the data filtration. We first establish the well-posedness of the mean-field equation and the corresponding interacting particle system. The proof adapts techniques from \cite{cattiaux2008probabilistic,monmarche2024time,conforti2023coupling}. A critical property is that $\rho_t$ satisfies the logarithmic Sobolev inequality (LSI) uniformly in time. This property enables the application of uniform-in-time propagation of chaos \citep{lacker2023sharp} and Talagrand's inequality \citep{otto2000generalization} in our subsequent analysis. A standard performance metric in online learning is regret, which compares the cumulative loss of the online learner against a suitable benchmark. We analyze both static and dynamic regret to obtain the following results.
\begin{itemize}
	\item[(1)] When the objective functional is displacement convex, Theorem \ref{thm:dc-regret} leverages the geometric properties of the Wasserstein space \citep{ambrosio2008gradient} to prove that the static regret is bounded by a constant. The proof relies on the chain rule in the Wasserstein space and the differentiability of the Wasserstein distance. While a similar approach appears in \cite{guo2022online}, Theorem \ref{thm:dc-regret} generalizes their result to the $L$-geodesically convex setting and rigorously justifies the necessary technical conditions.
	
	\item[(2)] In the general non-convex case, online-to-batch conversions imply that online learning is at least as hard as batch learning \citep{cesa2004generalization,agarwal2019learning}. Furthermore, because our data are driven by independent Brownian increments, sublinear regret cannot be expected. We therefore aim to establish explicit linear regret bounds. Theorem \ref{thm:PL-regret} derives a dynamic regret bound by combining the Polyak-Lojasiewicz (PL) condition, Talagrand's inequality, and a specialized It\^o formula established in Lemma \ref{lem:dFmu*}.
	
	\item[(3)] Bounding the static regret directly requires additional effort because the offline optimizer $\rho^*$ is anticipative. We apply Malliavin calculus in unconditional martingale difference (UMD) Banach spaces \citep{pronk2014tools} to analyze the dynamics of the objective functional at $\rho^*$. Lemma \ref{lem:Drho_exist} establishes the well-posedness of the Malliavin derivative $\cD \rho^*$, and Lemma \ref{lem:antic_Ito} proves an anticipating It\^o formula. Finally, Theorem \ref{thm:Mallivian-regret} derives an upper bound for the static regret using the PL condition, which effectively quantifies the anticipative behavior of the optimal measure $\rho^*$.	
	
	\item[(4)] While the preceding results focus on the mean-field limit, Section \ref{sec:particle-regret} analyzes the dynamic regret for the finite-particle system. Theorem \ref{thm:margin-regret} establishes a marginal-law dynamic regret bound that explodes as the entropy parameter $\beta$ approaches infinity. Consequently, although a high noise parameter $\beta$ is mathematically required to guarantee uniform-in-time propagation of chaos, it inflates the regret bound, requiring a larger number of simulated particles to maintain performance guarantees.  The proof utilizes relative entropy estimates for propagation of chaos. Notably, the requirement $\eta \beta^2 > 8 C^2_\sigma C^2_1$ from \citet[Theorem 2.1 (1)]{lacker2023sharp} parallels the condition $\alpha \beta^2 > 8 C^2_\sigma C^2_1$ required for the PL condition.
	
	\item[(5)] The previous results depend heavily on entropy regularization. Theorem \ref{thm:empircal_regret} derives the empirical dynamic regret without entropy in the objective functional, although the actual training dynamics still incorporate continuous Brownian exploration noise scaled by the temperature parameter $\beta$. To secure a meaningful bound, the $L^2$ penalty parameter $\lambda$ in Theorem \ref{thm:empircal_regret} must be sufficiently large to compensate for the error introduced by the exploration noise $\beta$.
\end{itemize}

Section \ref{sec:empirical} presents a two-part simulation study within a highly non-convex regime. First, we evaluate out-of-sample accuracy, demonstrating that the online approach significantly outperforms static optimization. Although this evaluation does not use regret as a criterion, it highlights the online method's practical suitability for time-varying environments. Second, we analyze empirical regret across key parameters. Both approximation error and regret decrease as network width increases. Furthermore, although quadratic regularization shrinks parameters toward the origin, a larger $L^2$ penalty can outpace this variance reduction and inflates regularized regret. Finally, choosing the temperature parameter $\beta$ involves a fundamental trade-off between variance and exploration.

Our work lies at the intersection of optimal transport, stochastic analysis, and online optimization. The mean-field analysis of neural networks has been investigated extensively \citep{chizat2018global,mei2018mean,sirignano2022mean}. Furthermore, \cite{gao2024global} analyze the convergence properties of gradient flows when training transformers. Mean-field online learning was studied by \cite{guo2022online}, primarily under the assumption of displacement convexity. In Euclidean spaces, the PL condition relaxes the convexity assumption while still ensuring convergence in static optimization. \cite{pun2023dist} explore discrete-time online learning under the PL condition, but they do not provide sufficient conditions for it to hold. Additionally, while Malliavin calculus is widely applied, its application in online learning remains rare. Our analysis also relies on the propagation of chaos; we refer to \cite{durmus2020elementary,lacker2023sharp,chen2025uniform} for comprehensive reviews. Finally, functional inequalities, including the LSI and Talagrand's inequality, are fundamental to our regret analysis. Other applications of these tools appear in the concentration of measure \citep{lacker2018liquidity}, dynamic games \citep{conforti2023game},  Schr\"odinger bridges \citep{conforti2020around}, and diffusion processes \citep{bartl2020functional}.

The remainder of the paper is organized as follows. Section \ref{sec:formulation} introduces the continuous-time online learning framework. Section \ref{sec:properties} establishes the well-posedness of the interacting particle system and its mean-field limit, and proves the uniform-in-time LSI. Section \ref{sec:mf-regret} derives bounds for both dynamic and static regret in the mean-field setting. Section \ref{sec:particle-regret} analyzes the dynamic regret for the finite-particle system. Section \ref{sec:empirical} presents the numerical studies. Section \ref{sec:future} concludes on future directions. All proofs are deferred to the appendix.

{\bf Notation}. For a vector or matrix $A$, denote $|A|$ as the Euclidean or Frobenius norm. For a function $X$, let $\|X\|_{L^\infty} := \esssup |X| $ be the essential supremum of the Euclidean or Frobenius norm $|X|$. For a matrix $\Sigma$, denote $\| \Sigma \|_{op} := \sqrt{\lambda_{\max} (\Sigma^\top \Sigma)}$ as the spectral norm. Let $\|\Sigma\|_{op, \infty} := \sup_{x,\theta} \| \Sigma(x, \theta) \|_{op}$ be the uniform bound on the spectral norm.

Next, we introduce differential notations. For a function $\sigma(x, \theta)$, $\nabla \sigma(x, \theta)$ means the gradient on $\theta$ only. Denote  $\Delta \sigma(x, \theta)$ as the Laplacian on $\theta$. To distinguish, we introduce $D^2_{\theta\theta} \sigma(x, \theta)$ as the Hessian on $\theta$. For cross derivatives, we use notations like $D^2_{x\theta}\sigma$ and $D^3_{xx\theta} \sigma$. To make notations compact, we also use subscripts like in $\sigma_{xx}$ to represent partial derivatives, if no ambiguity arises. The operator $\cD$ or $\cD_t$ stands for Malliavin derivatives.

Denote the set of all Borel probability measures on $\R^d$ as $\cP(\R^d)$. With $\mu \in \cP(\R^d)$, denote the $L^2(\mu)$ norm of a function $f$ as
\begin{equation*}
	\|f\|_{L^2(\mu)} := \left( \int | f(\theta) |^2 \mu(d\theta) \right)^{1/2}. 
\end{equation*}
For notational simplicity, we write $\int g(\theta) \mu(d\theta) = \ang{ \mu, g}$ afterwards. 

Consider two probability measures $\mu \in \cP(\R^d)$ and $\nu \in \cP(\R^d)$. Denote $\Pi(\mu, \nu)$ as the set of all the couplings that admit $\mu$ and $\nu$ as marginals. Let $p \in [1, \infty)$. The Wasserstein distance of order $p$ between $\mu$ and $\nu$ is defined by the formula
\begin{equation*}
	\cW_p(\mu, \nu) = \Big( \inf_{\gamma \in \Pi(\mu, \nu)} \int |\theta -  \theta'|^p \gamma(d \theta, d\theta') \Big)^{1/p}. 
\end{equation*}
Consider a map $\Phi$ from a measure space $\mathcal{X}$, equipped with a measure $\mu$, to an arbitrary space $\mathcal{Y}$. Denote by $\Phi_\# \mu$ the pushforward of $\mu$ by $\Phi$. More precisely, $(\Phi_\# \mu)[B] = \mu(\Phi^{-1}(B))$, where $\Phi^{-1}(B) := \{ x \in \mathcal{X} : \Phi(x) \in B \}$. The set of Borel probability measures with finite second moments is denoted as $\cP_2(\R^d) := \{ \mu \in \cP(\R^d) : \int |\theta|^2 d\mu < \infty \}$. If we equip it with the Wasserstein distance of order $2$, $(\cP_2(\R^d), \cW_2)$ is a metric space. Let $\cP^r(\R^d)$ denote Borel probability measures that are absolutely continuous with respect to the Lebesgue measure on $\R^d$. If $\rho \in \cP^r(\R^d)$, we still denote its density as $\rho$. Let $\cP^r_2(\R^d) = \cP^r(\R^d) \cap \cP_2(\R^d)$.

\section{Formulation}\label{sec:formulation}
We work on a product probability space that separates the randomness of the data stream from that of the learning dynamics. Let $(\Omega^W, \mathcal{F}^W, \mathbb{P}^W)$ and $(\Omega^B, \mathcal{F}^B, \mathbb{P}^B)$ be two complete probability spaces supporting, respectively, the data noise and the particle noise. We define the working probability space as
\[
(\Omega, \mathcal{F}, \mathbb{P}) := (\Omega^W \times \Omega^B, \mathcal{F}^W \otimes \mathcal{F}^B, \mathbb{P}^W \otimes \mathbb{P}^B).
\]

On $(\Omega^W, \mathcal{F}^W, \mathbb{P}^W)$, we construct an $m$-dimensional standard Brownian motion $W$ with its augmented natural filtration $\mathbb{F}^W = (\mathcal{F}^W_t)_{t \ge 0}$. On $(\Omega^B, \mathcal{F}^B, \mathbb{P}^B)$, we construct a sequence of independent $d$-dimensional standard Brownian motions $(B^i)_{i \ge 1}$ with filtration $\mathbb{F}^B = (\mathcal{F}^B_t)_{t \ge 0}$. 

All sources of randomness are assumed independent, and expectations with respect to $\mathbb{P}^W$ and $\mathbb{P}^B$ are denoted by $\mathbb{E}$ and $\mathbb{E}_B$, respectively.

Suppose the agent observes a data stream $(X_t, Y_t) =: Z_t$ modeled by the following dynamic:
\begin{equation}\label{eq:data}
	\begin{aligned}
		d X_t & = b_1 (t, \omega^W) dt + \Sigma_1(t, \omega^W) dW_t, \\
		d Y_t & = b_2 (t, \omega^W) dt + \Sigma_2(t, \omega^W) dW_t.
	\end{aligned}
\end{equation} 
Here, $X_t \in \mathbb{R}^n$ is the feature process and $Y_t \in \mathbb{R}$ is the response. The coefficients $b_1, b_2, \Sigma_1$, and $\Sigma_2$ are $\mathbb{F}^W$-progressively measurable and satisfy standard conditions ensuring existence and uniqueness of a strong solution. The process $Z_t = (X_t, Y_t)$ is therefore an exogenous diffusion, independent of the learning dynamics. 

Importantly, the learner does not know the coefficients of the data-generating process and only observes the trajectory $(Z_t)_{t \ge 0}$ in a non-anticipative manner. We note that, with the quadratic variation, the learner can only recover the covariance $\Sigma_i \Sigma^\top_j$, but not $\Sigma_i$ itself. The class of admissible learning rules is restricted to processes that are adapted to the natural observation filtration $\mathbb{F}^Z := (\mathcal{F}^Z_t)_{t \ge 0}$, where $\cF^Z_t := \sigma(Z_s : 0 \le s \le t)$.

In many optimal control and financial problems, the data stream $\{(X_t, Y_t)\}_{t \in [0, T]}$ is revealed sequentially, rendering traditional batch optimization impossible. Specifically, the learner's estimates must be entirely non-anticipative, relying solely on the available information filtration up to time $t$. An online learning framework becomes attractive since it adheres to causality constraints. Formally, all admissible learning rules must be adapted to the filtration $\mathbb{F}^Z$, reflecting the causality constraint inherent in the online setting.

\subsection{Two-layer neural networks}
To construct an online estimator, we deploy a two-layer neural network of finite width $N$:
\begin{equation}\label{eq:finiteNN}
	X_t \mapsto \frac{1}{N}\sum_{j=1}^N \sigma(X_t, \theta_t^{j,N}) =: \int_{\R^d}\sigma(X_t, \theta)\,\hat\rho_t^N(d\theta),
\end{equation}
where \(\sigma:\R^n\times\R^d\to\R\) is a non-linear feature map, \(\theta_t^{j,N}\in\R^d\) is the parameter of neuron \(j\), and \(\hat\rho_t^N\) is the empirical measure of parameters given by
\begin{equation}\label{eq:empirical}
	\hat\rho_t^N := \frac{1}{N}\sum_{j=1}^N \delta_{\theta_t^{j,N}}\,.
\end{equation}
At time $t$, the learner faces a loss with two components. The first is a mean squared error (MSE) that penalizes the discrepancy between the network's estimation \eqref{eq:finiteNN} and the true target $Y_t$. The second is an $L^2$ regularization term weighted by $\lambda > 0$. The loss at time $t$ is given by
\begin{equation*}
	\tilde{U}(\hat{\rho}^N_t, Z_t) := \left(\int_{\R^d}\sigma(X_t,\theta)\,\hat\rho_t^N(d\theta)-Y_t\right)^2
	+ \frac{\lambda}{2}\int_{\R^d}|\theta|^2\,\hat\rho_t^N(d\theta).
\end{equation*}
The $L^2$ penalty is motivated by both statistical learning and stochastic analysis. From a learning perspective, it acts as standard weight decay, penalizing excessively large parameter values to prevent overfitting and improve generalization. From a probabilistic perspective, this quadratic penalty translates into a strong linear mean-reverting drift in the underlying continuous-time gradient flow. This confinement force provides the essential dissipativity required to prevent the interacting particles from escaping to infinity. It ensures the stability of the particle system and is indispensable for obtaining regret bounds.

Clearly, we can expand the quadratic term and notice that $Y^2_t$ does not interact with $\hat{\rho}^N_t$. Therefore, we simply consider the cost functional as
\begin{equation}\label{eq:U-hatN}
	U(\hat{\rho}^N_t, Z_t) := \left(\int_{\R^d}\sigma(X_t,\theta)\,\hat\rho_t^N(d\theta) \right)^2 - 2 Y_t \left(\int_{\R^d}\sigma(X_t,\theta)\,\hat\rho_t^N(d\theta) \right)
	+ \frac{\lambda}{2}\int_{\R^d}|\theta|^2\,\hat\rho_t^N(d\theta).
\end{equation}

Because the data stream $Z_t$ evolves continuously in a diffusion environment, we track the optimal parameter distribution dynamically by a continuous-time gradient flow that minimizes the instantaneous cost \eqref{eq:U-hatN}. The gradient of \eqref{eq:U-hatN} with respect to a single particle $\theta^{i,N}_t$ has a linear confinement drift $-\lambda \theta^{i,N}_t$ and a non-linear term representing the empirical error signal. Since deterministic gradient flow often suffers from poor generalization, we inject independent noises and update $\theta^{i, N}_t$ as follows:
\begin{equation}\label{eq:particleSDE}
	d\theta^{i, N}_t = \Big[  - \lambda \theta^{i, N}_t - 2 \nabla \sigma(X_t, \theta^{i, N}_t) \Big( \frac{1}{N-1} \sum^N_{j=1, j\neq i} \sigma(X_t, \theta^{j, N}_t) - Y_t \Big) \Big] dt + \sqrt{2 \beta} d B^i_t, \, i = 1, \ldots, N.
\end{equation}
We refer to it as an online noisy particle gradient descent (ONPGD) method. Here, \((B^i)_{i=1}^N\) are independent Brownian motions, which are also independent of the data noise \(W\). It ensures continuous exploration of the parameter space. Finally, when computing the empirical error signal for the $i$-th particle, we deliberately decouple the particle from its own influence by removing the self-interaction term. It is a standard convention in the propagation of chaos literature \citep{lacker2023sharp}.

\subsection{Mean-field neural networks}
It is challenging to analyze the coupled ONPGD system \eqref{eq:particleSDE} with finite particles directly. Instead, we consider the overparameterized regime first. When the width of neural networks \(N\to\infty\), the estimator \eqref{eq:finiteNN} converges to
\begin{equation}\label{eq:mf-predictor}
	X_t \mapsto \int_{\R^d}\sigma(X_t, \theta) \,\rho_t(d \theta),
\end{equation}
where $\rho_t$ is the limit of the empirical measure \(\hat\rho_t^N\). Consequently, the cost functional \eqref{eq:U-hatN} with finite particles converges to the mean-field objective:
\begin{equation}\label{eq:U-def}
	U(\rho_t,Z_t)
	:=
	\left(\int_{\R^d}\sigma(X_t,\theta)\,\rho_t(d\theta)\right)^2
	-2Y_t\int_{\R^d}\sigma(X_t,\theta)\,\rho_t(d\theta)
	+\frac{\lambda}{2}\int_{\R^d}|\theta|^2\,\rho_t(d\theta).
\end{equation}
The trajectory of a single, representative neuron parameter is governed by the following McKean-Vlasov stochastic differential equation (SDE):
\begin{equation}\label{eq:mean-fieldSDE}
	d\theta_t = \left[ - \lambda \theta_t - 2 \nabla \sigma(X_t, \theta_t) \Big( \int \sigma(X_t, \vartheta) \rho_t(d\vartheta) - Y_t \Big) \right] dt + \sqrt{2 \beta} d B_t.
\end{equation}
The Brownian motion $B_t$ is independent of $W_t$ driving the data. Crucially, the mean-field distribution $\rho_t$ itself is random and adapted to the data filtration $\cF^Z_t$. Moreover, as a curve of probability measures, $\rho_t$ is the gradient flow of a regularized free energy with respect to the Wasserstein distance $\cW_2$; see \citet[Definition 4.2.6]{figalli2021invitation} for the definition of gradient flow in the Wasserstein space. The diffusion term in \eqref{eq:mean-fieldSDE} is equivalent to impose the entropy regularization in \eqref{eq:U-def}. Then the free energy functional is 
\begin{equation}\label{eq:F-def}
	F(\rho_t,Z_t)
	:=
	U(\rho_t,Z_t)
	+ \beta\int_{\R^d}\rho_t(\theta)\log\rho_t(\theta)\,d\theta.
\end{equation}

The evolution of $\rho_t$ corresponds to the steepest descent of $F(\cdot, Z_t)$ with respect to the Wasserstein metric. The associated nonlinear Fokker--Planck equation is the continuity equation driven by the gradient of the first variation $\frac{\delta F}{\delta \rho}$:
\begin{equation}\label{eq:FK-continuity}
	\begin{aligned}
		\partial_t \rho_t & = \nabla \cdot \left( \rho_t \nabla \left(\frac{\delta F(\rho_t, Z_t)}{ \delta \rho}\right) \right) = \nabla \cdot \left( \beta \nabla \rho_t + \rho_t \nabla\left(\frac{\delta U(\rho_t, Z_t)}{\delta \rho_t}\right) \right) \\
		& = \beta \Delta \rho_t + \nabla \cdot \left( \left[ \lambda \theta + 2 \left( \int \sigma(X_t, \vartheta)\rho_t(d\vartheta) - Y_t \right) \nabla \sigma(X_t, \theta) \right] \rho_t \right).
	\end{aligned}
\end{equation} 

The evolution equation \eqref{eq:FK-continuity} admits a natural interpretation as a continuous-time limit of online gradient descent under a time-varying objective. To make this precise, consider first a frozen-time problem where the data is fixed at some $Z_t$. In this case, minimizing the free energy $F(\cdot, Z_t)$ over probability measures leads to the classical Wasserstein gradient flow
\[
\partial_s \mu_s = \nabla \cdot \left( \mu_s \nabla \frac{\delta F(\mu_s, Z_t)}{\delta \rho} \right),
\]
where $s \ge 0$ denotes an artificial algorithmic time. This flow is well known to converge, under suitable conditions, to the equilibrium measure solving the static problem \eqref{prob:dynamic_t} defined later.

In contrast, the actual learning dynamic \eqref{eq:FK-continuity} evolves in the physical time $t$, while the objective itself depends on the streaming data $Z_t$. Therefore, the system does not minimize a fixed functional, but instead performs an infinitesimal gradient step at each time $t$ with respect to the instantaneous objective $F(\cdot, Z_t)$. In this sense, \eqref{eq:FK-continuity} can be viewed as a Wasserstein gradient flow with a time-dependent and stochastic energy landscape.

This interpretation highlights a key distinction with classical mean-field training: the measure-valued process $(\rho_t)_{t \ge 0}$ tracks a moving sequence of minimizers rather than converging to a single equilibrium. Consequently, the analysis naturally falls within the framework of online optimization and regret minimization, rather than asymptotic convergence to a stationary solution.

The ONPGD dynamic \eqref{eq:particleSDE} can be regarded as the $N$-particle system approximating the McKean-Vlasov SDE \eqref{eq:mean-fieldSDE}.  By shifting from the $N$-particle system to the Wasserstein gradient flow, we can utilize the geometric properties of Wasserstein spaces \citep{ambrosio2008gradient,figalli2021invitation} to establish regret bounds for continuous-time online learning.

\subsection{Regret}
To evaluate the performance of our online learning algorithm, we introduce two optimizers. In contrast to the online setting where the learner adapts to the data filtration progressively, the classical offline optimization framework assumes full access to the entire data trajectory $\{Z_t\}_{t \in [0,T]}$ in advance. The offline learner seeks a single parameter distribution $\rho$ that minimizes the aggregate cost over the entire horizon $[0, T]$: 
\begin{equation}\label{eq:statobj}
	\inf_{\rho} \int^T_0 F(\rho, Z_t) dt.
\end{equation}
Denote the optimizer of \eqref{eq:statobj} as $\rho^*$, which exists and is unique, shown in Lemma \ref{lem:rho*} later. It represents the best fixed decision in hindsight. Consequently, it is not adapted to the natural data filtration prior to time $T$. Comparing the online $\rho_t$ against this fixed hindsight optimizer yields our first performance criterion.
\begin{definition}\label{def:static-mf}
	The average static regret is given by
	\begin{equation}
		\begin{aligned}
			\cR_S(T) & := \E \left[ \int^T_0 F(\rho_t, Z_t) dt - \inf_{\rho} \int^T_0 F(\rho, Z_t) dt \right] \\
			& = \E \left[ \int^T_0 F(\rho_t, Z_t) dt - \int^T_0 F(\rho^*, Z_t) dt \right],
		\end{aligned}
	\end{equation}
	where the expectation is taken over the data trajectory $\{Z_t\}_{t \in [0, T]}$.
\end{definition}

While static regret is the standard criterion in classical online learning, measuring performance against a single, fixed measure $\rho^*$ is often a weak benchmark in highly non-stationary diffusion environments. To capture the algorithm's ability in a shifting environment, we introduce the concept of dynamic regret. Instead of aggregating the cost over $[0, T]$, we freeze the continuous data stream at a specific instant $t$ and consider the instantaneous optimization problem:
\begin{equation}\label{prob:dynamic_t}
	\inf_{\rho} F(\rho, Z_t).
\end{equation}

This frozen-time problem pauses the physical time $t$ and allows the parameters to evolve over a fictitious algorithmic time $s \ge 0$. Under a noisy gradient descent dynamic with the data fixed at $Z_t$, the parameter distribution follows the Wasserstein gradient flow $\mu_s$ given by
\begin{equation}\label{eq:equilibrium}
\begin{aligned}
	\partial_s \mu_s & = \nabla \cdot \left( \mu_s \nabla \left(\frac{\delta F(\mu_s, Z_t)}{ \delta \rho}\right) \right) \\
	& = \beta \Delta \mu_s + \nabla \cdot \left( \left[ \lambda \theta + 2 \left( \int \sigma(X_t, \vartheta)\mu_s(d\vartheta) - Y_t \right) \nabla \sigma(X_t, \theta) \right] \mu_s \right).
\end{aligned}
\end{equation} 

Under certain conditions, this fictitious-time gradient flow $\mu_s$ converges to a unique equilibrium measure as $s \to \infty$, which we denote as $\mu^*_t$. We will subsequently prove that $\mu^*_t$ is well-defined, unique, and achieves the infimum of Problem \eqref{prob:dynamic_t}. Although $\mu^*_t$ is also adapted to the data filtration, it is expensive to implement since the learner has to train the neural network in a long run for each instantaneous time $t$. Comparing the online $\rho_t$ against this moving sequence of optimal measures defines our second performance criterion.
\begin{definition}\label{def:dynamic-mf}
	The average dynamic regret is given by
	\begin{equation}
		\cR_{D}(T) := \E \left[ \int^T_0 F(\rho_t, Z_t) dt - \int^T_0 F(\mu^*_t, Z_t) dt \right],
	\end{equation}
	where the expectation is taken over the data trajectory $\{Z_t\}_{t \in [0, T]}$.
\end{definition}
Clearly, $\cR_S(T) \leq \cR_D(T)$. We also find that the dynamic regret is more analytically tractable. 

Definitions \ref{def:static-mf} and \ref{def:dynamic-mf} deal with mean-field distributions $\rho_t$. We refer to them as {\it mean-field regrets}. To study the regret for neural networks with finite width, we introduce two notions of {\it particle regret}, utilizing the instantaneous equilibrium measure $\mu^*_t$ as our dynamic benchmark. 

The distribution of each parameter $\theta^{i, N}_t$ in \eqref{eq:particleSDE} is absolutely continuous. We denote its probability density function as $\rho^{i, N}_t$. From a theoretical perspective, the expected prediction of the neural network can be evaluated through the average marginal law of its constituent neurons. We evaluate the entropy-regularized free energy of this averaged system as
\begin{align*}
	F(\pi^N_t, Z_t) := & \left( \frac{1}{N} \sum^N_{i=1} \int \sigma(X_t, \theta) \rho^{i, N}_t(d\theta) \right)^2 - 2 Y_t \left( \frac{1}{N} \sum^N_{i=1} \int \sigma(X_t, \theta) \rho^{i, N}_t(d\theta) \right) \\
	&  + \frac{\lambda}{2} \left( \frac{1}{N} \sum^N_{i=1} \int |\theta|^2 \rho^{i, N}_t(d\theta) \right) + \beta \left( \frac{1}{N} \sum^N_{i=1} \int \rho^{i, N}_t \log \rho^{i, N}_t \right).
\end{align*}
Here, $\pi^N_t$ denotes of the empirical measure of $N$ elements given by $\rho^{i, N}_t$. Assume the initial conditions are the same. Since the particle dynamics \eqref{eq:particleSDE} are invariant under permutation, the particles are exchangeable. Consequently, the marginal laws are identical, and the functional simplifies to evaluating a single representative neuron:
\begin{equation*}
	F(\pi^N_t, Z_t) = F(\rho^{1, N}_t, Z_t).
\end{equation*}

\begin{definition}
	The marginal-law dynamic regret is given by
	\begin{equation}\label{eq:regret-marginal}
		\cR_{\mathrm{law}}(T;N)
		:=
		\E\left[
		\int_0^T \Big(F(\rho_t^{1,N},Z_t)-F(\mu_t^*,Z_t)\Big)\,dt
		\right].
	\end{equation}
\end{definition}
This formulation measures the discrepancy between probability densities and excludes the empirical sampling errors. In practice, the true density $\rho^{i,N}_t$ is unobservable. A practitioner deploying the network only has access to pathwise realizations of the parameters, represented by the discrete empirical measure $\hat\rho_t^N$. To capture the true performance of the deployed algorithm, we must measure the regret directly on this empirical realization. 

It is ill-posed to evaluate the continuous entropy on discrete point masses. While one could employ estimation methods such as kernel density estimation to approximate the entropy, doing so introduces a separate, highly non-trivial statistical problem. Therefore, to isolate the core tracking capabilities of the finite network, we exclude the entropic term from the objective, focusing exclusively on the cost $U$.

\begin{definition}
	The empirical dynamic regret without entropy is given by
	\begin{equation}\label{eq:regret-empirical}
		\cR_{\mathrm{em}}(T;N)
		:=
		\E\left[ \E_{B} \left[ 
		\int_0^T \Big(U(\hat\rho_t^N,Z_t)-U(\mu_t^*,Z_t)\Big)\,dt \right] 
		\right],
	\end{equation}
	where \(\hat\rho_t^N\) is the empirical measure \eqref{eq:empirical} and $\E_{B}$ denotes the expectation over the particle noise.
\end{definition}
Since the dynamic regret bounds the static counterpart, we only consider the dynamic version of particle regrets. It is straightforward to define the analogous static particle regrets by replacing $\mu^*_t$ with the static hindsight benchmark $\rho^*$; these are omitted here for brevity.

\section{Properties of the mean-field and particle systems}\label{sec:properties}

In this section, we establish the well-posedness, moments bounds, and other useful properties of our mean-field SDE \eqref{eq:mean-fieldSDE} and the corresponding $N$-particle system \eqref{eq:particleSDE}.

There is a structural difference between our interaction term and the classical one in \cite{cattiaux2008probabilistic}. In \cite{cattiaux2008probabilistic}, the interaction potential is assumed to be convex at infinity. This geometric property ensures the interaction force is dissipative. It tends to reduce the distance between particles or is neutral. Our mean-field interaction term originates from the gradient flow of a neural network loss landscape. Rather than a physical force, it represents a shared empirical error signal. This interaction could be expansive, which pushes particles away and increases the energy. To ensure the stability, we  assume the interaction term is bounded and impose the following two standing assumptions.
\begin{assumption}\label{A:sigma}
	Suppose $\sigma$ and its derivatives below are uniformly bounded:
	\begin{align*}
		& \sup_{x \in \R^n, \theta \in \R^d} |\sigma(x, \theta)| \le C_\sigma, \quad \sup_{x \in \R^n, \theta \in \R^d} | D_\theta \sigma(x, \theta) | \le C_{1}, \quad \sup_{x \in \R^n, \theta \in \R^d} | D_x \sigma(x, \theta) | \le C_{1}, \\
		& \sup_{x \in \R^n, \theta \in \R^d} | D^2_{x\theta} \sigma(x, \theta) | \leq C_2, \quad \sup_{x \in \R^n, \theta \in \R^d} | D^2_{xx} \sigma(x, \theta) | \leq C_2, \quad \sup_{x \in \R^n} | D^2_{\theta \theta} \sigma(x, \theta) | \leq \frac{C_2}{1 + |\theta|},\\
		& \sup_{x \in \R^n, \theta \in \R^d} | D^3_{xx\theta} \sigma(x, \theta) | \leq C_3, \quad \sup_{x \in \R^n} | D^3_{\theta\theta\theta} \sigma(x, \theta) | \leq \frac{C_3}{1 + |\theta|}.
	\end{align*}
\end{assumption}
The third derivative conditions are used in the dynamic regret bounds in Theorem \ref{thm:PL-regret}, to deal with the Hessian terms from It\^o's lemma. It is also needed in Theorem \ref{thm:Lip-grad} to obtain the linear growth rate of $\nabla \log \rho_t(\theta)$ in \eqref{eq:logrho}. 

\begin{assumption}\label{A:data}
	 The drift $b_i$ and diffusion $\Sigma_i$, $i= 1, 2$ in the data process are $\bF^W$-progressively measurable. There exists a constant $C_{b, \Sigma} > 0$, such that, for any $t \geq 0$, we have $\E[ | b_i(t) |^2 ]^{1/2}$ and $\E[| \Sigma_i \Sigma^\top_i (t)|^2]^{1/2}$, $i= 1, 2$ bounded by $C_{b, \Sigma}$. 
	
	The initial value $X_0, Y_0 \in \cF^W_0$ and $\E[|X^2_0|] < \infty$. The SDE \eqref{eq:data} for data has a unique strong solution and $\sup_{t \geq 0} \|Y_t\|_{L^\infty} \leq C_z$.
\end{assumption}
We note that $\E\left[ \sup_{t \in [0, T]} |X_t|^2 \right] < \infty$ under Assumption \ref{A:data}. In this paper, we will denote a generic constant $C_* > 0$ that depends on some or all of $C_\sigma, C_z, C_1, C_2$, and $C_3$ only, which may vary from line by line.

\subsection{Well-posedness}
Fix a continuous data path $\{ (X_t, Y_t) \}_{t \geq 0}$. Proposition \ref{prop:MkV} obtains the well-posedness of the mean-field density $\rho_t$. The proof adopts a fixed-point argument. The third property on the stability with respect to mollification is mainly to prepare the proof of Theorem \ref{thm:Lip-grad}.
\begin{proposition}\label{prop:MkV}
	We have the following results:
	\begin{enumerate}
		\item For the initial value
			$\rho_0 \in L^1(\R^d) \cap L^\infty(\R^d) \cap \cP(\R^d)$,
			there exists a unique solution to the mean-field flow  \eqref{eq:FK-continuity}
			in $C([0,\infty); L^1(\R^d) \cap \cP(\R^d))
			\cap L^\infty([0,\infty); L^\infty(\R^d))$
			depending continuously on the initial value. The following time-uniform bound holds:
			\begin{equation}
				\label{eq:MkV-bound}
				\sup_{t \in [0, \infty)}
				\| \rho_t \|_{L^\infty}
				\leqslant C < \infty,
			\end{equation}
            where $C$ depends on $\| \rho_0 \|_{L^\infty}$.
			\item If the initial value $\rho_0$ also has finite $k$-th moment for some $k > 0$, then the mean-field flow $\rho_t$ has finite $k$-th moment, uniformly in time:
			\begin{equation}
				\sup_{t \in [0,\infty)}\int_{\R^d} | \theta |^k \rho_t(d\theta)
				\leq C < \infty,
			\end{equation}
			where $C$ depends on $k$, $\| \rho_0 \|_{L^\infty}$, and $\| \rho_0 \|_{L^k}$.
			
			\item Let $\sigma^\varepsilon = \sigma * \eta^\varepsilon$ be the mollified approximation of $\sigma$. Consider the initial value $\rho^\varepsilon_0$ that converges to $\rho_0$ in $L^1$
			and satisfies $\sup_\varepsilon \| \rho^\varepsilon_0 \|_{L^\infty} < \infty$. Then the solution $\rho^\varepsilon_t$ of the approximate mean-field flow
			\begin{equation}\label{eq:MkV-approx}
				\partial_t \rho^\varepsilon_t = \beta \Delta \rho^\varepsilon_t + \nabla \cdot \left( \left[ \lambda \theta + 2 \left( \ang{\rho^\varepsilon_t, \sigma^\varepsilon(X_t, \cdot)} - Y_t \right) \nabla \sigma^\varepsilon(X_t, \theta) \right] \rho^\varepsilon_t \right)
			\end{equation}
			converges to $\rho_t$ in $L^1$ for all $t \geq 0$. Furthermore, the $L^\infty$ norm and the $k$-th moment bounds above hold
			when we replace $\rho$ by $\rho^\varepsilon$.
		\end{enumerate}
	\end{proposition}
	
Like \cite{cattiaux2008probabilistic}, we can show the well-posedness of the particle system in Proposition \ref{prop:particle}. Although our interaction term lacks the symmetric convolutional structure in \cite{cattiaux2008probabilistic}, we can exploit the uniform boundedness of $\sigma$ and $\nabla \sigma$. Consequently, the interaction is bounded and the linear confinement drift $-\lambda \theta$ dominates the interaction at large distances, guaranteeing the well-posedness and uniform-in-time moment bounds. The proof is similar to \cite{cattiaux2008probabilistic} and thus omitted. 
\begin{proposition}\label{prop:particle}
		Suppose $\theta^{i, N}_0$ follows the same distribution $v_0$ with finite second moments. Then the particle system \eqref{eq:particleSDE} has a global unique strong solution. Furthermore, there exists some $C > 0$ such that for all $i$,
		\begin{equation}
			\sup_{t \geq 0} \E[ |\theta^{i, N}_t|^2] \leq \int_{\R^d} |\theta|^2 v_0(d\theta) + C.
		\end{equation}
		In general, if $v_0$ has finite $2k$-th moments with $k \geq 1$, then there exists a constant $C(k)$ such that for all $i$,
		\begin{equation}
			\sup_{t \geq 0} \E[ |\theta^{i, N}_t|^{2k}] \leq C(k) \Big(1 + \int_{\R^d} |\theta|^{2k} v_0(d\theta) \Big).
		\end{equation}		
\end{proposition}
	
Next, we study the gradient bounds that are used in particle regrets, see Theorem \ref{thm:margin-regret}. We follow the Hamilton-Jacobi-Bellman (HJB) equation approach in \cite{conforti2023coupling,chen2025uniform,monmarche2024time}. By applying a logarithmic transformation to the probability density, we recast the Fokker-Planck PDE into a semilinear HJB equation. This transformation reinterprets the evolution of the density as the value function of an underlying stochastic optimal control problem. Then we use the reflection coupling technique in \cite{eberle2016reflection} and the stability results on mollification in Proposition \ref{prop:MkV} to prove the first claim in Theorem \ref{thm:Lip-grad}. Importantly, the gradient $\nabla u_t$ in \eqref{eq:dud2u} is in the classical sense, while the Hessian $D^2 u_t$ is in the weak sense.

The second and third properties in Theorem \ref{thm:Lip-grad} follow similarly from \citet[Proposition 4.6]{chen2025uniform}. The $AC^2$ property of the flow, written as $(\rho_t)_{t \geq 0} \in AC^2_{loc} ([0, \infty); (\cP_2, \cW_2))$, is defined in \cite{ambrosio2008gradient} and \citet[Definition 4.5]{chen2025uniform}. The third property here ensures that we can apply the chain rule in \citet[Proposition 10.3.18]{ambrosio2008gradient} to well-behaved functionals.

Define the invariant probability density $\gamma_0 \propto \exp(-\frac{\lambda}{2\beta} |\theta|^2)$.
\begin{theorem}\label{thm:Lip-grad}
	The initial density $\rho_0$ is chosen such that $ - \log \frac{\rho_0}{ \gamma_0}$ is Lipschitz continuous. Let $(\rho_t)_{t \geq 0}$ be the solution to \eqref{eq:FK-continuity}. Denote $u_t := - \log \frac{\rho_t}{\gamma_0}$.
		Then we have the following results:
		\begin{enumerate}
			\item For all $t > 0$,
			\begin{equation}\label{eq:dud2u}
				\| \nabla u_t \|_{L^\infty} \leq C, \quad \| D^2 u_t \|_{L^\infty} \leq \frac{C}{\sqrt{t \wedge 1}},
			\end{equation}
			for some $C$ depending on
			$\| \rho_0 \|_{L^\infty}$ and $\| \nabla u_0 \|_{L^\infty}$, but independent of time $t$. Hence, there exists $L_\rho > 0$, such that, for all $t \geq 0$,
			\begin{equation}\label{eq:logrho}
				| \nabla \log \rho_t(\theta)| \leq L_\rho (1 + | \theta |).
			\end{equation}
			The constant $L_\rho$ depends at least on $\beta$ and $\lambda$.
			\item Furthermore, suppose $\rho_0 \in L^1(\R^d) \cap L^\infty(\R^d) \cap \cP(\R^d)$ and $\rho_0$ has finite second moment. Then the solution $\rho_t$ has finite entropy when $t > 0$:
			\begin{equation}
				\int | \log \rho_t(\theta) | \rho_t(\theta) d\theta < \infty. 
			\end{equation}
			
			\item If additionally $\rho_0$ has finite entropy, then 
			\begin{equation}
				\int^t_0 \int \frac{| \nabla \rho_s(\theta)|^2}{\rho_s(\theta)} d\theta ds < \infty, \text{ for every } t \geq 0. 
			\end{equation}
			Hence, $(\rho_t)_{t \geq 0} \in AC^2_{loc} ([0, \infty); (\cP_2, \cW_2))$ and has tangent vector given by
			\begin{equation}
				v_t(\theta) := - \beta \nabla \log \rho_t - \nabla\left(\frac{\delta U(\rho_t, Z_t)}{\delta \rho_t} \right),
			\end{equation}
			for $t \geq 0$ a.e. in the sense of \citet[Proposition 8.4.5]{ambrosio2008gradient}.
		\end{enumerate}	
\end{theorem}

To simplify the subsequent analysis and eliminate initialization errors, we impose the following standing assumption.
\begin{assumption}
	Suppose the initial distribution $\rho_0$ of the mean-field SDE \eqref{eq:mean-fieldSDE} and the initial distribution $\rho^{i, N}_0$, $i=1, \ldots, N$, of the $N$-particle system are the same and equal to the invariant probability density $\gamma_0 \propto \exp(-\frac{\lambda}{2\beta} |\theta|^2)$.
\end{assumption}

\subsection{Logarithmic Sobolev inequality}
A crucial notion that enables us to obtain the Polyak-Lojasiewicz (PL) condition and apply propagation of chaos results is the logarithmic Sobolev inequality (LSI).
\begin{definition}
	A probability measure $\nu$ on $\R^d$ satisfies the logarithmic Sobolev inequality with constant $\alpha > 0$ if for all functions $f: \R^d \to \R$,
	\begin{equation}\label{eq:LSIfdef}
		\int f^2 \log(f^2) d\nu - \int f^2 d\nu \log\left( \int f^2 d\nu \right) \leq \frac{2}{\alpha} \int | \nabla f |^2 d \nu.
	\end{equation}
	In short, we say $\nu$ satisfies the $LSI(\alpha)$.
\end{definition}
The LSI property \eqref{eq:LSIfdef} has another compact form. For absolutely continuous probability measures $\mu$ and $\nu$, we define the relative Fisher information by
\begin{equation}
	I(\mu | \nu) := \int \Big| \nabla \log \frac{d\mu}{d\nu} \Big|^2 d\mu = 4 \int \Big|  \nabla \sqrt{\frac{d\mu}{d\nu}} \Big|^2 d\nu,
\end{equation}
and the relative entropy by
\begin{equation}
	D_{KL}(\mu \| \nu) := \int \log \frac{d\mu}{d\nu} d\mu = \int \frac{d\mu}{d\nu} \log \frac{d\mu}{d\nu} d\nu.
\end{equation}
For all probability measure $\mu$ absolutely continuous with respect to (w.r.t.) $\nu$, we let $f = \frac{d\mu}{d\nu}$ and apply the condition \eqref{eq:LSIfdef} to $\sqrt{f}$. Then the probability measure $\nu$ satisfies the $LSI(\alpha)$ if
\begin{equation}\label{eq:LSI_entropy}
	D_{KL}(\mu \| \nu) \leq \frac{1}{2\alpha} I(\mu | \nu), \quad \text{for all $\mu$ absolutely continuous w.r.t $\nu$.}
\end{equation}
It is \citet[Definition 1]{otto2000generalization}. Intuitively, the LSI bounds the spread of a measure relative to a target $\nu$ measured by the KL divergence, by its local variation measured by the Fisher information. The LSI is closely related to the PL condition in the Wasserstein space, as demonstrated in Lemma \ref{lem:PL} later. Furthermore, the LSI implies the Talagrand's inequality \citep{otto2000generalization}, which allows us to bound the $\cW_2$ distance between measures in terms of their KL divergence.

Thanks to \cite{monmarche2024time}, the online density $\rho_t$ satisfies the LSI uniformly in time $t$.
\begin{lemma}\label{lem:rhotLSI}
	 There exists $\eta > 0$ such that $\rho_t$ satisfies the $LSI(\eta)$ for all $t \geq 0$. The constant $\eta$ depends on $\beta$, $\lambda$, $C_\sigma$, $C_z$, $C_1$, and $C_2$. If $\beta \geq 1$, the constant $\eta$ can be chosen as $\tilde{\eta}/\beta$ where $\tilde{\eta}$ is independent of $\beta$. 
\end{lemma}
In the proof of Lemma \ref{lem:rhotLSI}, we also find that when $\beta$ is large, $\tilde{\eta}$ becomes close to $\lambda$, since the transformation $\tilde{\theta}_t = \theta_t/\sqrt{\beta}$ shows that the nonlinear terms from $\sigma$ are vanishing.

\section{Regret bounds for the mean-field case}\label{sec:mf-regret} 

In this section, we obtain mean-field regrets bounds with three different methods, the displacement convexity method, the PL method, and the Malliavin calculus approach. The displacement convexity method is the most direct way and yields a constant regret if the functional is displacement convex. This idea has been exploited in \cite{guo2022online}. However, the displacement convexity assumption is too strong and usually requires the $L^2$ penalty parameter $\lambda$ sufficiently large. It motivates us to explore the PL and Malliavin methods for the non-convex case. It is widely known that online-to-batch conversions imply that online learning is at least as hard as batch learning; see \cite{cesa2004generalization} and \citet[Theorem 3]{agarwal2019learning}. Moreover, the learner faces a varying environment due to independent increments from Brownian motions in the data dynamic \eqref{eq:data}. Hence, sublinear regrets should not be expected in the general non-convex and non-stationary cases. Our main goal changes to find explicit regret bounds and quantify the roles of regularization parameters $\lambda$ and $\beta$, even though the regret is linear.

\subsection{The displacement convexity method for static regret}
First, we prove that the offline benchmark $\rho^*$ is well-posed. The proof relies on the Schauder fixed-point theorem \cite[Section 1.15, Theorem 1.C]{zeidler1995applied}.
\begin{lemma}\label{lem:rho*}
The optimizer of the offline optimization problem \eqref{eq:statobj} exists and is the unique solution to
\begin{equation}\label{eq:rho*}
	\rho^*(\theta) = \frac{1}{A} \exp \left[ - \frac{\lambda}{2 \beta} |\theta|^2 -\frac{2}{\beta T} \int^T_0  \left( \ang{\rho^*, \sigma(X_t, \cdot)} -  Y_t \right) \sigma(X_t, \theta) dt \right],
\end{equation}
where $A$ is a normalization constant such that $\rho^*$ is a probability density.
\end{lemma}

The work of \cite{mccann1997convexity} first introduced the convexity of functionals in the Wasserstein space. We recall several useful concepts and results from \cite{ambrosio2008gradient,figalli2021invitation}. Consider a functional $\cE: \cP_2(\R^d) \to (-\infty, +\infty]$. For simplicity, we restrict the domain of $\cE$ to be a subset of $\cP^r_2(\R^d)$ and assume $\cE(\mu) = +\infty$ for any $\mu \notin \cP^r_2(\R^d)$.  We denote the effective domain of $\cE$ by $B(\cE) := \{ \mu \in \cP^r_2(\R^d) : \cE(\mu) < \infty\}$, and assume $B(\cE) \neq \emptyset$. According to Brenier's theorem, for any $\mu, \nu \in \cP^r_2(\R^d)$, there exists a unique map $T^\nu_\mu: \R^d \to \R^d$, such that the optimal coupling $\gamma \in \Pi(\mu, \nu)$ achieving the Wasserstein distance $\cW^2_2(\mu, \nu)$ concentrates on the graph of $T^\nu_\mu$. Specifically, we have
\begin{equation*}
	\gamma(\{ (x, y) \in \R^d \times \R^d: y = T^\nu_\mu(x) \}) = 1.
\end{equation*}
We refer to \citet[Corollary 2.5.12]{figalli2021invitation} for a detailed proof.

The following definition adapts \citet[Definition 9.1.1]{ambrosio2008gradient} by restricting the domain to the absolutely continuous space $\cP^r_2(\R^d)$, following McCann's approach. This restriction guarantees that the optimal transport map $T_\mu^\nu$ uniquely exists, which allows us to define the displacement interpolation directly through maps rather than transport plans.
\begin{definition}\label{def:disp-convex}
	Given a constant $L \in \R$, we say the functional $\cE$ is $L$-geodesically convex if for any $\mu, \nu \in \cP^r_2(\R^d)$, we have
	\begin{equation*}
		\cE(\pi^{\mu \to \nu}_{s}) \leq (1 - s) \cE(\mu) + s \cE(\nu)  - \frac{L}{2} s(1 - s) \cW^2_2(\mu, \nu), \quad \forall s \in [0, 1].
	\end{equation*}
	Here, $\pi^{\mu \to \nu}_{s} := ((1 - s) \mathrm{Id} + s T^{\nu}_{\mu})_{\#} \mu$ is the displacement interpolation between $\mu$ and $\nu$. When $L=0$, we say $\cE$ is displacement convex.
\end{definition}

We define the Fr\'echet subdifferential as follows, adapting \citet[Definition 10.1.1]{ambrosio2008gradient}. Because $B(\cE) \subset \cP^r_2(\R^d)$, Definition \ref{def:Frechet} only needs to test probability measures $\nu$ that are absolutely continuous.
\begin{definition}\label{def:Frechet}
	Suppose the functional $\cE: \cP_2(\R^d) \to (-\infty, +\infty]$ is lower semi-continuous. Consider a measure $\mu \in B(\cE)$ and a vector field $\xi: \R^d \to \R^d$ satisfying $\xi \in L^2(\mu)$. We say that $\xi$ belongs to the Fr\'echet subdifferential $\partial \cE(\mu)$ if
	\begin{equation}
		\liminf_{\nu \to \mu} \frac{\cE(\nu) - \cE(\mu) - \int_{\R^d} \ang{\xi(\theta), T_\mu^{\nu}(\theta) - \theta} \mu(d\theta)}{\cW_2(\mu, \nu)} \geq 0,
	\end{equation}
	where $\nu$ converges to $\mu$ in $\cW_2$. 
\end{definition}

Lemma \ref{lem:subdiff} provides an equivalent characterization of the Fr\'echet subdifferential with the $L$-geodesic convexity. The proof follows identically to the arguments presented in \citet[Section 10.1.1, p.231]{ambrosio2008gradient}.
\begin{lemma}\label{lem:subdiff}
	Suppose $\cE$ is lower semi-continuous and $L$-geodesically convex under Definition \ref{def:disp-convex}. Then a vector field $\xi \in L^2(\mu)$ belongs to the Fr\'echet subdifferential $\partial \cE(\mu)$ if and only if
	\begin{equation}
		\cE(\nu) - \cE(\mu) \geq \int \ang{\xi(\theta), T^\nu_\mu(\theta) - \theta} \mu(d\theta) + \frac{L}{2} \cW^2_2(\mu, \nu), \quad \forall \nu \in B(\cE).
	\end{equation}  
\end{lemma}

We present the first static regret bound in Theorem \ref{thm:dc-regret}. Note that the functional $F(\cdot, Z_t)$ is lower semi-continuous with respect to the $\cW_2$ metric, since the entropy functional satisfies this property \cite[Exercise 45, p.331]{Santambrogio2015}. The inequality \eqref{eq:dc-regret} holds for every continuous path of $Z_t$. When $L \geq 0$, this bounds the regret by a constant. In the general case where $L < 0$, the static regret bound becomes linear in $T$ and the linear term is not explicit. Our proof exploits the differentiability of the squared Wasserstein distance $\cW^2_2$ and the third claim in Theorem \ref{thm:Lip-grad}, which clarifies the technical conditions required by \citet[Theorem 2]{guo2022online}. Moreover, the regret inequality \eqref{eq:dc-regret} is connected with evolution variational inequalities; see \citet[Equation 11.0.3]{ambrosio2008gradient} for the definition and further discussion.
\begin{theorem}\label{thm:dc-regret}
	Given a constant $L \in \R$, for any continuous data path $\{Z_t\}_{t \in [0, T]}$, suppose the functional $F(\cdot, Z_t)$ is $L$-geodesically convex under Definition \ref{def:disp-convex}. Then
	\begin{equation}\label{eq:dc-regret}
		\int^T_0 F(\rho_t, Z_t) dt - \int^T_0 F(\rho^*, Z_t) dt \leq \frac{\cW^2_2(\rho^*, \rho_0) - \cW^2_2(\rho^*, \rho_T)}{2} - \frac{L}{2} \int^T_0 \cW^2_2(\rho^*, \rho_t) dt.
	\end{equation}
	Hence, the average static regret satisfies
	\begin{equation}
		\begin{aligned}
			\cR_S(T) \leq & \frac{\E[\cW^2_2(\rho^*, \rho_0)] - \E[\cW^2_2(\rho^*, \rho_T)]}{2} - \frac{L}{2} \E\Big[ \int^T_0 \cW^2_2(\rho^*, \rho_t) dt \Big].
		\end{aligned} 
	\end{equation}
\end{theorem}
 If we attempt to apply the displacement convexity framework to the dynamic regret, a technical difficulty is that the benchmark measure $\mu^*_t$ depends on time. The proof of Theorem \ref{thm:dc-regret} no longer works in this case.

To establish a sufficient condition for displacement convexity, Lemma \ref{lem:dc-suff-cond} indicates that the $L^2$ regularization parameter $\lambda$ must be sufficiently large. An additional limitation of Lemma \ref{lem:dc-suff-cond} is that the entropy parameter $\beta$ does not appear in \eqref{eq:lamdc} to aid the displacement convexity condition. This absence occurs because the entropy functional itself is not strictly displacement convex \citep{mccann1997convexity}, which ultimately motivates our investigation into non-convex settings.
\begin{lemma}\label{lem:dc-suff-cond}
	Suppose the Hessian of $\sigma(x, \theta)$ on $\theta$ has a uniformly bounded spectral norm, i.e. $\|\sigma_{\theta \theta}\|_{op, \infty} < \infty$. Then the functional $F(\rho, z)$ is displacement convex in $\rho$ if
	\begin{equation}\label{eq:lamdc}
		\lambda \geq 2 \left( C_\sigma + C_z \right) \| \sigma_{\theta \theta}\|_{op, \infty}.
	\end{equation}
\end{lemma}

\subsection{The PL method for dynamic regret}
In Euclidean spaces, the PL condition relaxes the convexity assumption while still ensuring algorithm convergence. It is natural to consider the counterpart of the PL condition in Wasserstein spaces, which is deeply connected to the LSI. Various methods exist to verify that a probability measure satisfies the LSI; we state a standard perturbation lemma below.
\begin{lemma}[Holley-Stroock perturbation lemma, \cite{holley1987logarithmic}]\label{lem:HolleyStroock}
	Suppose the probability measure $\mu$ satisfies a $LSI(\alpha)$, and $\tilde{\mu} = e^{-\psi} \mu$ is a bounded perturbation of $\mu$, that is, $\psi \in L^\infty$ and $\tilde{\mu}$ is a probability measure. Then $\tilde{\mu}$ satisfies $LSI(\tilde{\alpha})$ with 
	$$ \tilde{\alpha} = \alpha e^{-\text{osc}(\psi)}, \qquad \text{ where } \text{osc}(\psi) := \sup \psi - \inf \psi.$$
\end{lemma}
Furthermore, \cite{aida1994logarithmic} and \cite{cattiaux2022functional} consider Lipschitz perturbations and derive the corresponding LSI constants.

\begin{definition}
	 Define the constant $$\alpha := \frac{\lambda}{\beta} e^{-C_{osc}},$$ 
	 where $C_{osc} = \frac{4}{\beta} (C^2_\sigma + C_z C_\sigma)$.
\end{definition}

Parallel to Lemma \ref{lem:rho*}, Lemma \ref{lem:mu*LSI} establishes the existence and uniqueness of the equilibrium measure $\mu^*_t$. The proof proceeds similarly but relies on Brouwer's fixed-point theorem. Due to Assumptions \ref{A:sigma} and \ref{A:data}, the oscillation remains bounded. The LSI property of $\mu^*_t$ subsequently follows from the Bakry-\'Emery criterion combined with Lemma \ref{lem:HolleyStroock}.
\begin{lemma}\label{lem:mu*LSI} 
	The equilibrium measure $\mu^*_t$ exists and is the unique solution to
	\begin{equation}\label{eq:mu*}
		\mu^*_t(\theta) = \frac{1}{C_t} \exp \left[ - \frac{\lambda}{2 \beta} |\theta|^2 -\frac{2}{\beta}  \sigma(X_t, \theta) \left(  \int \sigma(X_t, \vartheta)\mu^*_t(d\vartheta) -  Y_t \right)  \right],
	\end{equation}
	where $C_t$ is a normalization constant such that $\mu^*_t$ is a probability density. Moreover, $\mu^*_t$ satisfies the $LSI(\alpha)$. 
\end{lemma}

Lemma \ref{lem:F_decom} demonstrates that the suboptimality gap between the objective $F(\rho, Z_t)$ and the instantaneous minimum $F(\mu^*_t, Z_t)$ decomposes into a squared term and the KL divergence. The proof utilizes the optimality of $\mu^*_t$ and the quadratic structure of the objective $F$. A direct consequence of Lemma \ref{lem:F_decom} is that the KL divergence lower bounds this suboptimality gap. Conversely, applying Pinsker's inequality provides an upper bound in terms of the KL divergence, as detailed in Equation \eqref{eq:F_KL1} in the Appendix.
\begin{lemma}\label{lem:F_decom}
	Consider any $\rho \in \cP^r_2(\R^d)$ with finite entropy. The objective functional $F$ satisfies
	\begin{equation}\label{eq:F_gap_decomp}
		F(\rho, Z_t) - F(\mu^*_t, Z_t) = \big(\ang{\rho, \sigma(X_t, \cdot)} - \ang{\mu^*_t, \sigma(X_t, \cdot)} \big)^2 + \beta D_{KL}(\rho \parallel \mu^*_t).
	\end{equation}
\end{lemma}

Recall that for a smooth objective function $f(x)$ in a finite-dimensional Euclidean space, the PL condition states that the suboptimality gap satisfies 
\begin{equation}\label{eq:EuclideanPL}
	f(x) - f(x^*) \leq C |D_x f(x)|^2,
\end{equation}
where $x^*$ denotes the minimizer of $f$. In Lemma \ref{lem:PL}, the squared Wasserstein gradient replaces the standard Euclidean gradient in \eqref{eq:EuclideanPL}. This formulation intrinsically relates to the relative Fisher information. By applying the LSI, we establish the corresponding PL condition within the Wasserstein space.
\begin{lemma}\label{lem:PL}
	Suppose $\alpha \beta^2 > 8 C_\sigma^2 C_{1}^2$. The PL condition holds for $F(\rho, Z_t)$:
	\begin{equation}
		F(\rho, Z_t) - F(\mu^*_t, Z_t) \le C_{PL} \left\|\nabla \left( \frac{\delta F(\rho, Z_t)}{\delta \rho} \right) \right\|_{L^2(\rho)}^2, 
	\end{equation}
	where 
	\begin{equation}
		C_{PL} = \frac{2 C_\sigma^2 + \beta}{\alpha \beta^2 - 8 C_\sigma^2 C_{1}^2}.
	\end{equation}
\end{lemma}

The necessity of the technical condition $\alpha \beta^2 > 8 C_\sigma^2 C_{1}^2$ merits further discussion. This requirement essentially imposes a strictly positive lower bound on the temperature parameter to overcome the non-convex interactions within the functional. Indeed, this condition is comparable to the structural assumption required for the propagation of chaos in \cite{lacker2023sharp}.

We remark on the time-varying nature of the objective. The PL inequality in Lemma \ref{lem:PL} applies to the frozen-time functional $F(\cdot, Z_t)$, guaranteeing contraction toward the instantaneous minimizer $\mu_t^*$ at any fixed time $t$. In the online setting, however, the continuous evolution of the data process $Z_t$ renders the minimizer $\mu_t^*$ time-dependent.  Consequently, the learning dynamics of $\rho_t$ must track a moving target. Even when $\rho_t$ closely approximates $\mu_t^*$, the ongoing shift in $Z_t$ continuously displaces the minimizer. Therefore, the suboptimality gap does not simply decay to zero. Instead, it reflects a dynamic balance between the contraction induced by the PL condition and the underlying drift of the shifting objective.

This mechanism is made explicit in the It\^o expansion below: when differentiating the suboptimality gap \(F(\rho_t, Z_t) - F(\mu_t^*, Z_t)\), one obtains a contraction term governed by the PL inequality together with additional drift terms arising from the time dependence of \(Z_t\) and the induced evolution of \(\mu_t^*\).

The primary strategy for bounding the dynamic regret involves deriving the time evolution of the suboptimality gap, which relies on an appropriate application of It\^o's lemma.  Because $\rho_t$ satisfies the continuity equation, one can directly apply the chain rule from \cite{ambrosio2008gradient} to compute the differential $dF(\rho_t, Z_t)$. However, the equilibrium measure $\mu^*_t$ depends implicitly on the data $Z_t$ through the self-consistent equation  \eqref{eq:mu*}. This intricate interaction between $\mu^*_t$ and $Z_t$ makes the derivation of the corresponding It\^o's lemma non-trivial. We introduce several notations to address this challenge. Adopting the convention that first-order gradients are column vectors, we define the differential operator evaluated at a fixed measure $\rho$ as follows:
\begin{align*}
	\cL_z[F] := & F^\top_x b_1 dt + F_y b_2 dt + \frac{1}{2} \tr[F_{xx} \Sigma_1 \Sigma^\top_1] dt + F_{yx} \Sigma_1 \Sigma^\top_2 dt + \frac{1}{2} F_{yy} \Sigma_2 \Sigma^\top_2 dt \\
	& + F^\top_x \Sigma_1 d W_t + F_y \Sigma_2 dW_t \\
	:= & \cA[F] dt + F^\top_x \Sigma_1 d W_t + F_y \Sigma_2 dW_t,
\end{align*}
with derivatives obtained from a direct calculation, given by
\begin{align*}
	F_x(\rho, Z_t) & = 2 \ang{\rho, \sigma(X_t, \cdot)} \ang{\rho, \sigma_x (X_t, \cdot)} - 2 Y_t \ang{\rho, \sigma_x(X_t, \cdot)}, \quad F_y(\rho, Z_t) = -2 \ang{\rho, \sigma(X_t, \cdot)}, \\
	F_{xx} (\rho, Z_t) & = 2 \ang{\rho, \sigma(X_t, \cdot)} \ang{\rho, \sigma_{xx} (X_t, \cdot)} + 2 \ang{\rho, \sigma_x(X_t, \cdot)} \ang{\rho, \sigma_{x} (X_t, \cdot)}^\top  - 2 Y_t \ang{\rho, \sigma_{xx}(X_t, \cdot)}, \\
	F_{xy} (\rho, Z_t) & = - 2 \ang{\rho, \sigma_{x}(X_t, \cdot)} \in \R^{n \times 1}, \qquad F_{yx} (\rho, Z_t) = F_{xy} (\rho, Z_t)^\top, \qquad F_{yy} (\rho, Z_t) = 0.
\end{align*}

The self-consistent equation \eqref{eq:mu*} reveals that the temporal dependence of $\mu^*_t$ occurs strictly through the state variable $Z_t$. To emphasize this parametric dependence, we adopt the notation:
\begin{equation}
	\mu^*(\theta, Z_t) = \mu^*_t (\theta).
\end{equation}

For notational simplicity, letting $z = (x, y) \in \R^{n+1}$, we define the auxiliary functions:
\begin{align*}
	m(z) & := \int \sigma(x, \theta) \mu^*(\theta, z) d\theta, \quad e(z) := m(z) - y, \quad H(\theta, z) := \frac{\lambda}{2}|\theta|^{2} + 2 \sigma(x, \theta) e(z).
\end{align*}
For any function $f(z, \theta) : \R^{n + 1 + d} \to \R$ that is differentiable with respect to $z$, we introduce the following operator:
\begin{equation}
	\begin{aligned}
	\cS_u(z; f) & := D_u \Big( \int f(z, \theta) \mu^*(\theta, z) d \theta \Big) - \int D_u f(z, \theta) \mu^*(\theta, z) d \theta = \int f(z, \theta) D_u \mu^*(\theta, z) d \theta.
	\end{aligned} 
\end{equation}
The subscript $u$ denotes the specific components of $z$ with respect to which the derivatives are taken.

Lemma \ref{lem:dFmu*} demonstrates the differentiability of $m(z)$ via the implicit function theorem. Consequently, the self-consistent equation \eqref{eq:mu*} ensures that the density $\mu^*(\theta, z)$ is also differentiable with respect to $z$. This regularity allows us to establish It\^o's lemma for the composite process $F(\mu^*_t, Z_t)$. The dynamic interaction between $\mu^*_t$ and $Z_t$ introduces several new cross-variation terms governed by the operator $\cS_z$. These additional terms possess finite $L^\infty$ norms that asymptotically vanish as the temperature parameter $\beta$ approaches infinity.  This asymptotic behavior is physically intuitive: at extremely high temperatures, the equilibrium measure $\mu^*_t$ flattens toward a uniform distribution, thereby heavily damping its sensitivity to the incoming fluctuations from the data stream $Z_t$. 
\begin{lemma}\label{lem:dFmu*}
\begin{enumerate}
	\item The function $\mu^*(\theta, z)$ is differentiable in $z$. Moreover,
	\begin{align*}
		D_x m(z) & = \frac{\beta}{\beta + 2\mathrm{Var}_{\theta \sim \mu^*} [ \sigma(x, \theta)]} \Big[ \int \sigma_x(x, \theta) \mu^*(\theta, z) d \theta - \frac{2e(z)}{\beta} \mathrm{Cov}_{\theta \sim \mu^*} \big( \sigma(x, \theta), \sigma_x(x, \theta) \big) \Big], \\
		D_y m(z) & = \frac{2  \mathrm{Var}_{\theta \sim \mu^*} [\sigma(x, \theta)]}{\beta + 2 \mathrm{Var}_{\theta \sim \mu^*} [\sigma(x, \theta)]},
	\end{align*}
	where $\theta \sim \mu^*$ means the random parameter $\theta$ follows the distribution $\mu^*(\cdot, z)$.
	\item For $f(z, \theta) : \R^{n + 1 + d} \to \R$ that is differentiable on $z$,
	\begin{equation}
		\cS_z(z ; f) = - \frac{1}{\beta} \mathrm{Cov}_{\theta \sim \mu^*} \big[f(z, \theta), D_z H(\theta, z) \big].
	\end{equation}
	
	\item The functional $F(\mu^*_t, Z_t)$ satisfies the It\^o's lemma given by
	\begin{equation}\label{eq:dFmu*}
		\begin{aligned}
			d F(\mu^*_t, Z_t) = & \cA[F](\mu^*_t, Z_t) dt + F^\top_x(\mu^*_t, Z_t) \Sigma_1 d W_t + F_y(\mu^*_t, Z_t) \Sigma_2 dW_t \\
			& + \tr \big[\cS_x(Z_t; \sigma) \ang{\mu^*_t, \sigma_x(X_t, \cdot)}^\top \Sigma_1 \Sigma^\top_1 \big] dt \\
			& + \tr \big[(\ang{\mu^*_t, \sigma(X_t, \cdot)}  - Y_t) \cS_x(Z_t; \sigma_x)  \Sigma_1 \Sigma^\top_1 \big] dt \\
			& -2 \cS_x(Z_t; \sigma)^\top \Sigma_1 \Sigma^\top_2 dt - \cS_y(Z_t; \sigma) \Sigma_2 \Sigma^\top_2 dt.
		\end{aligned}
	\end{equation}
	Here, $\cS_x(z; \sigma_x) := (\cS_x(z; \sigma_{x_1}), \ldots, \cS_x(z; \sigma_{x_n})) \in \R^{n \times n}$. 
	\item There exists a constant $C_* > 0$ that depends on $C_\sigma, C_z$, and $C_1$ only, such that
	\begin{align*}
		\| \cS_x(Z_t; \sigma) \|_{L^\infty}, \| \cS_y(Z_t; \sigma) \|_{L^\infty}, \| \cS_x(Z_t; \sigma_x) \|_{L^\infty}  \leq &  \frac{C_*}{\beta} +  \frac{C_*}{\beta^2}.
	\end{align*}
\end{enumerate}
\end{lemma}

Theorem \ref{thm:PL-regret} presents our main dynamic regret bound. The proof combines the PL condition with Talagrand's transportation inequality to construct a governing differential inequality. Integrating this relation directly yields the stated upper bound.
\begin{theorem}\label{thm:PL-regret} 
	Suppose $\alpha \beta^2 > 8 C_\sigma^2 C_{1}^2$. The average dynamic regret satisfies
	\begin{equation}\label{eq:PL-regret}
		\begin{aligned}
		& \cR_D(T) \leq 2 C_{PL} \E[F(\rho_0, Z_0) - F(\mu^*_0, Z_0)] + \frac{C_* C^2_{b, \Sigma} C^2_{PL}}{\alpha \beta} T + C_* C_{b, \Sigma} C_{PL} \Big(\frac{1}{\beta} + \frac{1}{\beta^2} \Big) T,
		\end{aligned}
	\end{equation}
	where $C_* > 0$ depends on $ C_\sigma, C_z, C_1, C_2$, and $C_3$ only.
\end{theorem}

To isolate the structural impact of the entropy parameter $\beta$ and the regularization parameter $\lambda$ on this bound, we consider the following limits for a fixed $\lambda > 0$:
\begin{equation*}
	\lim_{\beta \to \infty} C_{PL}= \frac{1}{\lambda} \quad \text{ and } \quad \lim_{\beta \to \infty} \frac{C^2_{PL}}{\alpha \beta} = \frac{1}{\lambda^3}.
\end{equation*}
The first component of the regret bound \eqref{eq:PL-regret} is a non-negative constant independent of the time horizon $T$, representing the initial suboptimality of the prior measure $\rho_0$. Both the second and third components scale with $C_{b, \Sigma}$, a constant capturing the overall volatility and drift of the underlying diffusion environment. Specifically, the third component originates directly from the previously discussed $\cS_z$ sensitivity terms. If one assumes that the environmental parameter $C_{b, \Sigma}$ decays at an appropriate rate over time, achieving sublinear dynamic regret remains theoretically possible. However, such an assumption implies a continuously freezing environment, which is overly restrictive and unrealistic for most practical online learning scenarios.

\subsection{The Malliavin approach for static regret}
Although dynamic regret bounds static regret, one might wonder if it is possible to derive an upper bound for the static regret directly using the PL method. Doing so may effectively quantify the anticipating behaviors of the optimal measure $\rho^*$.

Because the offline optimizer is anticipative, Malliavin calculus provides a suitable framework for investigating the dynamics of the objective functional $F(\rho^*, Z_t)$.  We write $\rho^*(\theta) = \rho^*(\theta, Z)$ to emphasize the dependence on the full data trajectory $Z$, omitting the specific time interval $[0, T]$ for notational brevity. This problem requires delicate handling because the random variable $\rho^*(\theta, Z)$ is not real-valued but instead resides in a functional space. Consequently, we formulate the problem within $L^2(\R^d)$, which permits the application of Malliavin calculus for unconditional martingale difference (UMD) Banach spaces \citep{pronk2014tools}.

We denote the domain of the Malliavin derivative operator $\cD$ on $L^2(\R^d)$-valued random variables by $\bD^{1, 2}(L^2(\R^d))$. For a precise definition of the operator $\cD$ and its corresponding domain, we refer readers to \citet[Section 2.2]{pronk2014tools}. Throughout our analysis, any application of Malliavin calculus is performed exclusively on the data probability space $(\Omega^W, \cF^W, \p^W)$ with respect to the underlying Wiener process $W$. Since the particle driving noises $B$ and $(B^i)_{i=1}^\infty$ are defined on an independent probability space, they act as constants under the Malliavin derivative operator $\cD$.

The first technical challenge is establishing that the Malliavin derivative $\cD \rho^*(\theta, Z)$ is well-defined. This presents a difficulty similar to the one encountered in Lemma \ref{lem:dFmu*}, as the optimal measure $\rho^*$ depends on $Z$ implicitly. The primary strategy utilizes the implicit function theorem in Banach spaces \cite[Theorem 4.E]{zeidler1995vol2} combined with the Malliavin chain rule for UMD Banach spaces \cite[Proposition 3.8]{pronk2014tools}. Furthermore, the Lax-Milgram theorem plays a pivotal role in demonstrating that the associated differential operator possesses a bounded inverse.
\begin{lemma}\label{lem:Drho_exist}
		The density $\rho^*(\theta, Z)$ is Fr\'echet differentiable on $Z$ and the Fr\'echet derivative $\partial_{Z} \rho^*(\theta, Z)$ is continuous and bounded. Moreover, $\rho^* \in \bD^{1, 2}(L^2(\R^d))$ and the Malliavin derivative is given by
		\begin{equation}
			\cD \rho^*(\theta, Z) = \partial_{Z} \rho^*(\theta, Z) \cD Z.
		\end{equation}
\end{lemma}

Because the optimal measure $\rho^*$ relies on $Z_t$ strictly through a time integral from $0$ to $T$, the Malliavin derivative $\cD_t \rho^* = (\cD \rho^*)_t$ is continuous for all $t \in [0, T]$. To characterize its properties further, we define the ratio:
\begin{equation}
	\psi_t(\theta) := \frac{\cD_t \rho^*(\theta)}{\rho^*(\theta)}.
\end{equation}
	
The following auxiliary process will be used in the dynamics of $\psi_t$:
	\begin{equation*}
		S_t(\theta) := \frac{2}{\beta T} \int_t^T \left[ (\langle \rho^*, \sigma(X_s, \cdot) \rangle - Y_s) \sigma_x(X_s, \theta) \cD_t X_s + \sigma(X_s, \theta) \left( \langle \rho^*, \sigma_x(X_s, \cdot) \rangle \cD_t X_s - \cD_t Y_s \right) \right] ds.
	\end{equation*}
Additionally, we introduce an integral operator defined by
\begin{align*}
	\cK(\psi_t, Z, \theta) := \frac{2}{\beta T} \int^T_0 \ang{\rho^*, \psi_t(\cdot)\sigma(X_s, \cdot)} \sigma(X_s, \theta) ds.
\end{align*}

Thanks to the specific quadratic structure of the functional $F$, Lemma \ref{lem:psi} establishes a linear integral equation governing $\psi_t(\theta)$. Furthermore, we establish an $L^1$-norm bound on the Malliavin derivative $\cD_t \rho^*$ in \eqref{eq:Mal-L1}. This bound reveals the strength of the anticipative behavior exhibited by $\rho^*$. The entropy parameter $\beta$ weakens the dependence on future information. Because a larger $\beta$ flattens the optimal distribution $\rho^*$ toward a uniform measure, it suppresses the sensitivity of the optimizer to future data fluctuations.   
\begin{lemma}\label{lem:psi}
		Given $Z$, the function $\psi_t$ is the solution to the following Fredholm integral equation of the second kind:
		\begin{equation}
			\psi_t(\theta) = -[ S_t(\theta) + \cK(\psi_t, Z, \theta)] + \ang{\rho^*, S_t(\cdot) +\cK(\psi_t, Z, \cdot)}.
		\end{equation}
		Moreover, the $L^1$-norm of the Malliavin derivative of the density $\rho^*$ satisfies
		\begin{equation}\label{eq:Mal-L1}
			| \cD_t \rho^* |_{L^1} = \int |\cD_t \rho^*(\theta)| d\theta \leq \sup_{\theta \in \R^d} |S_t(\theta)| \leq \frac{C_*}{\beta T} \int^T_t (|\cD_t X_s| + |\cD_t Y_s|) ds,
		\end{equation}
		where $C_* > 0$ depends on $C_\sigma, C_1, C_z$ only.
\end{lemma}

Because the optimal measure $\rho^*$ is random and depends on the entire data path $Z_{[0,T]}$, we must establish an anticipating It\^o's formula applicable to our specific framework. The primary strategy adapts the proof of \citet[Theorem 3.2.2]{nualart2006malliavin}. First, we utilize the representation of the objective functional $F(\rho^*, Z_t)$ obtained from the proof of Lemma \ref{lem:rho*}:
\begin{align*}
	F(\rho^*, Z_t) = & \langle \rho^*, \sigma(X_t, \cdot) \rangle^2 - 2 Y_t \langle \rho^*, \sigma(X_t, \cdot) \rangle   \\
	&  - \frac{2}{T} \int^T_0  \langle \rho^*, \sigma(X_t, \cdot) \rangle^2 dt  + \frac{2}{T} \int^T_0 Y_t \langle \rho^*, \sigma(X_t, \cdot) \rangle dt  - \beta \log(A).
\end{align*}
This representation is more tractable because $\rho^*$ remains fixed with respect to the time index $t$ within the differential, meaning that only the first two terms actively contribute to the anticipating It\^o's formula presented in Lemma \ref{lem:antic_Ito}. Furthermore, integrals of the form $\int^t_0 f(u) \delta W_u$ denote the Skorohod integral. We refer to \citet[Chapter 3]{nualart2006malliavin} for a comprehensive discussion of its mathematical properties. Compared to \citet[Theorem 3.2.2]{nualart2006malliavin}, our proof differs primarily in establishing the well-posedness of these Skorohod integrals.
\begin{lemma}\label{lem:antic_Ito}
		Suppose
		\begin{equation}
			\begin{aligned}
				& \E \Big[ \int^T_0 \int^T_0 | \cD_s b_i(r) |^4 dr ds \Big] + \E \Big[ \int^T_0 \int^T_0 | \cD_s \Sigma_i(r) |^4 dr ds \Big] < \infty, \quad i= 1, 2.
			\end{aligned}
		\end{equation}
		Then
		\begin{equation}\label{eq:Ito}
			\begin{aligned}
				F(\rho^*, Z_t) - F(\rho^*, Z_0) = & \int^t_0 \cA[F](\rho^*, Z_u) du + \int^t_0 \cM_u du \\
				& +  \int^t_0 F^\top_x(\rho^*, Z_u) \Sigma_1(u) \delta W_u + \int^t_0 F_y(\rho^*, Z_u) \Sigma_2(u) \delta W_u,
			\end{aligned}
		\end{equation}
		where
		\begin{align*}
			\cM_t := & 2 \Sigma_1(t) \big(\langle \rho^*, \sigma(X_t, \cdot) \rangle - Y_t\big) \ang{\cD_t \rho^*, \sigma_x(X_t, \cdot)} + 2 \ang{\cD_t \rho^*, \sigma(X_t, \cdot)} (\Sigma_1(t) \ang{ \rho^*, \sigma_x(X_t, \cdot)}  - \Sigma_2 (t)).
		\end{align*}
\end{lemma}

To derive the static regret bound in Theorem \ref{thm:Mallivian-regret}, we employ a similar methodology by replacing the It\^o's formula used in Theorem \ref{thm:PL-regret} with the anticipating counterpart established above.
\begin{theorem}\label{thm:Mallivian-regret}
	Suppose $\alpha \beta^2 > 8 C_\sigma^2 C_{1}^2$ and there exists a constant $C_{a} > 0$ independent of $T$, such that
	\begin{equation}
		\E \Big[ \Big( \int_t^T (|\cD_t X_s| + |\cD_t Y_s|) ds \Big)^2 \Big]^{1/2} \leq C_{a} T, \quad t \in [0, T].
	\end{equation}
	Then the average static regret satisfies
		\begin{equation}\label{eq:Mallivian-regret}
			\begin{aligned}
			\cR_S(T) & \leq \Big(\E[F(\rho_0, Z_0) - F(\rho^*, Z_0)]- C_{PL} C_* C_{b, \Sigma} \Big(1 + \frac{C_a}{\beta} \Big) \Big) C_{PL} (1 - e^{-T/C_{PL}}) \\
				& \quad +  C_{PL} C_* C_{b, \Sigma} \Big(1 + \frac{C_a}{\beta} \Big) T,
			\end{aligned}
		\end{equation}
		where $C_* > 0$ depends on $ C_\sigma, C_z, C_1, C_2$, and $C_3$ only.
\end{theorem}
A crucial difference in this proof is that we bound $\E[\cA[F] (\rho_t, Z_t) - \cA[F] (\rho^*, Z_t)]$ directly using the standing Assumptions \ref{A:sigma} and \ref{A:data}, rather than relying on the Wasserstein distance $\cW_2(\rho_t, \rho^*)$. The primary reason is the absence of a decomposition result analogous to Lemma \ref{lem:F_decom}. In fact, the suboptimality gap $F(\rho_t, Z_t) - F(\rho^*, Z_t)$ is not guaranteed to be non-negative. Because of this relatively loose estimation technique, the resulting static regret bound in \eqref{eq:Mallivian-regret} may appear quantitatively worse than its dynamic counterpart. Indeed, as the entropy parameter $\beta$ approaches infinity while the regularization penalty $\lambda > 0$ remains fixed, the upper bound converges to 
\begin{equation*}
	\Big(\E[F(\rho_0, Z_0) - F(\rho^*, Z_0)]-  \frac{C_* C_{b, \Sigma}}{\lambda} \Big) \frac{(1 - e^{-\lambda T})}{\lambda}  +  \frac{C_* C_{b, \Sigma} T}{\lambda}.
\end{equation*}
This asymptotic limit takes the form $(C + C T)/\lambda$, because the initial evaluation $F(\rho^*, Z_0)$ remains bounded as the time horizon $T$ approaches infinity, even though the optimal measure $\rho^*$ depends on $T$. Although the dependence on $\lambda$ yields a worse rate compared to the dynamic regret bound, this formulation provides a significant theoretical advantage. Specifically, it explicitly isolates and quantifies the anticipative behaviors on the overall regret bounds through the constant $C_a$. From an algorithmic perspective, a higher temperature $\beta$ introduces stronger diffusive noise into the online learning dynamics, ultimately diluting its sensitivity to future data realizations.

\section{Regret bounds for the particle case}\label{sec:particle-regret} 

Our analysis of the particle regret builds upon recent progress in quantifying the propagation of chaos using relative entropy estimates \citep{lacker2023sharp}. A key condition is the time-uniform LSI for the mean-field flow $\rho_t$. Furthermore, applying \citet[Theorem 2.1]{lacker2023sharp} requires the entropy parameter $\beta$ to be sufficiently large. This condition is crucial because counterexamples demonstrate that propagation of chaos cannot hold uniformly in time when the temperature is small \citep[Remark 2.2]{lacker2023sharp}. Loosely speaking, a high temperature (large $\beta$) overrides the mean-field interactions that may cause the system to have multiple invariant measures.

\subsection{Marginal-law dynamic regret}
Let $P_t^{k}$ denote the joint law of $(\theta_t^{1,N},\dots,\theta_t^{k,N})$. When $k=1$, it reduces to $P_t^{1} = \rho^{1, N}_t$, which represents the marginal density of the particle $\theta^{1, N}_t$. Lemma \ref{lem:PoC} follows directly from the relative entropy estimates concerning the propagation of chaos established in \citet[Theorem 2.1 (1)]{lacker2023sharp}.
\begin{lemma}\label{lem:PoC}
	If $\eta \beta^2 > 8 C^2_\sigma C^2_1$, then there exists a constant $C_{poc} > 0$ such that 
	\begin{equation*}
		D_{KL}(P^{k}_t \parallel \rho_t^{\otimes k}) \leq C_{poc} \frac{k^2}{N^2},
	\end{equation*}
	for all $t \geq 0$ and $k=1, \ldots, N$. The constant $C_{poc}$ depends only on $\eta$, $\beta$, $C_\sigma$, and $C_1$.
\end{lemma}
We can compare the assumption required in Lemma \ref{lem:PoC} with our condition $\alpha \beta^2 > 8 C^2_\sigma C^2_1$ governing the mean-field regret bounds in Theorems \ref{thm:PL-regret} and \ref{thm:Mallivian-regret}. Because the remark following Lemma \ref{lem:rhotLSI} indicates that $\eta \beta$ asymptotically approaches $\lambda$ for large $\beta$, the requirement $\eta \beta^2 > 8 C^2_\sigma C^2_1$ in Lemma \ref{lem:PoC} is comparable to the structural condition $\alpha \beta^2 > 8 C^2_\sigma C^2_1$.

Propositions \ref{prop:MkV} and \ref{prop:particle} ensure the existence of a constant $M_2(\beta, \lambda) <\infty$ such that
\begin{equation}\label{eq:M2}
	\sup_{t\in[0,T]}\int_{\R^d}|\theta|^2\,\rho_t(d\theta)\le M_2(\beta, \lambda), \qquad \sup_{N\ge1}\sup_{t\in[0,T]}\int_{\R^d}|\theta|^2 \rho^{1, N}_t(d\theta)\le M_2(\beta, \lambda).
\end{equation}
The notation $M_2(\beta, \lambda)$ indicates that this constant depends at least on the parameters $\beta$ and $\lambda$, while remaining independent of the total number of particles $N$.

Theorem \ref{thm:margin-regret} bounds the marginal-law dynamic regret using the mean-field $\cR_{D}(T)$, which is subsequently controlled by Theorem \ref{thm:PL-regret}. Regarding the final two terms in \eqref{eq:margin-regret}, we emphasize that the constant $L_\rho$ also exhibits a dependence on $\beta$ and $\lambda$. The bounds in \eqref{eq:piNbound} and \eqref{eq:margin-regret} explode as $\beta$ approaches infinity, and the second moment of $\rho_t$ similarly grows with $\beta$. Consequently, when the entropy or temperature parameter $\beta$ is large, the system requires a correspondingly larger number of particles $N$ to achieve a tight regret bound. This establishes a trade-off in the algorithm design: while a high noise parameter $\beta$ is necessary to guarantee uniform-in-time propagation of chaos, it simultaneously inflates the constants in the marginal-law regret bounds, forcing the user to simulate more particles to maintain the desired performance guarantees.
\begin{theorem}\label{thm:margin-regret}
	Suppose $\eta \beta^2 > 8 C^2_\sigma C^2_1$. Then
	\begin{equation}\label{eq:piNbound}
		|F(\pi^N_t, Z_t) - F(\rho_t, Z_t)| \leq   \frac{C_\eta}{N} + \frac{\beta C_{poc}}{N^2},
	\end{equation}
	where $C_\eta := [ 2(C_\sigma + C_z) C_1 + (\lambda + \beta L_\rho)(1 + 2 \sqrt{M_2(\beta, \lambda)})] \sqrt{2 C_{poc}/\eta}$. Moreover, the marginal-law dynamic regret satisfies
	\begin{equation}\label{eq:margin-regret}
		\cR_{\mathrm{law}}(T; N) \leq \cR_{D}(T) + \Big(\frac{C_\eta}{N} + \frac{\beta C_{poc}}{N^2} \Big) T.
	\end{equation}
\end{theorem}

\subsection{The empirical dynamic regret without entropy}
The previous results rely heavily on the entropy regularization. If we solely consider regret bounds for the unregularized functional $U$ without the entropy term, additional analysis is required to control the dependence on the temperature parameter $\beta$. In this subsection, we aim to establish a regret bound for the functional $U$ evaluated at the empirical measure $\hat{\rho}^N_t$. First, we present a preliminary result that controls the second moments under the instantaneous optimal measure $\mu^*_t$, which is subsequently utilized in Lemma \ref{lem:U_regret}.
\begin{lemma}\label{lem:mu*_2m}
	For all $t \ge 0$, the second moment under $\mu^*_t$ is uniformly bounded as follows:
	\begin{equation}
		\int_{\R^d} |\theta|^2 \mu^*_t(d\theta) \le Q_*(\beta, \lambda),
	\end{equation}
	where the constant
	\begin{equation}\label{eq:Q*}
		Q_*(\beta, \lambda) :=  \frac{\beta d}{\lambda} \exp\left( \frac{4 C_\sigma (C_\sigma + C_z)}{\beta} \right) = \frac{\beta d}{\lambda} e^{C_{osc}}
	\end{equation}
	is independent of $t$.
\end{lemma}

We establish the following counterpart to Theorem \ref{thm:PL-regret} concerning the dynamic regret evaluated with mean-field densities, explicitly excluding the entropy term. We recall that the expectation $\E[\cdot]$ is under $\p^W$. The proof of Lemma \ref{lem:U_regret} modifies the arguments used for Theorem \ref{thm:PL-regret}.
\begin{lemma}\label{lem:U_regret}
	Suppose $\alpha \beta^2 > 8 C_\sigma^2 C_{1}^2$. Then
	\begin{align}
		\mathbb E\Big[\int_0^T\big(U(\rho_t,Z_t)-U(\mu_t^*,Z_t)\big)\,dt\Big]
		&\le
		\frac{2C_1^2 + \lambda}{\alpha\beta} \cR_{D}(T)
		\\
		&\quad+
		\left(2C_1(C_\sigma+C_z)+\lambda\sqrt{Q_*(\beta, \lambda)}\right)\sqrt{\frac{2}{\alpha\beta}}\sqrt{T\cR_D(T)}. 	\nonumber
	\end{align}
\end{lemma}

The empirical measure $\hat{\rho}_t^N$ is a random quantity arising from both the data stream $Z$ and the independent Brownian motions $B$ driving the interacting particles.  In contrast, the mean-field density $\rho_t$ inherits randomness solely from the data stream. Lemma \ref{lem:UdiffmfN} quantifies the discrepancy between the functional $U$ evaluated at $\hat{\rho}_t^N$ and $\rho_t$, conditioned on a fixed realization of the data trajectory $Z$ driven by $W_t$. We recall that $\E_{B}[\cdot]$ denotes the expectation taken strictly under the particle noise measure $\p^B$. The proof of Lemma \ref{lem:UdiffmfN} relies on the propagation of chaos established in Lemma \ref{lem:PoC} for $k=2$ and heavily utilizes the underlying quadratic structure of the functional $U$.
\begin{lemma}\label{lem:UdiffmfN}
	Suppose $\eta \beta^2 > 8 C^2_\sigma C^2_1$. Given the data $Z$, we have 
	\begin{equation}
		\E_{B}\left[\int_0^T\big(U(\hat{\rho}_t^N, Z_t)-U(\rho_t,Z_t)\big)\,dt\right] \leq \frac{C_U(\beta, \lambda, \eta)}{N} T,
	\end{equation}
	where
	\begin{equation}\label{eq:CU}
		C_U(\beta, \lambda, \eta) := 2C_\sigma^2 + \sqrt{\frac{2 C_{poc}}{\eta}} \Big(2\sqrt2\,C_\sigma C_1 + 2C_z C_1 + \lambda\,\sqrt{M_2(\beta, \lambda)}\Big)
	\end{equation}
	and $M_2(\beta, \lambda)$ is defined in \eqref{eq:M2}.
\end{lemma}

Because the bound derived in Lemma \ref{lem:UdiffmfN} does not depend on the data $Z$, we can directly combine it with Lemma \ref{lem:U_regret} to obtain the empirical dynamic regret without entropy.
\begin{theorem}\label{thm:empircal_regret}
	Suppose $\alpha \beta^2 > 8 C_\sigma^2 C_{1}^2$ and $\eta \beta^2 > 8 C^2_\sigma C^2_1$. Then the empirical dynamic regret without entropy satisfies 
	\begin{equation}\label{eq:empirical_regret}
		\begin{aligned}
			\cR_{\mathrm{em}}(T; N) \leq & \frac{2C_1^2 + \lambda}{\alpha\beta} \cR_{D}(T)  +
			\left(2C_1(C_\sigma+C_z)+\lambda\sqrt{Q_*(\beta, \lambda)}\right)\sqrt{\frac{2}{\alpha\beta}}\sqrt{T\cR_D(T)} \\
			& + \frac{C_U(\beta, \lambda, \eta)}{N} T,
		\end{aligned}
	\end{equation}
	where $Q_*(\beta, \lambda)$ and $C_U(\beta, \lambda, \eta)$ are defined in \eqref{eq:Q*} and \eqref{eq:CU}, respectively.
\end{theorem}

We now compare the bounds established in Theorems \ref{thm:margin-regret} and \ref{thm:empircal_regret}. If we take the mean-field limit $N \to \infty$ first, Theorem \ref{thm:margin-regret} leaves only the term $\cR_D(T)$, which remains finite when $\beta \to \infty$. However, the corresponding bound in Theorem \ref{thm:empircal_regret} explodes in this regime. Consequently, to achieve a meaningful bound, the $L^2$ penalty parameter $\lambda$ must be chosen to be sufficiently large in Theorem \ref{thm:empircal_regret} to compensate for the explosion caused by $\beta$.

As an additional remark, even if we assume that the penalty $\lambda$ is sufficiently large, it remains non-trivial to utilize the displacement convexity method to directly investigate the empirical dynamic regret. A primary difficulty arises because the benchmark measure $\mu^*_t$ is time-varying. Consequently, the proof techniques supporting Theorem \ref{thm:dc-regret} no longer apply, as they inherently rely on a static target measure. We leave this problem for future research.

\section{Simulation studies}\label{sec:empirical}

This section presents a simulation study to complement the theoretical results. The empirical analysis consists of two main parts. First, to motivate the online learning approach, we investigate the out-of-sample accuracy of both offline and online methods when the underlying data-generating process exhibits periodic trends or nonlinear time-varying dynamics. Second, we estimate the empirical regrets and evaluate the impact of the time horizon $T$, the number of particles $N$ (representing the network width), and the regularization parameters $\lambda$ and $\beta$. We do not strictly constrain $\lambda$ and $\beta$ to be sufficiently large to satisfy the theoretical conditions $\alpha \beta^2 > 8 C_\sigma^2 C_{1}^2$ and $\eta \beta^2 > 8 C^2_\sigma C^2_1$. Consequently, the simulation study operates primarily within the highly non-convex regime. We present several empirical phenomena that align with our regret bounds, alongside other observations that highlight the need for tighter theoretical results in future work.

\subsection{Experimental setting}
\label{subsec:sim_models}
We introduce two models to generate data for the estimation problem.

{\bf Periodic regression model}. In practical scenarios, such as airline passenger demand forecasting, data often exhibit highly seasonal or periodic patterns. In the first model, the one-dimensional covariate follows an Ornstein-Uhlenbeck (OU) process, given by 
\begin{equation*}
    dX_t = - 0.5 X_t\,dt + dW_t^X, 
\end{equation*}
and the response variable is generated according to
\begin{equation*}
    Y_t = \sin(h t) X_t + \xi_t, \qquad h = 0.3. 
\end{equation*}
Here, the noise $\xi_t$ is generated by an auxiliary OU process with an independent Brownian motion:
\begin{equation*}
    d \xi_t = -1.5 \xi_t\,dt+ 0.25 \,dW_t^\xi. 
\end{equation*}
To rigorously satisfy Assumption \ref{A:data}, the response $Y_t$ must be bounded. One can ensure this condition by applying a smooth truncation function, such as $C_z\tanh(Y_t/C_z)$. We omit this truncation in the present simulations for numerical simplicity.

{\bf Time-varying nonlinear model}. In the second setting, we assume that the response $Y_t$ depends on the covariate $X_t$ through a highly nonlinear relationship. The ground-truth data-generating process is defined as
\begin{equation*}
    \begin{aligned}
       dX_t & = - 0.7 X_t\,dt + 0.7 dW_t^X, \\ 
       Y_t &= f_t(X_t) + \xi_t,
    \end{aligned}
\end{equation*}
where $W_t^X$ is a three-dimensional Brownian motion. The noise $\xi_t$ is generated by
\begin{equation*}
    d \xi_t = -5 \xi_t\,dt+ 0.2 \,dW_t^\xi. 
\end{equation*} 
The nonlinear function $f_t(x)$ takes the form of a neural network:
\begin{equation*}
f_t(x) = \frac{2.5}{M} \sum_{m=1}^M \sigma(x, \varphi_t^m),
\qquad M=100,
\end{equation*}
and each component represents a network neuron with a hyperbolic tangent activation function $\phi = \tanh$:
\begin{equation}\label{eq:neuron}
    \sigma(x,\theta)= a\,\phi(w^\top x+b), \qquad   \theta=(a,w,b).
\end{equation}
The true parameter vector $\varphi_t^m$, $m = 1, \ldots, M$, in each neuron $\sigma(x, \varphi_t^m)$ follows an OU process given by 
\begin{equation*}
	d\varphi_t^m = - 0.6 \bigl(\varphi_t^m - \bar{\varphi}^m \bigr) dt + 0.9 dW_t^m.
\end{equation*}
For each instance, we sample entries of the initial values $\varphi_0^m$ and the mean level $\bar{\varphi}^m$ independently from $\mathcal{N}(0,0.8^2)$.

{\bf Neural network architecture}. For both the online and offline optimization tasks, we employ a two-layer neural network model:
\begin{equation*}
x \mapsto \frac{1}{N} \sum^N_{i = 1}\sigma(x, \theta^{i,N}). 
\end{equation*}
We utilize the exact neuron structure $\sigma$ defined in \eqref{eq:neuron}, but with trainable parameters $\theta$.  The default network width is set to $N = 80$. This architecture can be equivalently represented using the empirical measure $\hat\rho_t^N$. The previously discussed smooth truncation technique can be applied to satisfy the standing Assumption \ref{A:sigma}, although it is omitted here for numerical simplicity.

The online learner is trained using the Euler-Maruyama discretization of the interacting particle system \eqref{eq:particleSDE}. To maintain consistency with standard machine learning practices, we include the $i$-th particle's self-interaction term in the empirical sum: 
\begin{equation*}
	\theta_{t_{k+1}}^{i,N}
	= \theta_{t_k}^{i,N}+
	\left[
	-\lambda \theta_{t_k}^{i,N}
	-2\left( \frac{1}{N} \sum^N_{j = 1} \sigma(X_{t_k},\theta^{j, N}_{t_k}) - Y_{t_k}\right)\nabla \sigma(X_{t_k},\theta_{t_k}^{i,N})
	\right]\Delta t
	+\sqrt{2\beta \Delta t}\,\xi_k^{i,N},
\end{equation*}
where the independent noise increments satisfy $\xi_k^{i,N} \sim \mathcal{N}(0,1)$ for $i=1,\dots,N$ and $k = 1,\dots, 1000$. In contrast, the offline learner fits the neural network by minimizing the aggregated loss over the entire training trajectory on the interval $[0,T]$.

\subsection{Out-of-sample performance}
We evaluate the out-of-sample performance across $30$ independent trials. For both online and offline methods, we use $\lambda = 0.1$ and $\beta = 0.02$ in this subsection. For each trial, we generate a training dataset and a test dataset over a time horizon $T=20$ with a step size $\Delta t=0.02$. Let $K=T/\Delta t = 1000$ and $t_k = k\Delta t$ for $k = 1,\dots, K$. In the periodic regression model, the training and test datasets are generated from independent covariate paths sharing the identical deterministic time-varying regression coefficient $\sin(h t)$. In the dynamic neural response model, we first simulate a single trajectory of $\{\varphi_t^m\}_{t\in[0,T],\,m=1,\dots,M}$ representing the ground-truth parameters. Subsequently, we generate two distinct observation streams from this shared parameter path: a training stream and a test stream, each possessing independent input trajectories and observation noise. Consequently, both the online and offline learners are evaluated within the exact same drifting environment, while their generalization capacity is strictly measured on an independent out-of-sample stream.

\begin{table}[ht]
	\centering
	\setlength{\tabcolsep}{4pt}

	\begin{tabular}{llrrrr}
	\toprule
	\multicolumn{6}{c}{\textbf{Panel A: Summary statistics}}\\
	\midrule
	Scenario & Method & Mean & SD & 95\% CI lower & 95\% CI upper \\
	\midrule
	periodic  & offline & 0.53384 & 0.25795 & 0.43752 & 0.63016 \\
	periodic  & online  & 0.24949 & 0.10533 & 0.21016 & 0.28882 \\
	nonlinear & offline & 0.04063 & 0.01145 & 0.03636 & 0.04491 \\
	nonlinear & online  & 0.03447 & 0.00894 & 0.03113 & 0.03781 \\
	\bottomrule
\end{tabular}

	\vspace{0.8em}

	\begin{tabular}{lrrr}
		\toprule
		\multicolumn{4}{c}{\textbf{Panel B: Paired comparison}}\\
		\midrule
		Scenario & Mean difference  & $t$-test $p$-value &  Wilcoxon $p$-value \\
		\midrule
		periodic  & -0.28435  & $1.93\times 10^{-9}$ &  $9.31\times 10^{-9}$ \\
		nonlinear & -0.00617  & 0.00024        &  $3.90\times 10^{-5}$ \\
		\bottomrule
	\end{tabular}

	\caption{Out-of-sample MSE comparison between offline and online training under periodic and nonlinear dynamics. In Panel B, the mean difference is computed as online minus offline, so negative values indicate that the online learner achieves a lower average MSE.}
	\label{tab:offline_online_comparison}
\end{table}

Table \ref{tab:offline_online_comparison} reports the sample means, standard deviations (SDs), $95\%$ confidence intervals (CIs), and paired hypothesis tests for the average mean squared error (MSE) evaluated on the test set:
\begin{equation}\label{eq:MSE}
\frac{1}{K} \sum^K_{k = 1}
\left(
\int \sigma(X_{t_k}^{\mathrm{test}},\theta)\,\hat\rho_{t_k}^N(d\theta)-Y_{t_k}^{\mathrm{test}}
\right)^2.
\end{equation}

The results in Table \ref{tab:offline_online_comparison} demonstrate that the online learning approach outperforms the offline method in both environments. This performance gap, in terms of the out-of-sample MSE, highlights a fundamental limitation of static optimization when applied to non-stationary data streams.

In the periodic regression scenario, the true relationship between $X_t$ and $Y_t$ continuously oscillates. Any static predictor fitted over the entire interval $[0,T]$ is structurally misspecified because it attempts to average out the periodic dynamics. The second moment of the response $Y_t$ is approximately 0.52481. But the offline learner achieves a test error of 0.53384, indicating that the global static fit totally missed the periodic pattern. Conversely, the online learner updates sequentially and successfully adapts to the local regression landscape, dropping the test error to 0.24949. When the data-generating process undergoes significant shifts, the continuous tracking capability of the online learner yields a substantial predictive advantage.

A similar phenomenon occurs in the time-varying nonlinear model. In this scenario, the second moment of the response $Y_t$ is approximately 0.04713. The offline learner yields a test error of 0.04063, whereas the online method achieves a statistically significant reduction, reaching a test error of 0.03447. The absolute magnitude of the test error is smaller here than in the periodic case for two primary reasons. First, the intrinsic magnitude of the response $Y_t$ is smaller. Second, employing the identical neuron architecture $\sigma$ for both the learner and the ground-truth data-generating process reduces structural misspecification errors. Nevertheless, the paired statistical tests in Panel B confirm the superior tracking capability of the online learner.

\subsection{Empirical regrets and parameter effects}
In this subsection, we perform regret analysis on the time-varying nonlinear model only. To evaluate the empirical dynamic regret without entropy, we approximate the continuous-time integral over the interval $[0, T]$ using a discrete grid. Specifically, although the training data trajectory consists of $K=T/\Delta t = 1000$ discrete time steps, we evaluate the integrand at a coarser uniform subgrid with a spacing of $100$ steps to reduce computational overhead.

A computational challenge is that the instantaneous optimal measure $\mu_t^*$ lacks an explicit closed-form expression. We approximate it at each evaluation time $t_{k_l}$ using importance sampling and root finding. With the data observation $(x, y) := (X_{t_{k_l}}, Y_{t_{k_l}})$, let $m$ approximate the integral $\int \sigma(x, \vartheta) \mu^*_{t_{k_l}}(d\vartheta)$. By the self-consistent equation \eqref{eq:mu*}, the exact value $m^*$ is the unique fixed point $m^* = \Phi(m^*)$, where
\begin{equation*}
   \Phi(m) := \frac{\int \sigma(x, \theta) \exp \left( - \frac{\lambda}{2 \beta} |\theta|^2 -\frac{2}{\beta} \sigma(x, \theta) (m - y) \right) d\theta}{\int \exp \left( - \frac{\lambda}{2 \beta} |\theta|^2 -\frac{2}{\beta} \sigma(x, \theta) (m - y) \right) d\theta}.
\end{equation*}
To approximate $\Phi(m)$, we draw $N_{\mathrm{IS}} = 20000$ independent samples $\theta^{(1)},\dots,\theta^{(N_{\mathrm{IS}})}$ from the prior distribution $\mathcal{N}(0, \frac{\beta}{\lambda}I_d)$. For any trial value $m$, we compute the normalized importance weights:
\begin{equation*}
    w_i(m) = \frac{\exp\!\left(-\frac{2}{\beta}(m-y)\sigma(x,\theta^{(i)})\right)}{\sum_{j=1}^{N_{\mathrm{IS}}} \exp\!\left(-\frac{2}{\beta}(m-y)\sigma(x,\theta^{(j)})\right)}, \qquad i=1,\dots,N_{\mathrm{IS}}.
\end{equation*}
The map $\Phi(m)$ is then approximated by the weighted sum:
\begin{equation*}
    \widehat{\Phi}(m) := \sum_{i=1}^{N_{\mathrm{IS}}} w_i(m)\,\sigma(x,\theta^{(i)}).
\end{equation*}
We find the approximate fixed point $\hat{m}$ by solving the root-finding problem $\widehat{\Phi}(m) - m = 0$. This yields the empirical reference measure:
\begin{equation*}
    \hat{\mu}_{t_{k_l}} := \sum_{i=1}^{N_{\mathrm{IS}}} w_i(\hat{m}) \delta_{\theta^{(i)}}.
\end{equation*}
At each evaluation time $t_{k_l}$, the {\it approximate instantaneous regret} is
\begin{equation*}
    R_{t_{k_l}} = U(\hat{\rho}^N_{t_{k_l}}, Z_{t_{k_l}})-U(\hat{\mu}_{t_{k_l}}, Z_{t_{k_l}}).
\end{equation*}
We obtain the {\it approximate cumulative regret} up to time $t_{k_m}$ by aggregating these instantaneous regrets using the trapezoidal rule:
\begin{equation*}
 \widehat{\cR}_{t_{k_m}} = \sum_{\ell=1}^{m-1} \frac{R_{t_{k_\ell}}+R_{t_{k_{\ell+1}}}}{2}\, \bigl(t_{k_{\ell+1}}-t_{k_\ell}\bigr).
\end{equation*}
Two primary sources of numerical error affect $\widehat{\cR}_t$. First, because computing the reference measure $\hat{\mu}_{t}$ is expensive, we evaluate the integrand on a sparse temporal subgrid instead of every simulated time step. Second, the finite importance sampling size $N_{\mathrm{IS}}$ leads to approximation errors relative to the true continuous limit.

We investigate the impact of the network width $N$, the entropy parameter $\beta$, and the $L^2$ regularization penalty $\lambda$ on the learning dynamics. Figure \ref{fig:regret-four-panel} illustrates the approximate instantaneous regret and the approximate cumulative regret evaluated using the functional $U$ in \eqref{eq:U-def}, which includes the $L^2$ penalty. To isolate the prediction accuracy from the regularization penalty, Figure \ref{fig:nonreg-regret-four-panel} plots the unregularized regret, which excludes the $L^2$ penalty in $U$.

\begin{figure}[H]
\centering
\includegraphics[width=0.9\textwidth]{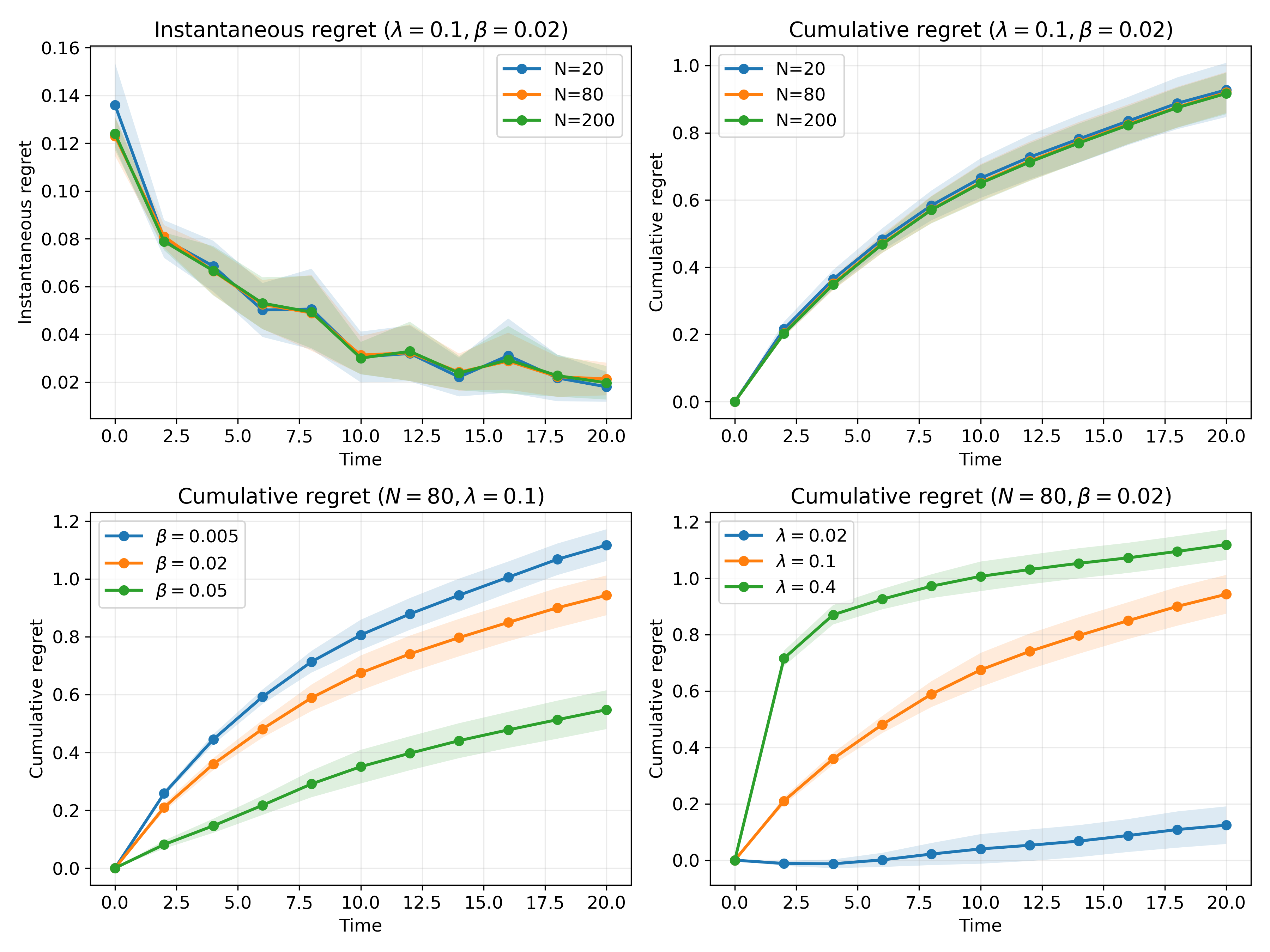}
\caption{Regularized regret dynamics. Top panels: approximate instantaneous (left) and cumulative (right) regret across varying network widths $N$. Bottom panels: approximate cumulative regret across varying exploration noise $\beta$ (left) and regularization penalty $\lambda$ (right). Unmentioned parameters are held constant.}\label{fig:regret-four-panel}
\end{figure}

\begin{figure}[H]
\centering
\includegraphics[width=0.9\textwidth]{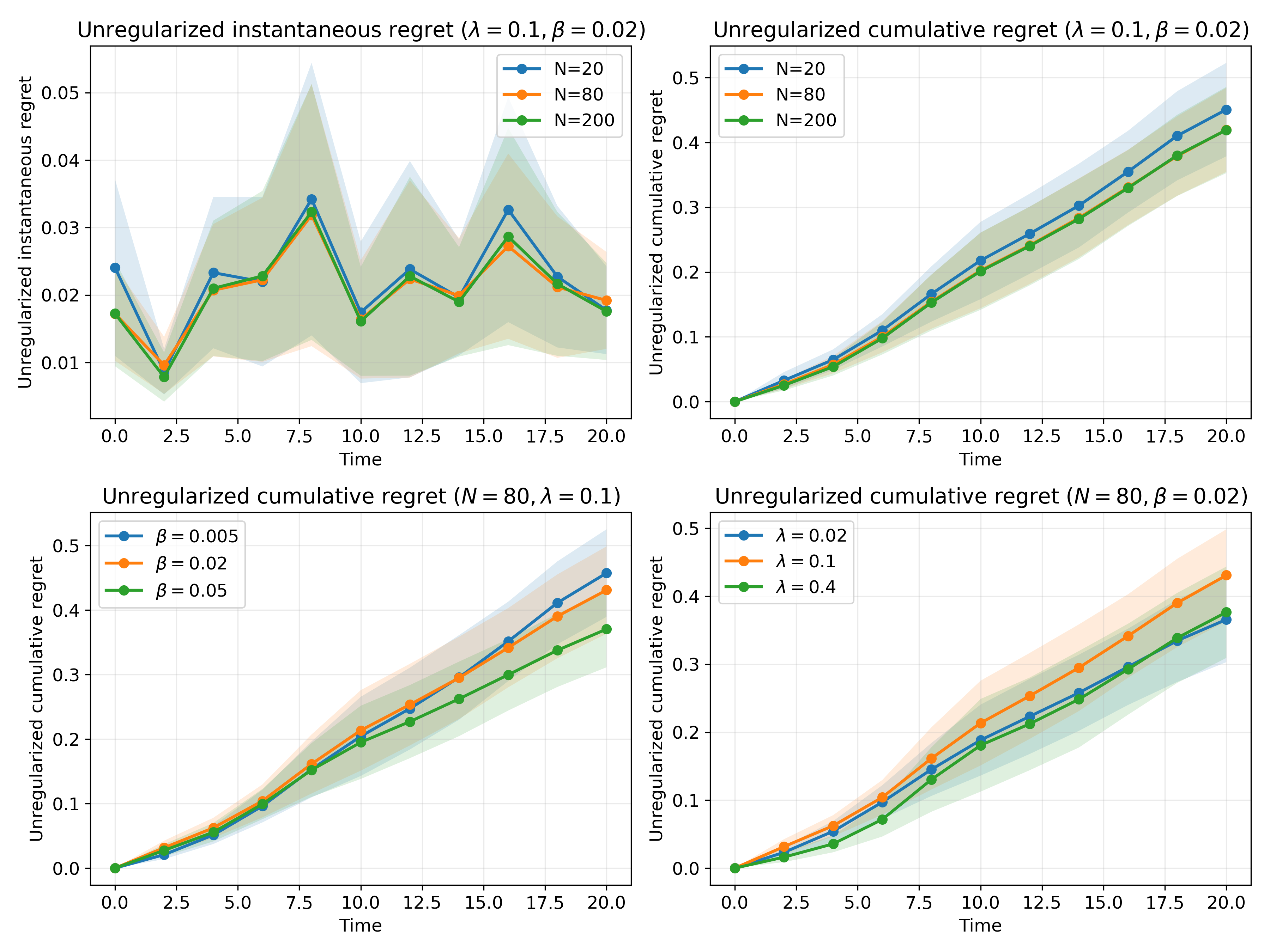}
\caption{Unregularized regret dynamics. The panel layout and parameter variations are identical to Figure \ref{fig:regret-four-panel}.}\label{fig:nonreg-regret-four-panel}
\end{figure}

{\bf Impact of network width ($N$)}. The top-left panels of Figures \ref{fig:regret-four-panel} and \ref{fig:nonreg-regret-four-panel} depict instantaneous regret trajectories. The regularized instantaneous regret initially decreases then stabilizes, while unregularized instantaneous regret remains relatively flat. Crucially, due to variations in the data process, both curves eventually stabilize at non-zero positive levels. This causes the cumulative regrets in the top-right panels to grow linearly, as predicted by Theorem \ref{thm:empircal_regret}. Increasing network width slightly reduces unregularized regret. The minimal gap between $N=80$ and $N=200$ shows diminishing returns. This implies that for sufficiently wide networks, the mean-field tracking error dominates the finite-particle approximation error. Furthermore, the top panels of Figure \ref{fig:regret-four-panel} indicate that regularized regret is insensitive to variations in $N$.

{\bf Impact of exploration noise ($\beta$)}. The bottom-left panels show that increasing the entropy parameter reduces cumulative regret in both settings, indicating that exploration benefits dominate in the low-noise regime. Although the bound in \eqref{eq:empirical_regret} may not be applicable here, it shows a trade-off when choosing $\beta$. Second-moment terms like $M_2(\beta, \lambda)$ and $Q_*(\beta, \lambda)$ grow with the injected noise, reflecting increased variance. In contrast, the coefficients $1/(\alpha\beta)$ and $C_{PL}$ shrink, capturing the advantage of exploration. A sufficient thermal noise helps particles escape local traps, but the theoretical bound also diverges at high temperatures due to variance. These findings demonstrate that an intermediate noise level may optimally balance this trade-off.

{\bf Impact of the $L^2$ penalty ($\lambda$)}. The bottom-right panels reveal distinct trends between the regularized and unregularized regrets. Figure \ref{fig:nonreg-regret-four-panel} demonstrates that $\lambda$ has a non-monotone effect on the unregularized regret. In contrast, Figure \ref{fig:regret-four-panel} shows that a larger $\lambda$ increases the regularized cumulative regret. This increase is primarily driven by the $L^2$ penalty term. When comparing the cases of $\lambda=0.1$ and $\lambda=0.4$, the parameter $\lambda$ increases by a factor of four, yet the resulting regret increases by a factor of approximately $1.2$. It implies that the second moment of the particles is decreasing, as the stronger regularization shrinks the parameters closer to the origin. However, $\lambda$ increases at a faster rate than the second moment decreases within this local regime.

In contrast, the theoretical bound in \eqref{eq:empirical_regret} decreases as $\lambda$ grows, provided both $\beta$ and $\lambda$ are sufficiently large. A primary reason is that while confinement restricts network expressivity, the boundedness of $\sigma$ diminishes the impact of the MSE. This highlights a gap between theoretical guarantees and empirical practice that remains challenging to resolve in future work.

\begin{figure}[ht]
\centering
\includegraphics[width=0.9\textwidth]{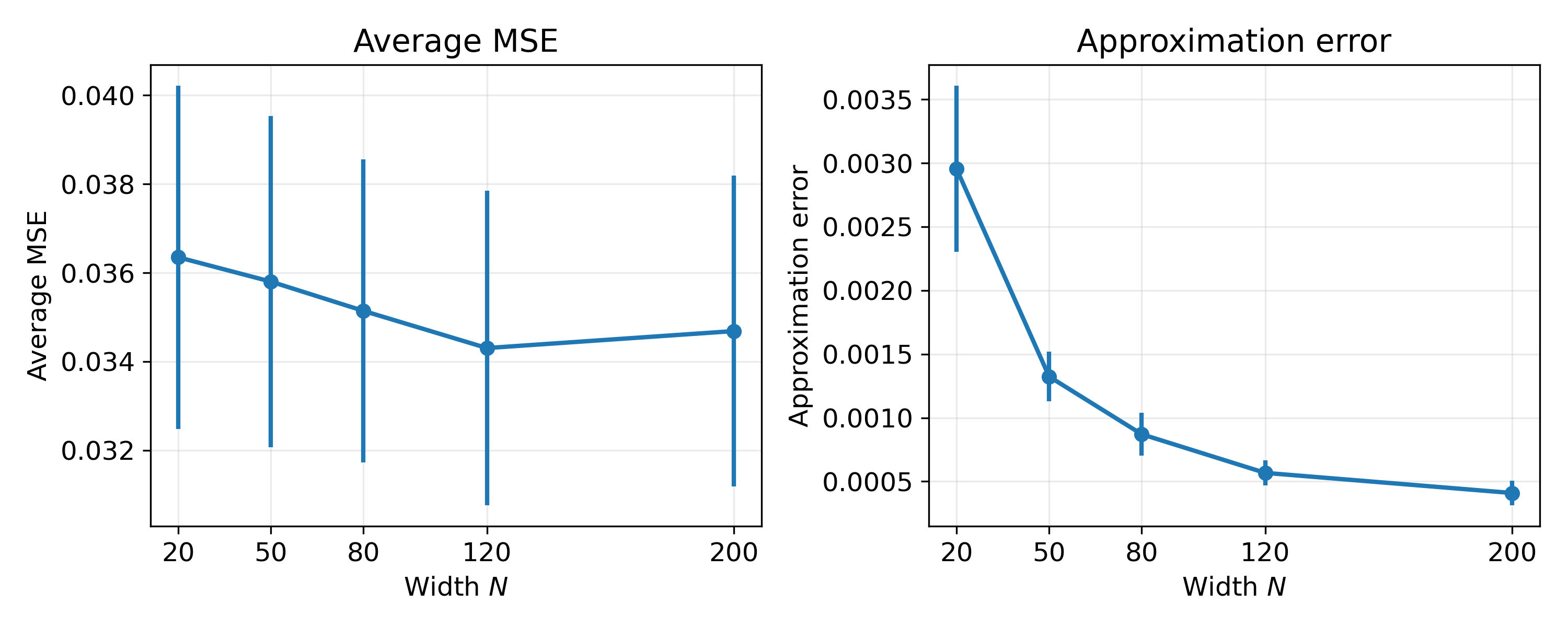}
\caption{Impact of network width $N$ at fixed $(\lambda,\beta)=(0.1, 0.02)$. Average out-of-sample MSE (left) and approximation error relative to the large-width reference (right).}
\label{fig:width-panel}
\end{figure}

\begin{figure}[ht]
\centering
\includegraphics[width=0.9\textwidth]{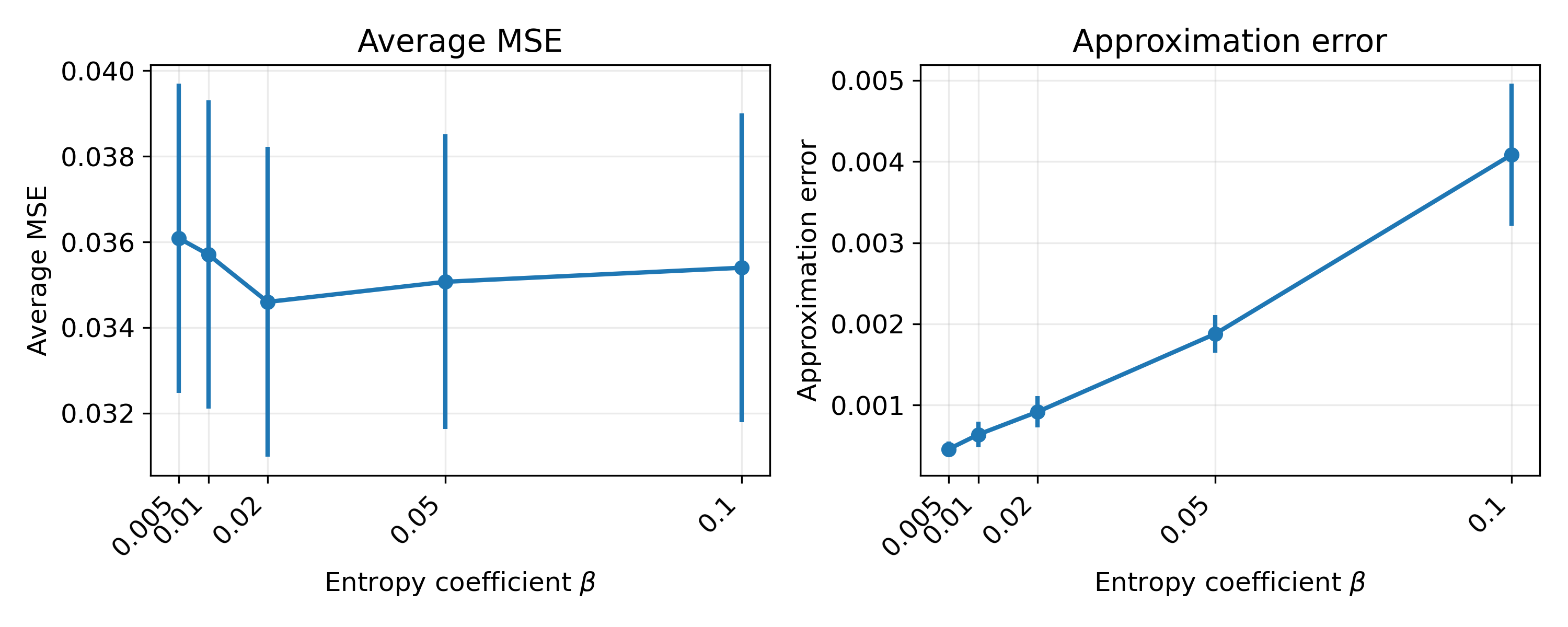}
\caption{Impact of entropy coefficient $\beta$ at fixed $(N, \lambda) = (80, 0.1)$. Average out-of-sample MSE (left) and approximation error relative to the large-width reference (right).}
\label{fig:beta-panel}
\end{figure}

\begin{figure}[ht]
\centering
\includegraphics[width=0.9\textwidth]{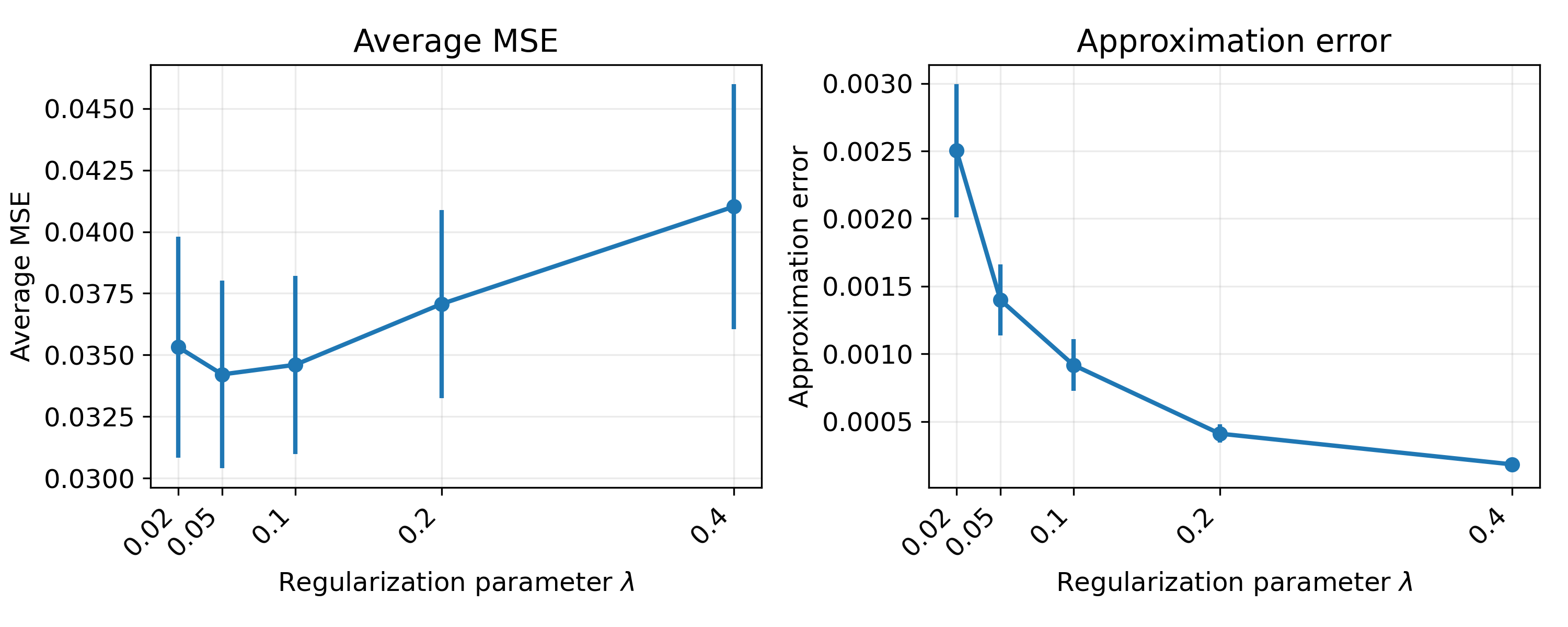}
\caption{Impact of $L^2$ regularization parameter $\lambda$ at fixed $(N, \beta) = ( 80, 0.02)$. Average out-of-sample MSE (left) and approximation error relative to the large-width reference (right).}
\label{fig:lambda-panel}
\end{figure}

Figures \ref{fig:width-panel}, \ref{fig:beta-panel}, and \ref{fig:lambda-panel} evaluate out-of-sample performance across various parameters. All panels display sample means and 95\% confidence intervals computed across repeated trials. The left panels report the average out-of-sample MSE defined in \eqref{eq:MSE}. The right panels present the out-of-sample approximation error relative to the estimated instantaneous optimal measure $\hat{\mu}_{t_k}$, given by
\begin{equation*}
\frac{1}{K}\sum_{k=1}^{K} \Bigl( \int \sigma(X_{t_k}^{\mathrm{test}},\theta)\,\hat\rho_{t_k}^N(d\theta) - \int \sigma(X_{t_k}^{\mathrm{test}},\theta)\,\hat{\mu}_{t_k}(d\theta) \Bigr)^2,
\end{equation*}
where $\hat{\mu}_{t_k}$ can also be obtained by training a large network on data fixed at time $t_k$. A lower approximation error does not necessarily imply better model performance, since the reference measure itself may yield predictions that differ significantly from the observed data in out-of-sample tests. Instead, it quantifies the discrepancy between the two predictions. In contrast, a lower out-of-sample MSE directly indicates better predictive performance.

Figure \ref{fig:width-panel} indicates that both the average out-of-sample MSE and approximation error decrease with $N$. However, the performance improvement is minimal once the network reaches a moderate width. In Figure \ref{fig:beta-panel}, the left panel reveals a U-shaped pattern where the average MSE reaches a minimum near $\beta = 0.02$, confirming that the trade-off in choosing $\beta$ persists out-of-sample. The right panel demonstrates that the approximation error strictly increases with $\beta$. This rise is primarily driven by the higher variance associated with larger entropy coefficients and estimation errors within the reference measure. Finally, the left panel of Figure \ref{fig:lambda-panel} reveals a subtle U-shaped pattern, with the average MSE reaching a minimum near $\lambda = 0.05$. Beyond this optimal point, the error increases prominently because excessive parameter confinement degrades the model's expressivity. Conversely, the right panel demonstrates that the approximation error strictly decreases as $\lambda$ grows. This occurs because stronger regularization tightly constrains the parameters of both models near the origin, inherently reducing the discrepancy between them.

\section{Future directions}\label{sec:future}
This work highlights several directions for future research. First, bridging the continuous-time framework with discrete-time settings requires new technical tools and remains an open problem. Second, relaxing the crucial conditions $\eta \beta^2 > 8 C^2_\sigma C^2_1$ and $\alpha \beta^2 > 8 C^2_\sigma C^2_1$ is an important direction. Exploring alternative concepts of convexity in the Wasserstein space could yield the theoretical insights needed to address this. Third, establishing non-trivial regret lower bounds can further quantify the performance gap. Finally, extending our analysis to broader neural network architectures warrants further investigation.

\bibliographystyle{apalike}
\bibliography{ref.bib}

\appendix

\section{Proofs of Section \ref{sec:properties}}
\subsection{Well-posedness of the McKean-Vlasov SDE and particle system}
\begin{proof}[{\bf Proof of Proposition \ref{prop:MkV}}]
	We prove the three claims separately.
	
	\medskip
	\noindent\textbf{1. Existence, uniqueness, and the uniform \(L^\infty\) bound.}
	We use the fixed-point argument in Proposition 4.3 of \cite{monmarche2024time}.
	Fix \(T \in (0,\infty)\), and define
	\[
	\mathcal X := C\bigl([0,T];L^1(\R^d)\cap \mathcal P(\R^d)\bigr)
	\]
	equipped with the uniform metric
	\[
	d(\mu,\nu):=\sup_{t\in[0,T]}\|\mu_t-\nu_t\|_{L^1},
	\]
	for $\mu,\nu\in \mathcal X$. If \(\mu\in\mathcal X\), then \(\mu_t\) is a probability measure and has a density function, still denoted by \(\mu_t\).
	Clearly, \(\mathcal X\) is a complete metric space.
	
	Given \(\mu\in\mathcal X\), define \(\rho=\mathcal T[\mu]\) to be the unique probability solution to
	\[
	\partial_t \rho_t
	=
	\beta \Delta \rho_t+\nabla\cdot\bigl[\rho_t\, b(\mu_t,Z_t,\theta)\bigr],
	\]
	with
	\[
	b(\mu_t,Z_t,\theta)
	=
	\lambda\theta
	+
	2\bigl(\langle \mu_t,\sigma(X_t,\cdot)\rangle-Y_t\bigr)\nabla\sigma(X_t,\theta).
	\]
	The $L^\infty$ norms below are all over $\theta$. We observe that
	\[
	\begin{aligned}
		\|b(\mu_t,Z_t, \theta) - b(\nu_t, Z_t, \theta)\|_{L^\infty}
		&=
		2\bigl\|
		\langle \mu_t,\sigma(X_t, \cdot)\rangle \nabla\sigma(X_t, \theta)
		-
		\langle \nu_t, \sigma(X_t, \cdot)\rangle \nabla\sigma(X_t, \theta)
		\bigr\|_{L^\infty} \\
		&\le
		2C_\sigma C_1 \|\mu_t-\nu_t\|_{L^1}.
	\end{aligned}
	\]
	Under our assumptions on \(\sigma\), Proposition A.1 of \cite{monmarche2024time} yields
	\[
	\|\rho_t-\rho_t'\|_{L^1}
	\le
	e^{C't}\|\rho_0-\rho_0'\|_{L^1}
	+
	C'\sqrt{t}\sup_{s\in[0,t]}
	\|b(\mu_s,Z_s,\theta)-b(\nu_s,Z_s,\theta)\|_{L^\infty},
	\]
	where \(\rho=\mathcal T[\mu]\), \(\rho'=\mathcal T[\nu]\), and $C'$ is a constant defined in the Proposition A.1 of \cite{monmarche2024time}. Hence,
	\[
	\|\rho_t-\rho_t'\|_{L^1}
	\le
	e^{C't}\|\rho_0-\rho_0'\|_{L^1}
	+
	2C' C_\sigma C_1 \sqrt{t}\sup_{s\in[0,t]}\|\mu_s-\nu_s\|_{L^1}.
	\]
	If we set \(\rho_0=\rho_0'\) and \(t=T\), then
	\[
	\sup_{t\in[0,T]}\|\rho_t-\rho_t'\|_{L^1}
	\le
	2C_1 C_\sigma C' \sqrt{T}\sup_{s\in[0,T]}\|\mu_s-\nu_s\|_{L^1}.
	\]
	Choosing \(T_0>0\) small enough such that \(2C_1 C_\sigma C' \sqrt{T_0} < 1\), we see that \(\mathcal T\) is a contraction on \(\mathcal X\). This proves the existence and uniqueness of a solution \(\rho=\mathcal T[\rho]\) on \([0,T_0]\). Repeating the same argument on \([T_0,2T_0]\), \([2T_0,3T_0]\), and so on, we obtain existence and uniqueness on \([0,\infty)\), with
	\[
	\rho\in C\bigl([0,\infty);L^1(\R^d)\cap\mathcal P(\R^d)\bigr).
	\]
	
	To show the boundedness, we apply \citet[Proposition A.1]{monmarche2024time} again to get
	\[
	\|\rho_t\|_{L^\infty}
	\le
	C_\infty\bigl(\|\rho_0\|_{L^\infty}+1\bigr),
	\]
	for some \(C_\infty\) independent of \(t\) and \(T\). Therefore,
	\[
	\rho\in L^\infty\bigl([0,\infty);L^\infty(\R^d)\bigr).
	\]
	With a similar argument using the stability estimates on each small time interval, we also obtain continuous dependence on the initial value.
	
	\medskip
	\noindent\textbf{2. Uniform moment bound.}
	We can calculate the expectation of the Lyapunov function
	\[
	V_k(\theta):=(1+|\theta|^2)^{k/2}.
	\]
	This gives a differential inequality for
	\[
	\int_{\R^d}V_k(\theta)\,\rho_t(d\theta),
	\]
	from which we obtain the uniform bound on the \(k\)-th moment. The details are omitted here for simplicity.
	
	\medskip
	\noindent\textbf{3. Stability with respect to the approximation.}
	Let $\sigma^\varepsilon=\sigma*\eta^\varepsilon$ be the mollified function. Define
	\[
	b^\varepsilon(\mu_t,Z_t,\theta)
	=
	\lambda\theta
	+
	2\bigl(\langle \mu_t,\sigma^\varepsilon(X_t,\cdot)\rangle-Y_t\bigr)\nabla\sigma^\varepsilon(X_t,\theta).
	\]
	Then \(b^\varepsilon\) satisfies the same conditions as \(b\) in \citet[Proposition A.1]{monmarche2024time}. It follows that
	\[
	\|\rho_t-\rho_t^\varepsilon\|_{L^1}
	\le
	e^{C't}\|\rho_0-\rho_0^\varepsilon\|_{L^1}
	+
	C'\sqrt{t}\sup_{s\in[0,t]}
	\|b(\rho_s,Z_s,\theta)-b^\varepsilon(\rho_s^\varepsilon,Z_s,\theta)\|_{L^\infty},
	\]
	where \(\rho^\varepsilon\) is the unique solution to
	\[
	\partial_t \rho_t^\varepsilon
	=
	\beta \Delta \rho_t^\varepsilon
	+
	\nabla\cdot\bigl[\rho_t^\varepsilon\, b^\varepsilon(\rho_t^\varepsilon,Z_t,\theta)\bigr].
	\]
	By previous arguments, we can show that
	\[
	\rho^\varepsilon\in
	C\bigl([0,\infty);L^1(\R^d)\cap\mathcal P(\R^d)\bigr)
	\cap
	L^\infty\bigl([0,\infty);L^\infty(\R^d)\bigr).
	\]
	
	To control the difference between $\rho_t$ and $\rho^\varepsilon_t$, we prove a bound on the drift difference. By the triangle inequality,
	\[
	\begin{aligned}
		\|b(\rho_t,Z_t, \theta)-b^\varepsilon(\rho_t^\varepsilon,Z_t, \theta)\|_{L^\infty}
		&\le
		\|b(\rho_t,Z_t,\theta)-b(\rho_t^\varepsilon,Z_t,\theta)\|_{L^\infty}
		+\|b(\rho_t^\varepsilon,Z_t,\theta)-b^\varepsilon(\rho_t^\varepsilon,Z_t,\theta)\|_{L^\infty}.
	\end{aligned}
	\]
	The first term is bounded by
	\[
	\|b(\rho_t,Z_t,\theta) - b(\rho_t^\varepsilon, Z_t, \theta)\|_{L^\infty}
	\le
	2 C_1 C_\sigma
	\|\rho_t-\rho_t^\varepsilon\|_{L^1}.
	\]
	For the second term, we recall that the mollified functions satisfy
	\begin{align*}
		& \|\nabla\sigma(X_t, \theta) - \nabla\sigma^\varepsilon(X_t, \theta) \|_{L^\infty} \leq \| D^2_{\theta \theta} \sigma(X_t, \theta) \|_{L^\infty}  \varepsilon \leq C_2 \varepsilon, \\
		& \|\sigma-\sigma^\varepsilon\|_{L^\infty}\leq \| \nabla \sigma(X_t, \theta) \|_{L^\infty} \varepsilon \leq C_1 \varepsilon.
	\end{align*}
	Then
	\begin{equation*}
		\|b(\rho_t^\varepsilon,Z_t,\cdot)-b^\varepsilon(\rho_t^\varepsilon,Z_t,\cdot)\|_{L^\infty}
		\leq (2 C_z C_2 + 2 C_\sigma C_2 + 2 C^2_1) \varepsilon = C_* \varepsilon.
	\end{equation*}
	Therefore, we apply this bound in \citet[Proposition A.1]{monmarche2024time} to obtain 
	\begin{equation}\label{eq:rhovardiff}
		\|\rho_t-\rho_t^\varepsilon\|_{L^1}
		\le
		e^{C't}\|\rho_0-\rho_0^\varepsilon\|_{L^1}
		+
		C'\sqrt{t}\sup_{s\in[0,t]} \Bigl(2 C_1 C_\sigma \|\rho_s-\rho_s^\varepsilon\|_{L^1} + C_* \varepsilon \Bigr).
	\end{equation} 
	Similarly, let \(T_0>0\) be small enough, such that $\delta_{T_0} := 2 C_1 C_\sigma C' \sqrt{T_0} < 1$. The inequality \eqref{eq:rhovardiff} ensures that, for \(t\in[0,T_0]\),
	\begin{equation}\label{eq:rhovardiff2}
		\|\rho_t-\rho_t^\varepsilon\|_{L^1} - \delta_{T_0} \sup_{s\in[0, T_0]} \|\rho_s-\rho_s^\varepsilon\|_{L^1}
		\le
		e^{C' T_0}\|\rho_0-\rho_0^\varepsilon\|_{L^1}
		+ C'\sqrt{T_0} C_* \varepsilon.
	\end{equation}
	We take supremum over \(t\in[0,T_0]\) and note that the right hand side and the second term on the left of \eqref{eq:rhovardiff2} are independent of $t$,
	\begin{equation*}
		(1-\delta_{T_0})\sup_{s\in[0, T_0]}\|\rho_s-\rho_s^\varepsilon\|_{L^1}
		\le
		e^{C' T_0}\|\rho_0 - \rho_0^\varepsilon\|_{L^1}
		+ C'\sqrt{T_0} C_* \varepsilon.
	\end{equation*}
	Since \(\delta_{T_0}<1\), we obtain \(\rho_t^\varepsilon\to \rho_t\) in \(L^1\) as \(\varepsilon\to0\), for \(t\in[0,T_0]\). Repeat this argument on \([T_0,2T_0]\), \([2T_0,3T_0]\), and so on. We obtain \(\rho_t^\varepsilon\to\rho_t\) in \(L^1\) for all \(t\ge0\).
	
	Finally, we prove the \(L^\infty\) norm and the \(k\)-th moment bounds for \(\rho^\varepsilon\). The drift
	\[
	b^\varepsilon(\rho_t^\varepsilon,Z_t,\theta)
	=
	\lambda\theta
	+
	2\bigl(\langle \rho_t^\varepsilon,\sigma^\varepsilon(X_t,\cdot)\rangle-Y_t\bigr)\nabla\sigma^\varepsilon(X_t,\theta)
	\]
	is a sum of a linear function and a bounded function. Moreover,
	\[
	\|\nabla\cdot(-b^\varepsilon)\|_{L^\infty}
	=
	\Bigl\|
	\lambda d
	+
	2\bigl(\langle \rho_t^\varepsilon,\sigma^\varepsilon(X_t,\cdot)\rangle-Y_t\bigr)
	\Delta \sigma^\varepsilon(X_t,\theta)
	\Bigr\|_{L^\infty}
	\le \lambda d + 2(C_\sigma + C_z) \sqrt{d} C_2,
	\]
	which is independent of \(\varepsilon\). Applying \citet[Proposition A.1]{monmarche2024time} again gives the \(L^\infty\) bound, while the \(k\)-th moment bound follows from the same Lyapunov argument as above.
\end{proof}

We denote $C^n(\R^d; \R^{d'})$ the space of continuous functions $f: \R^d \to \R^{d'}$ that are $n$ times differentiable and whose partial derivatives of order $n$ are continuous. For $T > 0$, $C^{m, n}([0, T] \times \R^d; \R^{d'})$ is the space of continuous functions $f: [0, T] \times \R^d \to \R^{d'}$ that are $m$ times differentiable in the time and $n$ times differentiable in the space, with continuous partial derivatives of order $m$ in time and of order $n$ in space. Denote $C^n_{Lip}(\R^d; \R^{d'})$ the space of functions $f \in C^n(\R^d; \R^{d'})$ and $f$ is Lipschitz continuous. Similarly, $C^n_{p}(\R^d; \R^{d'})$ denotes the space of functions $f \in C^n(\R^d; \R^{d'})$ and $f$ has polynomial growth. 

Lemma \ref{lem:HJB-lip-hess} extends \citet[Theorem 4.5]{monmarche2024time} to the case where $\varphi_t$ is the sum of a bounded and a Lipschitz part.
\begin{lemma}
	\label{lem:HJB-lip-hess}
	Let $T > 0$ and $\beta > 0$. Suppose $u \in C^{1,2}_p([0,T] \times \R^d; \R)$ is a classical solution to the HJB equation
	\begin{equation}
		\partial_t u_t = \beta \Delta u_t - \beta | \nabla u_t |^2
		+ \tilde{b}_t \cdot \nabla u_t + \varphi_t,
	\end{equation} 
	for some $\tilde{b} \in C^{0,2}_p ([0,T] \times \R^d; \R^d)$ and $\varphi \in C^{0,2}_p ([0,T] \times \R^d; \R)$. Assume the initial condition $u_0 \in C^3_{Lip} (\R^d; \R)$. Suppose the drift $\tilde b$ satisfies the weak convexity condition
	\begin{equation}\label{eq:weak_convex}
		(\tilde b_t(x) - \tilde b_t(y)) \cdot (x - y)
		\leq \kappa_{\tilde b}(|x - y|) |x - y|^2
	\end{equation} 
	for some $C^1$-continuous
	$\kappa_{\tilde b} : (0,\infty) \to \mathbb R$
	such that $\int_0^1 r \bigl( \kappa_{\tilde b}(r) \vee 0 \bigr) d r < \infty$
	and $\liminf_{r \to \infty} \kappa_{\tilde b}(r) < 0$.
	Suppose
	$\sup_{t \in [0,T]}\lVert \nabla \tilde{b}_t \rVert_{L^\infty} < \infty$.
	Then, we have the following results:
	\begin{enumerate}
		\item Suppose $\varphi_t =  \varphi_{1, t} + \varphi_{2, t}$ with bounded $\varphi_{1, t}$ and $L_2$-Lipschitz $\varphi_{2, t}$ for all $t \in [0,T]$, then
		\begin{equation}
			\label{eq:Lip-grad}
			\| \nabla u_t \|_{L^\infty}
			\leq C e^{-c t} \| \nabla u_0 \|_{L^\infty}
			+ C \int_0^t e^{-cs} \Big( \frac{ \| \varphi_{1, t - s} \|_{L^\infty}}{\sqrt{s \wedge 1}} + L_2 \Big)
			ds, \quad t \in [0, T],
		\end{equation}
		where $C$, $c > 0$ and depend only on $\kappa_{\tilde b}$ and $\beta$.
		\item
		If additionally, $\nabla \varphi_t \in L^\infty$ for all $t \in [0,T]$,
		then 
		\begin{equation}\label{eq:HJB-global-hess}
			\begin{aligned}
				\| D^2 u_t \|_{L^\infty}
				\leq & \frac{C'e^{-c't}}{\sqrt{ t \wedge 1}}
				\| \nabla u_0 \|_{L^\infty} + \int_0^t \frac{C' e^{-c'v}}{\sqrt{v \wedge 1}}
				\big( \| \nabla \varphi_{t-v} \|_{L^\infty}
				+ \| \nabla \tilde b_{t-v} \cdot \nabla u_{t-v} \|_{L^\infty} \big)
				dv,
			\end{aligned}
		\end{equation}
		for all $t \in [0,T]$, where $C'$, $c' > 0$ and depend only on $\kappa_{\tilde b}$, $\beta$, $\| \nabla u_0 \|_{L^\infty}$, and $\sup_{t \in [0,T]} \| \varphi_t \|_{L^\infty}$.
	\end{enumerate}
\end{lemma}

\begin{proof}[{\bf Proof of Lemma \ref{lem:HJB-lip-hess}}]
	A direct calculation shows that \(u_t\) satisfies
	\[
	\partial_t u_t
	=
	\beta \Delta u_t - \beta|\nabla u_t|^2+\tilde b_t\cdot \nabla u_t+\varphi_t
	=
	\beta \Delta u_t+\inf_{\alpha}\bigl\{ \beta |\alpha|^2 - 2 \beta \alpha\cdot \nabla u_t\bigr\}
	+\tilde b_t\cdot \nabla u_t+\varphi_t.
	\]
	It can be understood as a stochastic optimal control problem with the objective given by
	\[
	u(t,x)
	=
	\inf_{\alpha}\,
	\mathbb E\Bigl[
	u_0(X_t^{\alpha,x})
	+
	\int_0^t
	\bigl(
	\varphi_{t-s}(X_s^{\alpha,x})+ \beta|\alpha_s|^2
	\bigr)\,ds
	\Bigr],
	\]
	where the state process solves
	\[
	\left\{
	\begin{aligned}
		dX_s^{\alpha,x}
		&=
		\bigl(\tilde b_{t-s}(X_s^{\alpha,x})-2\beta\alpha_s\bigr)\,ds
		+
		\sqrt{2\beta} dB_s,
		\qquad s\in[0,t],\\
		X_0^{\alpha,x}&=x.
	\end{aligned}
	\right.
	\]
	
	Let \(\alpha^\ast\) be an optimal control for the starting point \(y\), which exists thanks to the quadratic structure. Using the same control \(\alpha^\ast\) for the process starting from \(x\), we obtain
	\[
	u(t,x)-u(t,y)
	\le
	\mathbb E\Bigl[
	u_0(X_t^{\alpha^\ast,x})-u_0(X_t^{\alpha^\ast,y})
	+
	\int_0^t
	\bigl(
	\varphi_{t-s}(X_s^{\alpha^\ast,x})
	-
	\varphi_{t-s}(X_s^{\alpha^\ast,y})
	\bigr)\,ds
	\Bigr].
	\]
	By swapping \(x\) and \(y\), we get
	\[
	|u(t,x)-u(t,y)|
	\le
	\sup_{\alpha}
	\mathbb E\Bigl[
	|u_0(X_t^{\alpha,x})-u_0(X_t^{\alpha,y})|
	+
	\int_0^t
	|\varphi_{t-s}(X_s^{\alpha,x})-\varphi_{t-s}(X_s^{\alpha,y})|\,ds
	\Bigr].
	\]
	
	We now follow the reflection coupling technique. Two Brownian motions \(B^x\) and \(B^y\) are coupled up to the coupling time $\tau:=\inf\{s\ge 0:\ X_s^{\alpha,x}=X_s^{\alpha,y}\}$,
	namely,
	\[
	dB_s^y
	=
	\left(
	I
	-
	2\frac{(X_s^{\alpha,y}-X_s^{\alpha,x})(X_s^{\alpha,y}-X_s^{\alpha,x})^\top}
	{|X_s^{\alpha,y}-X_s^{\alpha,x}|^2}
	\right)dB_s^x,
	\qquad s\le \tau,
	\]
	and \(dB_s^x=dB_s^y\) for \(s>\tau\). The controlled state processes follow
	\[
	\left\{
	\begin{aligned}
		dX_s^{\alpha,x}
		&=
		\bigl(\tilde b_{t-s}(X_s^{\alpha,x})-2\alpha_s\bigr)\,ds
		+
		\sqrt{2\beta} dB_s^x, \quad s \in[0, t], \\
		X_0^{\alpha,x}&=x,
	\end{aligned}
	\right.
	\]
	and
	\[
	\left\{
	\begin{aligned}
		dX_s^{\alpha,y}
		&=
		\bigl(\tilde b_{t-s}(X_s^{\alpha,y})-2\alpha_s\bigr)\,ds
		+
		\sqrt{2\beta} dB_s^y,  \quad s \in[0, t],  \\
		X_0^{\alpha,y}&=y.
	\end{aligned}
	\right.
	\]
	Then the difference process satisfies
	\[
	d\bigl(X_s^{\alpha,x}-X_s^{\alpha,y}\bigr)
	=
	\bigl(
	\tilde b_{t-s}(X_s^{\alpha,x})-\tilde b_{t-s}(X_s^{\alpha,y})
	\bigr)\,ds
	+
	2\sqrt{2 \beta}d\tilde{W}_s,
	\]
	where \(\tilde{W}\) is a Brownian motion by L\'evy's characterization. By the reflection-coupling argument of \cite{eberle2016reflection} and \citet[Theorem 4.5]{monmarche2024time}, the value of $|X_s^{\alpha,x} - X_s^{\alpha,y}|$ is stochastically dominated by $(r_s)_{s \geq 0}$:
	\begin{equation*}
		d r_s = r_s \kappa_{\tilde b}(r_s) ds + 2 \sqrt{2\beta} d \tilde{W}_s.
	\end{equation*}
	Moreover, \cite{eberle2016reflection} and \citet[Theorem 4.5]{monmarche2024time} obtained that
	\begin{equation}\label{eq:Eberle}
		\p(r_s > 0) = \mathbb P(\tau>s)
		\le
		\frac{C|x-y|e^{-cs}}{\sqrt{s\wedge 1}},
		\qquad
		\mathbb E\bigl[|X_s^{\alpha,x}-X_s^{\alpha,y}|\bigr]
		\le
		C|x-y|e^{-cs}.
	\end{equation}
	The constants $C$ and $c$ depend on $\kappa_{\tilde b}$ and $\beta$, as derived in \cite{eberle2016reflection}. 
	
	Next, we bound the difference of running costs. It can be decomposed as
	\[
	|\varphi_{t-s}(X_s^{\alpha,x})-\varphi_{t-s}(X_s^{\alpha,y})|
	\le
	|\varphi_{1,t-s}(X_s^{\alpha,x})-\varphi_{1,t-s}(X_s^{\alpha,y})|
	+
	|\varphi_{2,t-s}(X_s^{\alpha,x})-\varphi_{2,t-s}(X_s^{\alpha,y})|.
	\]
	\begin{itemize}
		\item[(1)] Since \(\varphi_1\) is bounded,
		\[
		|\varphi_{1,t-s}(X_s^{\alpha,x})-\varphi_{1,t-s}(X_s^{\alpha,y})|
		\le
		2M_1\,\mathbf 1_{\{s<\tau\}},
		\qquad
		M_1:=\sup_{r\in[0,t]}\|\varphi_{1,r}\|_{L^\infty}.
		\]
		Then
		\[
		\mathbb E\Bigl[
		|\varphi_{1,t-s}(X_s^{\alpha,x})-\varphi_{1,t-s}(X_s^{\alpha,y})|
		\Bigr]
		\le
		2M_1\,\mathbb P(\tau>s).
		\]
		\item[(2)] Since \(\varphi_2\) is Lipschitz, we obtain
		\[
		|\varphi_{2,t-s}(X_s^{\alpha,x})-\varphi_{2,t-s}(X_s^{\alpha,y})|
		\le
		L_2\,|X_s^{\alpha,x}-X_s^{\alpha,y}|,
		\]
		which leads to 
		\[
		\mathbb E\Bigl[
		|\varphi_{2,t-s}(X_s^{\alpha,x})-\varphi_{2,t-s}(X_s^{\alpha,y})|
		\Bigr]
		\le
		L_2\,\mathbb E\bigl[|X_s^{\alpha,x}-X_s^{\alpha,y}|\bigr].
		\]
		We can bound it with \eqref{eq:Eberle}.
	\end{itemize}
	For the initial terms, we use the Lipschitz continuity of \(u_0\) and \eqref{eq:Eberle} to prove
	\[
	\mathbb E\Bigl[
	|u_0(X_t^{\alpha,x})-u_0(X_t^{\alpha,y})|
	\Bigr]
	\le
	\|\nabla u_0\|_{L^\infty}\,
	\mathbb E\bigl[|X_t^{\alpha,x}-X_t^{\alpha,y}|\bigr]
	\le
	C\|\nabla u_0\|_{L^\infty}|x-y|e^{-ct}.
	\]
	
	Combining the above estimates yields
	\[
	\begin{aligned}
		|u(t,x)-u(t,y)|
		&\le
		C\|\nabla u_0\|_{L^\infty}|x-y|e^{-ct}
		+
		\int_0^t
		\left(
		2M_1\frac{C|x-y|e^{-cs}}{\sqrt{s\wedge 1}}
		+
		L_2\,C|x-y|e^{-cs}
		\right)\,ds \\
		&=
		|x-y|\,C\left(
		\|\nabla u_0\|_{L^\infty}e^{-ct}
		+
		\int_0^t
		e^{-cs}
		\left(
		\frac{2M_1}{\sqrt{s\wedge 1}}+L_2
		\right)\,ds
		\right).
	\end{aligned}
	\]
	Hence
	\[
	\|\nabla u_t\|_{L^\infty}
	\le
	C e^{-ct}\|\nabla u_0\|_{L^\infty}
	+
	C\int_0^t
	e^{-cs}
	\left(
	\frac{\|\varphi_{1,t-s}\|_{L^\infty}}{\sqrt{s\wedge 1}}+L_2
	\right)\,ds,
	\]
	which is exactly \eqref{eq:Lip-grad}.
	
	For the Hessian estimate, the argument is the same as in \citet[Theorem 4.5]{monmarche2024time} and \cite{conforti2023coupling}. 
    
\end{proof}

\begin{proof}[{\bf Proof of Theorem \ref{thm:Lip-grad}}]
	Define
	\[
	\sigma^\varepsilon(X_t,\theta):=(\sigma(X_t,\cdot)*\eta^\varepsilon)(\theta),
	\qquad
	K_t^\varepsilon(\mu):=-2\bigl(\langle \mu,\sigma^\varepsilon(X_t,\cdot)\rangle-Y_t\bigr).
	\]
	We smooth the initial condition by setting
	\[
	\rho_0^\varepsilon(\theta)\propto \exp\bigl(-(u_0*\eta^\varepsilon)(\theta)\bigr)\,\gamma_0(\theta),
	\qquad \text{where} \quad 
	u_0=-\log\frac{\rho_0}{\gamma_0}.
	\]
	The approximated Fokker--Planck equation is
	\[
	\partial_t \rho_t^\varepsilon
	=
	\beta \Delta \rho_t^\varepsilon
	+
	\nabla\cdot\Bigl(
	\rho_t^\varepsilon\bigl[
	\lambda\theta
	+
	2\bigl(\langle \rho_t^\varepsilon,\sigma^\varepsilon(X_t,\cdot)\rangle-Y_t\bigr)
	\nabla \sigma^\varepsilon(X_t,\theta)
	\bigr]
	\Bigr),
	\]
	that is,
	\[
	\partial_t \rho_t^\varepsilon
	=
	\beta \Delta \rho_t^\varepsilon
	-
	\nabla\cdot\Bigl(
	\rho_t^\varepsilon\bigl[
	K_t^\varepsilon(\rho_t^\varepsilon)\nabla \sigma^\varepsilon(X_t,\theta)
	-\lambda\theta
	\bigr]
	\Bigr).
	\]
	
	By Proposition \ref{prop:MkV} (3), together with the fact that \(u_0\) has linear growth and \(\gamma_0\) has quadratic decay, we have
	\[
	\rho_0^\varepsilon \in L^1(\R^d)\cap L^\infty(\R^d)\cap \mathcal P(\R^d).
	\]
	Hence, the approximated Fokker--Planck equation has a unique solution
	\[
	\rho^\varepsilon \in C\bigl([0,\infty);L^1(\R^d)\cap \mathcal P(\R^d)\bigr)
	\cap L^\infty\bigl([0, \infty);L^\infty(\R^d)\bigr).
	\]
	Moreover, Scheff\'e's lemma shows that \(\rho_0^\varepsilon\to \rho_0\) in \(L^1\), and a simliar growth rate argument as above shows that \(\sup_\varepsilon \|\rho_0^\varepsilon\|_{L^\infty}<\infty\). Then Proposition \ref{prop:MkV} (3) implies that
	\[
	\rho_t^\varepsilon \to \rho_t \quad\text{ in }L^1(\R^d), \qquad \forall \; t\ge 0.
	\]
	
	Next, we consider $u_t^\varepsilon:=-\log(\rho_t^\varepsilon/\gamma_0)$. It satisfies the HJB equation
	\begin{equation*}
        \partial_t u_t^\varepsilon = \beta \Delta u_t^\varepsilon - \beta |\nabla u_t^\varepsilon|^2 + \tilde b_t^\varepsilon \cdot \nabla u_t^\varepsilon + \varphi_t^\varepsilon,
    \end{equation*}
	where
	\begin{align*}
		\tilde b_t^\varepsilon(\theta) & = -\lambda\theta + 2\bigl(\langle \rho_t^\varepsilon,\sigma^\varepsilon(X_t,\cdot)\rangle-Y_t\bigr)
		\nabla \sigma^\varepsilon(X_t,\theta), \\
		\varphi_t^\varepsilon(\theta)
		& =  K_t^\varepsilon(\rho_t^\varepsilon)\Delta \sigma^\varepsilon(X_t,\theta)
		- \frac{\lambda}{\beta} \theta\cdot \nabla \sigma^\varepsilon(X_t,\theta)\,
		K_t^\varepsilon(\rho_t^\varepsilon).
	\end{align*}
	
	We use \citet[Proposition 3.1]{conforti2023coupling} to claim that \(u_t^\varepsilon\) is a classical solution to the approximated HJB equation. Indeed, comparing our problem with the notation there,
	the initial value is \(u_0^\varepsilon\), the running cost is \(\varphi_{t-s}^\varepsilon + \beta|\alpha_s|^2\), and the drift is \(\tilde b_{t-s}^\varepsilon - 2\beta\alpha_s\). Moreover,
	\begin{itemize}
		\item[(1)] Thanks to the properties of mollification, \(u_0^\varepsilon \in C^3(\R^d)\) and \(u_0^\varepsilon\) is Lipschitz.
		\item[(2)] \citet[Equation (5)]{conforti2023coupling} and our weak convexity condition \eqref{eq:weak_convex} differ by a negative sign and whether $\kappa_{\tilde b^\varepsilon}$ contains $\beta$. With our notation, we can use 
		\begin{equation}\label{eq:kappa}
			\kappa_{\tilde b^\varepsilon}(r) = - \lambda + \frac{C_*}{r}, \quad r > 0,
		\end{equation}
		where $C_* > 0$ only depends on $C_1$, $C_\sigma$, and $C_z$.
		\item[(3)] Since \(\sigma\) is bounded, the derivatives \(D^k \sigma^\varepsilon\) are bounded by constants depending on \(\varepsilon\). Our Assumption \ref{A:sigma} ensures that $D^3_{\theta \theta \theta} \sigma^\varepsilon \cdot \theta$ and $D^2_{\theta\theta} \sigma^\varepsilon \cdot \theta$ are bounded. Then Assumptions 1.1 and 1.2 in \cite{conforti2023coupling} are satisfied.
	\end{itemize}
	Therefore, \citet[Proposition 3.1]{conforti2023coupling} showed that \(u^\varepsilon\) is a classical solution with polynomial growth, namely $u^\varepsilon \in C_p^{1,2}\bigl([0,T]\times \R^d\bigr)$. Besides, $\tilde b^\varepsilon \in C_p^{0,2}$, $\varphi^\varepsilon \in C_p^{0,2}$, and $ \sup_{t\in[0,T]}\|\nabla \tilde b_t^\varepsilon\|_{L^\infty}$ is bounded. We can apply Lemma \ref{lem:HJB-lip-hess} to obtain
	\begin{equation}\label{eq:grad_est}
		\|\nabla u_t^\varepsilon\|_{L^\infty}
		\le
		Ce^{-ct}\|\nabla u_0^\varepsilon\|_{L^\infty}
		+
		C\int_0^t e^{-cs}
		\left(
		\frac{\|\varphi_{1,t-s}^\varepsilon\|_{L^\infty}}{\sqrt{s\wedge 1}}
		+
		L_2^\varepsilon
		\right)\,ds,
	\end{equation}
	where $\varphi_{1,t}^\varepsilon= K_t^\varepsilon(\rho_t^\varepsilon)\Delta \sigma^\varepsilon(X_t,\theta)$ and \(L_2^\varepsilon\) is the Lipschitz constant of $\varphi_{2,t}^\varepsilon = -\frac{\lambda}{\beta} \theta\cdot \nabla \sigma^\varepsilon(X_t,\theta)\, K_t^\varepsilon(\rho_t^\varepsilon)$. It remains to show the right-hand side can use constants independent of $\varepsilon$:
	\begin{itemize}
		\item[(1)]  \(\tilde b^\varepsilon\) satisfies the weak convexity condition with the function $\kappa_{\tilde b^\varepsilon}$ in \eqref{eq:kappa} that does not depend on $\varepsilon$. Then $C$ and $c$ in \eqref{eq:grad_est} do not depend on $\varepsilon$. However, a side remark is that they rely at least on $\beta$ and $\lambda$. 
		\item[(2)] Thanks to the properties of mollification, $u_0^\varepsilon$ is Lipschitz with the same constant of $u_0$.
		\item[(3)] The function $\varphi_{1,t}^\varepsilon= K_t^\varepsilon(\rho_t^\varepsilon)\Delta \sigma^\varepsilon(X_t,\theta)$ is bounded uniformly in \(\varepsilon\). To check the Lipschitz constant $L_2^\varepsilon$, we find
		\[
		\nabla\bigl(\theta\cdot \nabla \sigma^\varepsilon(X_t,\theta)\bigr)
		=
		\nabla \sigma^\varepsilon(X_t,\theta)
		+
		D_{\theta\theta}^2 \sigma^\varepsilon(X_t,\theta)\cdot \theta
		\]
		is bounded uniformly in \(\varepsilon\).
	\end{itemize}
	Hence the right-hand side of \eqref{eq:grad_est} is independent of \(\varepsilon\). Similarly, since \(\nabla \varphi_t^\varepsilon \in L^\infty\), Lemma \ref{lem:HJB-lip-hess} yields
	\begin{equation}\label{eq:Hessian_est}
		\|D^2 u_t^\varepsilon\|_{L^\infty}
		\le
		\frac{C'e^{-c't}}{\sqrt{t\wedge 1}}
		\|\nabla u_0^\varepsilon\|_{L^\infty}
		+
		\int_0^t
		\frac{C'e^{-c'v}}{\sqrt{v\wedge 1}}
		\left(
		\|\nabla \varphi_{t-v}^\varepsilon\|_{L^\infty}
		+
		\|\nabla \tilde b_{t-v}^\varepsilon \cdot \nabla u_{t-v}^\varepsilon\|_{L^\infty}
		\right)\,dv.
	\end{equation}
	Again, the right-hand side is bounded by a constant independent of \(\varepsilon\). 
	
	We now pass to the limit \(\varepsilon\to 0\). Since \(\nabla u_t^\varepsilon\)  and \(D^2 u_t^\varepsilon\) are uniformly bounded on every compact set of $\theta$ for each fixed \(t>0\), the Arzel\`a--Ascoli theorem and a diagonal argument provide a subsequence \(\varepsilon_k\to 0\) such that, for every \(t>0\),
	\[
	u_t^{\varepsilon_k}\to \bar u_t
	\quad\text{and}\quad
	\nabla u_t^{\varepsilon_k}\to \nabla \bar u_t
	\]
	uniformly on compact subsets of \(\R^d\). Since \(\rho_t^{\varepsilon_k}\to \rho_t\) in \(L^1\), up to a further subsequence we also have almost everywhere convergence. Then $\bar u_t = -\log\frac{\rho_t}{\gamma_0} = u_t$. The same gradient estimate holds for $u_t$ with a generic constant $C$:
	\[
	\|\nabla u_t\|_{L^\infty}\le C,
	\qquad t>0.
	\]
	Moreover, we emphasize that $\nabla u_t$ is in the classical sense here.
	
	For the Hessian, Banach--Alaoglu's theorem yields a subsequence such that $D^2 u_t^{\varepsilon_k}\rightharpoonup^\ast H_t$ in $L^\infty$. The limit \(H_t\) actually equals to the weak Hessian of \(u_t\). By the weak-\(^\ast\) lower semicontinuity of the \(L^\infty\)-norm and the estimate \eqref{eq:Hessian_est}, we have
	\[
	\|D^2 u_t\|_{L^\infty}
	\le
	\liminf_{k\to\infty}\|D^2 u_t^{\varepsilon_k}\|_{L^\infty}
	\le
	\frac{C}{\sqrt{t\wedge 1}},
	\qquad t>0.
	\]
	
	For the linear growth rate \eqref{eq:logrho}, since $\nabla u_t = -\nabla \log \rho_t + \nabla \log \gamma_0 = - \nabla \log \rho_t -\lambda \theta/\beta$, we obtain
	\begin{equation*}
		|\nabla \log \rho_t(\theta)| \le \|\nabla u_t\|_{L^\infty} + \frac{\lambda}{\beta}|\theta| \le L_\rho(1+|\theta|),
	\end{equation*}
	for some constant \(L_\rho>0\). This proves the first claim.
	
	For last two claims, our problem satisfies the conditions (2.2) and (2.4) of \cite{chen2025uniform}. Indeed, they follow from the fact that \(\sigma(X_t,\cdot)\) is Lipschitz and \(D^2_{\theta\theta} \sigma\) is bounded. The proof is the same as in \citet[Proposition 4.6 (2) and (3)]{chen2025uniform}. We omit it here.
\end{proof}

\subsection{Time-uniform LSI}
\begin{proof}[{\bf Proof of Lemma \ref{lem:rhotLSI}}]
	We verify that our case satisfies the assumptions in \citet[Theorem 1.4]{monmarche2024time}. They set their diffusion coefficient $\sigma = 1$. Hence, we consider the transformation $\tilde{\theta}_t = \frac{\theta_t}{\sqrt{\beta}}$. Then \eqref{eq:mean-fieldSDE} reduces to 
	\begin{equation}
		d\tilde{\theta}_t = \left[ - \lambda \tilde{\theta}_t - 2 \frac{\nabla \sigma(X_t, \sqrt{\beta} \tilde{\theta}_t)}{\sqrt{\beta}} \Big( \int \sigma(X_t, \sqrt{\beta} \vartheta) \tilde{\rho}_t(d\vartheta) - Y_t \Big) \right] dt + \sqrt{2} d B_t,
	\end{equation}
	where $\tilde{\rho}_t$ is the distribution of $\tilde{\theta}_t$ and $\nabla \sigma(X_t, \sqrt{\beta} \tilde{\theta}_t)$ denotes the derivative w.r.t $\theta$ and substitutes in $\sqrt{\beta} \tilde{\theta}_t$, instead of the derivative on $\tilde{\theta}$. In the notations of \cite{monmarche2024time}, the drift is written as
	\begin{align*}
		b_t (\tilde{\theta}) & = a_0 (\tilde{\theta}) +	g_t(\tilde{\theta}), \\
		a_0 (\tilde{\theta}) & = - \lambda \tilde{\theta}, \quad g_t(\tilde{\theta})  =  - 2 \frac{\nabla \sigma(X_t, \sqrt{\beta} \tilde{\theta})}{\sqrt{\beta}} \Big( \int \sigma(X_t, \sqrt{\beta} \vartheta) \tilde{\rho}_t(d\vartheta) - Y_t \Big).
	\end{align*}
	For simplicity, we define $K_t = - 2(\int \sigma(X_t, \sqrt{\beta} \vartheta) \tilde{\rho}_t(d\vartheta) - Y_t)$, which is bounded by $2 (C_\sigma + C_z)$. The functions $a_0, g_t \in C^1(\R^d, \R^d)$ have bounded derivatives on $\tilde{\theta}$. Indeed, $D_{\tilde{\theta}} g_t(\tilde{\theta}) = K_t D^2_{\theta \theta} \sigma(X_t, \sqrt{\beta} \tilde{\theta})$ is bounded by $2 (C_\sigma + C_z) C_2$.
	
	The generator $\cL_0 = a_0 \cdot \nabla + \Delta$ admits a unique $C^2$ invariant probability density $\tilde{\gamma}_0 \propto \exp(-\frac{\lambda}{2} |\tilde{\theta}|^2)$. Clearly, it satisfies the $LSI(\lambda)$. Moreover, we find
	\begin{align*}
		\tilde{b}_t & := 2D_{\tilde{\theta}} \log \tilde{\gamma}_0 - b_t = -2\lambda\tilde{\theta} - (-\lambda \tilde{\theta} + g_t) = - \lambda \tilde{\theta} - g_t(\tilde{\theta}), \\
		\varphi_t & := - D_{\tilde{\theta}} \cdot g_t - g_t \cdot D_{\tilde{\theta}} \log \tilde{\gamma}_0  = -K_t \Delta \sigma (X_t, \sqrt{\beta} \tilde{\theta}) + \lambda K_t \frac{\nabla \sigma(X_t, \sqrt{\beta} \tilde{\theta})}{\sqrt{\beta}} \cdot \tilde{\theta} =: \varphi_{1, t} + \varphi_{2, t}.
	\end{align*}
	Here, we emphasize $\Delta \sigma$ and $\nabla \sigma$ are derivatives on $\theta$. The function $\varphi_t$ follows the convention in \citet[Theorem 1.4]{monmarche2024time} and differs Lemma \ref{lem:HJB-lip-hess} by a negative sign.  By Assumption \ref{A:sigma}, $\varphi_{1, t}$ is uniformly bounded by $2(C_\sigma + C_z) C_2 \sqrt{d}$ and $\varphi_{2, t}$ is Lipschitz. Since 
	\begin{equation*}
		D_{\tilde{\theta}} \varphi_{2, t}(\tilde{\theta}) = \lambda K_t D^2_{\theta \theta} \sigma(X_t, \sqrt{\beta} \tilde{\theta}) \cdot \tilde{\theta} + \lambda K_t \frac{\nabla \sigma(X_t, \sqrt{\beta} \tilde{\theta})}{\sqrt{\beta}},
	\end{equation*}
	the Lipschitz constant of $\varphi_{2, t}$ can be chosen as $\frac{2\lambda}{\sqrt{\beta}}(C_\sigma + C_z)(C_2 + C_1)$.
	
	For the new drift $\tilde{b}_t$, we need to show that, there exists a constant $c > 0$,
	\begin{equation*}
		(\tilde{b}_t(\hat{x}) - \tilde{b}_t(\hat{y})) \cdot (\hat{x} - \hat{y}) \le -c |\hat{x} - \hat{y}|^2,
	\end{equation*}
	for $\hat{x}, \hat{y} \in \R^d$ satisfying $|\hat{x} - \hat{y}| \ge R$.
	
	In our case, $|g_t|$ is bounded by $C_g := 2(C_\sigma + C_z) C_1/\sqrt{\beta}$. If we set $R = 4 C_g/\lambda$ and $c = \lambda/2$, then
	\begin{align*}
		(\tilde{b}_t(\hat{x}) - \tilde{b}_t(\hat{y})) \cdot (\hat{x} - \hat{y}) & = -\lambda |\hat{x} - \hat{y}|^2 - (\hat{g}_t(\hat{x}) - \hat{g}_t(\hat{y})) \cdot (\hat{x} - \hat{y}) \leq  -\lambda |\hat{x} - \hat{y}|^2 + 2 C_{g} |\hat{x} - \hat{y}| \\
		& \leq (2 C_g/R - \lambda) |\hat{x} - \hat{y}|^2 = - \frac{\lambda}{2} |\hat{x} - \hat{y}|^2,
	\end{align*}  
	when $|\hat{x} - \hat{y}| \ge R$. Moreover, noting that $g_t$ is $L_g$-Lipschitz with $L_g :=2(C_\sigma + C_z)C_2$, we have
	\begin{align*}
		(\tilde{b}_t(\hat{x}) - \tilde{b}_t(\hat{y})) \cdot (\hat{x} - \hat{y}) & \leq  -\lambda |\hat{x} - \hat{y}|^2 + L_g |\hat{x} - \hat{y}|^2  \leq L |\hat{x} - \hat{y}|^2
	\end{align*}  
	for some $L \geq 0$, when $|\hat{x} - \hat{y}| < R$.
	
	Since we simply choose the initial distribution $\rho_0$ as the invariant measure of $\cL_0$, the first part of \citet[Theorem 1.4]{monmarche2024time} shows that there exists $\tilde{\eta} > 0$ such that $\tilde{\rho}_t$ satisfies the $LSI(\tilde{\eta})$. The constant $\tilde{\eta}$ depends on $\beta$, $\lambda$, $C_\sigma$, $C_z$, $C_1$, and $C_2$. 
	
	Finally, we transfer the LSI back to the original density $\rho_t$. For any test function $h: \R^d \to \R$, define $f(\tilde{\theta}) := h(\sqrt{\beta} \tilde{\theta})$. We have
	\begin{equation*}
		\int f^2(\tilde{\theta}) \log(f^2(\tilde{\theta})) d \tilde{\rho}_t - \int f^2(\tilde{\theta}) d \tilde{\rho}_t \log\left( \int f^2(\tilde{\theta}) d \tilde{\rho}_t \right) \leq \frac{2}{\eta_x} \int | D_{\tilde{\theta}} f(\tilde{\theta}) |^2 d \tilde{\rho}_t.
	\end{equation*}
	Since $\theta = \sqrt{\beta} \tilde{\theta}$, the left-hand side is the exact term for $h$ under $\rho_t$. For the right-hand side, $D_{\tilde{\theta}} f(\tilde{\theta}) = \sqrt{\beta} D_\theta h(\sqrt{\beta} \tilde{\theta}) = \sqrt{\beta} D_\theta h(\theta)$. We obtain
	\begin{equation*}
		\int h^2(\theta) \log(h^2(\theta)) d\rho_t - \int h^2(\theta) d\rho_t \log\left( \int h^2(\theta) d\rho_t \right) \leq \frac{2\beta}{\tilde{\eta}} \int | D_\theta h(\theta) |^2 d\rho_t.
	\end{equation*}
	Thus, $\rho_t$ satisfies the $LSI(\eta)$ for all $t \ge 0$, with $\eta = \tilde{\eta}/\beta$.
	
	Clearly, if $\beta$ is lower bounded, the Lipschitz constants and other constants can be chosen independent of $\beta$. Then $\tilde{\eta}$ is also independent of $\beta$. 
\end{proof}

\section{Proofs of Section \ref{sec:mf-regret}}

\subsection{The displacement convexity method}

\begin{proof}[{\bf Proof of Lemma \ref{lem:rho*}}]
	To show the existence of a solution to \eqref{eq:rho*}, we consider the space of continuous functions on $[0, T]$ bounded by $C_\sigma$:
	\begin{equation*}
		\cU := \Big\{ u \in C([0, T]) \, \Big| \sup_{t \in [0, T]} |u(t)| \leq C_\sigma \Big\}.
	\end{equation*}
	Define a functional $f$ that maps $u_{[0, T]} \in \cU$ to a probability density as
	\begin{equation*}
		f(\theta; u_{[0, T]}) := \frac{1}{A(u_{[0, T]})}  \exp \left[ - \frac{\lambda}{2 \beta} |\theta|^2 -\frac{2}{\beta T} \int^T_0  \left( u_t -  Y_t \right) \sigma(X_t, \theta) dt \right],
	\end{equation*} 
	where $A(u_{[0, T]}) := \int \exp \left[ - \frac{\lambda}{2 \beta} |\theta|^2 -\frac{2}{\beta T} \int^T_0  \left( u_t -  Y_t \right) \sigma(X_t, \theta) dt \right] d\theta$.
	
	Consider
	\begin{equation}
		\Phi(t; u_{[0, T]}) := \int \sigma(X_t, \theta) f(\theta; u_{[0, T]}) d \theta. 
	\end{equation}
	Since $\sigma(X_t, \theta)$ is uniformly bounded by $C_\sigma$, the functional $\Phi$ is a map from $\cU$ to itself. The continuity of $\Phi$ in time $t$ follows from the continuity of $X_t$. 
	
	Moreover, $\Phi$ is a compact operator \cite[Section 1.11.2, Definition 11]{zeidler1995applied}. Indeed, $\Phi$ is continuous as an operator between normed spaces. By Arzel\`a-Ascoli theorem, for any bounded sequence $\{u_n\}_{n \geq 1}$, there exists a subsequence $\{u_{n'}\}$ of $\{u_n\}$ such that the sequence $\Phi(\cdot; u_{n'})$ is convergent in $C([0, T])$. Since $\cU$ is a bounded, closed, convex, nonempty subset of a Banach space $C([0, T])$, the Schauder fixed-point theorem \cite[Section 1.15, Theorem 1.C]{zeidler1995applied} shows that $\Phi$ has at least one fixed point $u^*$, such that $\Phi(t; u^*_{[0, T]}) = u^*(t)$. With $\rho^*(\theta) = f(\theta; u^*_{[0, T]})$, we have proved the existence.
	
	To show the uniqueness, we argue by contradiction. Assume there are two different solutions $\rho^*_1$ and $\rho^*_2$. Note that
	\begin{align*}
		H_i(\theta) := \frac{\lambda}{2 \beta} |\theta|^2 + \frac{2}{\beta T} \int^T_0  \left( \ang{\rho^*_i, \sigma(X_t, \cdot)} -  Y_t \right) \sigma(X_t, \theta) dt = - \log \rho^*_i(\theta) - \log A_i, \; i = 1, 2,
	\end{align*}
	where $A_i$ is the normalization constant in $\rho^*_i$.
	
	Denote $u_i(t) := \ang{\rho^*_i, \sigma(X_t, \cdot)}$, $i=1, 2$. Then the first form of $H_i$ implies that
	\begin{equation*}
		H_1(\theta) - H_2(\theta) = \frac{2}{\beta T} \int^T_0 (u_1(t) - u_2(t)) \sigma(X_t, \theta) dt.
	\end{equation*}
	Multiplying both sides by $\rho^*_1 - \rho^*_2$ and integrating over $\theta$,
	\begin{align*}
		\int (\rho^*_1(\theta) - \rho^*_2(\theta)) (H_1(\theta) - H_2(\theta)) d\theta & = \frac{2}{\beta T} \int (\rho^*_1(\theta) - \rho^*_2(\theta)) \int^T_0 (u_1(t) - u_2(t)) \sigma(X_t, \theta) dt d\theta \\
		& = \frac{2}{\beta T} \int^T_0 (u_1(t) - u_2(t))^2 dt \geq 0.
	\end{align*}
	On the other hand, if we use the second form of $H_i$,
	\begin{align*}
		& \int (\rho^*_1(\theta) - \rho^*_2(\theta)) (H_1(\theta) - H_2(\theta)) d\theta \\
		& \quad = \int (\rho^*_1(\theta) - \rho^*_2(\theta)) (- \log \rho^*_1(\theta) - \log A_1 + \log \rho^*_2(\theta) + \log A_2) d\theta \\
		& \quad = - D_{KL}(\rho^*_1 \| \rho^*_2)  - D_{KL}(\rho^*_2 \| \rho^*_1) \leq 0.
	\end{align*}
	Comparing these two inequalities, we must have $\rho^*_1 = \rho^*_2$.
	
	Finally, we prove the optimality of $\rho^*$. For any density $\rho$, we have
	\begin{align*}
		\beta D_{KL}(\rho \| \rho^*) = & \beta \int \rho(\theta) \log \rho(\theta) d\theta  + \beta \log(A) + \frac{\lambda}{2} \int |\theta|^2 \rho(\theta) d \theta  \\
		& + \frac{2}{T} \int^T_0 \ang{\rho, \sigma(X_t, \cdot)} (\ang{\rho^*, \sigma(X_t, \cdot)} - Y_t) dt.
	\end{align*}
	Hence,
	\begin{align*}
		\int^T_0 F(\rho, Z_t) dt = & \int^T_0 \Big(\ang{\rho, \sigma(X_t, \cdot)}^2 - 2 \ang{\rho^*, \sigma(X_t, \cdot)} \ang{\rho, \sigma(X_t, \cdot)}  \Big) dt + \beta T D_{KL}(\rho \| \rho^*) - \beta T \log(A).
	\end{align*}
	It also implies that
	\begin{align*}
		\int^T_0 F(\rho^*, Z_t) dt = & - \int^T_0 \ang{\rho^*, \sigma(X_t, \cdot)}^2  dt - \beta T \log(A).
	\end{align*}
	Therefore,
	\begin{align*}
		\int^T_0 F(\rho, Z_t) dt - \int^T_0 F(\rho^*, Z_t) dt = \int^T_0 ( \ang{\rho, \sigma(X_t, \cdot)} - \ang{\rho^*, \sigma(X_t, \cdot)})^2  dt + \beta T D_{KL}(\rho \| \rho^*) \geq 0.
	\end{align*}
	It ensures the optimality of $\rho^*$.
\end{proof}

\begin{proof}[{\bf Proof of Theorem \ref{thm:dc-regret}}]
	Fix $\rho \in \cP^r_2(\R^d)$. Consider the functional $\phi(\rho_t) := \frac{1}{2}\cW^2_2(\rho, \rho_t)$. \citet[Corollary 10.2.7]{ambrosio2008gradient} proved the differentiability of the Wasserstein distance. The Fr\'echet subdifferential of this specific functional $\phi$ evaluated at $\rho_t$ is the mapping $\theta \mapsto \theta - T^{\rho}_{\rho_t}(\theta)$. Moreover, Property (3) in Theorem \ref{thm:Lip-grad} ensures that we can apply the chain rule in \citet[Proposition 10.3.18]{ambrosio2008gradient} to $\phi(\rho_t)$. Define a vector field 
	\begin{equation*}
		\xi_t(\theta) :=  \beta \nabla \log \rho_t + \nabla\left(\frac{\delta U(\rho_t, Z_t)}{\delta \rho_t} \right).
	\end{equation*}
	Theorem \ref{thm:Lip-grad} proved that $\rho_t$ is $C^1$ in space. Similar to \citet[Lemma 10.4.1]{ambrosio2008gradient}, we can show that $\xi_t(\theta)$ belongs to the subdifferential $\partial F(\rho_t, Z_t)$. The tangent velocity vector $v_t(\theta)$ of $\rho_t$ takes the negative direction of $\xi_t(\theta)$, given by $v_t(\theta) = - \xi_t(\theta)$.	Hence, 
	\begin{equation}
		\frac{d \cW^2_2(\rho, \rho_t)}{dt} = 2 \int \ang{v_t(\theta), \theta - T^\rho_{\rho_t}(\theta)}  \rho_t(d\theta) = 2 \int \ang{\xi_t(\theta), T^\rho_{\rho_t}(\theta) - \theta}  \rho_t(d\theta).
	\end{equation}
	Since we assumed $F(\cdot, Z_t)$ is $L$-geodesically convex,
	\begin{equation}
		\int \ang{\xi_t(\theta), T^\rho_{\rho_t}(\theta) - \theta} \rho_t(d\theta) \leq F(\rho, Z_t) - F(\rho_t, Z_t) - \frac{L}{2} \cW^2_2(\rho_t, \rho).
	\end{equation}
	Then
	\begin{equation*}
		\cW^2_2(\rho, \rho_T) - \cW^2_2(\rho, \rho_0) = \int^T_0 d \cW^2_2(\rho, \rho_t) \leq 2 \int^T_0 \Big(F(\rho, Z_t) - F(\rho_t, Z_t) - \frac{L}{2} \cW^2_2(\rho_t, \rho) \Big) dt.
	\end{equation*}
	Therefore,
	\begin{equation*}
		\int^T_0 F(\rho_t, Z_t) dt - \int^T_0 F(\rho, Z_t) dt \leq \frac{\cW^2_2(\rho, \rho_0) - \cW^2_2(\rho, \rho_T)}{2}  - \frac{L}{2} \int^T_0 \cW^2_2(\rho_t, \rho) dt.
	\end{equation*}
	We set $\rho = \rho^*$ to obtain the desired result.
\end{proof}

\begin{proof}[{\bf Proof of Lemma \ref{lem:dc-suff-cond}}]
	Let $(\rho_s)_{s \in [0,1]}$ be a geodesic connecting two absolutely continuous probability measures $\rho_0$ and $\rho_1$. There is a unique optimal map $\Phi$ from $\rho_0$ to $\rho_1$: $\rho_s = (\Phi_s)_{\#}\rho_0$, where $\Phi_s(\theta) = \theta + s(\Phi(\theta) - \theta)$. We denote $u(\theta) = \Phi(\theta) - \theta$. The condition for displacement convexity is $\frac{d^2}{ds^2} F(\rho_s, z) \geq 0$. 
	
	A classic result by \cite{mccann1997convexity} shows that, the entropy functional is displacement convex, but it is not strictly displacement convex on $\R^d$. Then the entropy term does not provide a positive constant to the lower bound. It only ensures non-negativity. The remaining terms of $F$ satisfy
	\begin{align*}
		\frac{d^2}{ds^2} U(\rho_s, z) = & 2 \Big[  \int \nabla \sigma(x, \Phi_s(\theta)) \cdot u(\theta) \rho_0(\theta) d \theta \Big]^2 \\
		& + 2 \Big(\int \sigma(x, \Phi_s(\theta)) \rho_0(\theta) d \theta - y \Big) \int u(\theta)^\top \nabla^2 \sigma(x, \Phi_s(\theta)) u(\theta) \rho_0(\theta) d \theta \\
		& + \lambda \int | u(\theta) |^2 \rho_0(\theta) d \theta.
	\end{align*}
	Therefore, a sufficient condition for $\frac{d^2}{ds^2} F(\rho_s, z) \geq 0$ is $\lambda \geq 2 \left( C_\sigma + C_z \right) \| \sigma_{\theta \theta}\|_{op, \infty}$.
\end{proof}

\subsection{The PL method for dynamic regret}
\begin{proof}[{\bf Proof of Lemma \ref{lem:mu*LSI}}] 
	The equilibrium $\mu^*_t$ for \eqref{eq:equilibrium} should satisfy 
	\begin{equation*}
		\nabla \cdot \left( \mu^*_t \nabla \left(\frac{\delta F(\mu^*_t, Z_t)}{ \delta \rho}\right) \right) = 0.
	\end{equation*}
	With integration by parts, it implies 
	\begin{equation*}
		\nabla \left(\frac{\delta F(\mu^*_t, Z_t)}{ \delta \rho}\right) = 0 \quad \text{$\mu^*_t$-a.s.}
	\end{equation*}
	Then $\frac{\delta F(\mu^*_t, Z_t)}{ \delta \rho}$ is a constant $\mu^*_t$-a.s. We obtain the self-consistent equation \eqref{eq:mu*}.
	
	To prove the existence and uniqueness of the solution $\mu^*_t$ to \eqref{eq:mu*}, we introduce a probability density:
	\begin{equation*}
		f_t (\theta, u) := \frac{1}{C_t(u)}  \exp \left[ - \frac{\lambda}{2 \beta} |\theta|^2 -\frac{2}{\beta} \sigma(X_t, \theta) \left(u -  Y_t \right)  \right],
	\end{equation*}
	where the function $C_t(u) = \int  \exp \left[ - \frac{\lambda}{2 \beta} |\theta|^2 -\frac{2}{\beta} \sigma(X_t, \theta) \left(u -  Y_t \right)  \right] d \theta$. Note that $\sigma(X_t, \theta)$ is assumed to be bounded. The path $(X_t, Y_t)$ is continuous almost surely and thus bounded. For each given $u \in \R$, the quadratic term $-\frac{\lambda}{2\beta} |\theta|^2$ dominates at infinity. Hence, the partition function $C_t(u)$ is finite and positive for each $u \in \R$. Moreover, $C_t(u)$ is continuous and differentiable in $u$ by applying dominated convergence theorem on compact intervals of $u$.

	Define a function 
	\begin{equation*}
		\Phi_t(u) := \int \sigma(X_t, \theta) f_t(\theta, u) d \theta.
	\end{equation*}
	Since $\sigma$ is bounded by $C_\sigma$, the expectation $\Phi_t(u)$ must also be bounded by $C_\sigma$ for all $ u \in \R$. The function $\Phi$ maps the real line into the compact interval $[-C_\sigma, C_\sigma]$. Moreover, $\Phi_t(u)$ is continuous in $u$ by the dominated convergence theorem. Restricting the domain to $[-C_\sigma, C_\sigma]$, we have a continuous map $\Phi_t$ from $[-C_\sigma, C_\sigma]$ to itself. By Brouwer's fixed point theorem, there exists at least one $u^* \in [-C_\sigma, C_\sigma]$ such that $u^* = \Phi_t(u^*)$. Then $\mu^*_t = f_t (\theta, u^*)$, which proves the existence.

	To show the uniqueness, by the dominated convergence theorem, we can differentiate under the integral sign to obtain
	\begin{equation}
		\begin{aligned} 
			\partial_u \Phi_t(u) & = - \frac{2}{\beta} \left[  \int \sigma(X_t, \theta)^2 f_t(\theta, u) d \theta - \left(  \int \sigma(X_t, \theta) f_t(\theta, u) d \theta \right)^2 \right] \\
			& = - \frac{2}{\beta} \text{Var}_{\theta \sim f_t} [\sigma(X_t, \theta)] \leq 0.
		\end{aligned}
	\end{equation}
	Hence, the fixed point $u^*$ must be unique.
	
	The Bakry-\'Emery criterion shows that Gaussian distribution $ \propto e^{- \frac{\lambda}{2\beta} | \theta |^2}$ satisfies the $LSI(\lambda/\beta)$. Since we have assumed $|\sigma| \leq C_\sigma$ and $|Y_t| \leq C_z$, the perturbation part in $\mu^*_t$ satisfies
	\begin{equation*}
		\left|\frac{2}{\beta}  \sigma(X_t, \theta) \left(  \int \sigma(X_t, \vartheta)\mu^*_t(d\vartheta) -  Y_t \right) \right| \leq \frac{2}{\beta} (C^2_\sigma + C_z C_\sigma).
	\end{equation*}
	The oscillation constant is at most $C_{osc} = \frac{4}{\beta} (C^2_\sigma + C_z C_\sigma)$. By Holley-Stroock perturbation Lemma \ref{lem:HolleyStroock}, $\mu^*_t$ satisfies the $LSI(\lambda e^{-C_{osc}}/\beta)$.
\end{proof}

\begin{proof}[{\bf Proof of Lemma \ref{lem:F_decom}}] 
	First, we rewrite $\mu^*_t$ in \eqref{eq:mu*} as
	\begin{equation}
		\mu^*_t(\theta) = \frac{1}{C_t} \exp \left[  - \frac{V_t(\theta)}{\beta} - \frac{a_t(\theta)}{\beta} \ang{\mu^*_t, \sigma(X_t, \cdot)} \right],
	\end{equation}
	with $V_t(\theta) = \frac{\lambda}{2} |\theta|^2 - 2  \sigma(X_t, \theta) Y_t$ and $ a_t(\theta) = 2 \sigma(X_t, \theta)$.  It simplifies the KL divergence as
	\begin{align*}
		D_{KL} (\rho \parallel \mu^*_t) & = \int \rho \log\rho d\theta - \int \rho \log \mu^*_t d\theta \\
		& = \int \rho \log\rho d\theta + \int \rho \Big(  \frac{V_t(\theta)}{\beta} + \frac{a_t(\theta)}{\beta} \ang{\mu^*_t, \sigma(X_t, \cdot)} + \log C_t \Big)  d\theta.
	\end{align*}
	Therefore,
	\begin{align*}
		\beta \int \rho \log \rho d\theta = \beta 	D_{KL} (\rho \parallel \mu^*_t) - \ang{\rho, V_t} - \ang{\rho, a_t} \ang{\mu^*_t, \sigma(X_t, \cdot)} - \beta \log C_t.
	\end{align*}
	Then,
	\begin{align*}
		F(\rho, Z_t) = (\ang{\rho, \sigma(X_t, \cdot)})^2- 2 \ang{\mu^*_t, \sigma(X_t, \cdot)} \ang{\rho, \sigma(X_t, \cdot)} + \beta 	D_{KL} (\rho \parallel \mu^*_t) - \beta \log C_t.
	\end{align*}
	Similarly, we have
	\begin{align*}
		F(\mu^*_t, Z_t) = - (\ang{\mu^*_t, \sigma(X_t, \cdot)})^2 - \beta \log C_t.
	\end{align*}
	We obtain the desired result by calculating their difference.
\end{proof}

\begin{proof}[{\bf Proof of Lemma \ref{lem:PL}}]
	{\bf Step 1}. For notational simplicity, define $J : =\ang{\rho, \sigma(X_t, \cdot)} - \ang{\mu^*_t, \sigma(X_t, \cdot)}$. Since $\sigma(X_t, \theta)$ is bounded, we have
	\begin{align*}
		| J | = \left| \int \sigma(X_t, \theta) \rho(d\theta) - \sigma(X_t, \theta) \mu^*_t(d\theta) \right|  \leq 2 C_\sigma \| \rho - \mu^*_t \|_{TV}.
	\end{align*}
	By Pinsker's inequality,
	\begin{equation*}
		\| \rho - \mu^*_t \|_{TV} \leq \sqrt{\frac{1}{2} D_{KL}(\rho \parallel \mu^*_t)}.
	\end{equation*}
	Hence,
	\begin{equation*}
		J^2 \leq 4 C^2_\sigma \| \rho - \mu^*_t \|^2_{TV} \leq 2 C^2_\sigma D_{KL}(\rho \parallel \mu^*_t).
	\end{equation*}
	Then Lemma \ref{lem:F_decom} leads to
	\begin{equation}\label{eq:F_KL1}
		F(\rho, Z_t) - F(\mu^*_t, Z_t) \leq ( 2 C^2_\sigma + \beta) D_{KL}(\rho \parallel \mu^*_t).
	\end{equation}

	{\bf Step 2}. To deal with the right-hand side, we note that
	\begin{align*}
		\nabla \left(\frac{\delta F(\rho, Z_t)}{ \delta \rho}\right) & = \beta \nabla \log \rho + \lambda \theta + 2 ( \ang{\rho, \sigma(X_t, \cdot)} - Y_t) \nabla \sigma(X_t, \theta), \\
		\nabla \log \Big( \frac{\rho}{\mu^*_t} \Big) & = \nabla \log \rho + \frac{\lambda}{\beta} \theta + \frac{2}{\beta} (\ang{\mu^*_t, \sigma(X_t, \cdot)} - Y_t) \nabla \sigma(X_t, \theta) .
	\end{align*}
	Then
	\begin{equation}\label{eq:diff_grad1}
		\nabla \left(\frac{\delta F(\rho, Z_t)}{ \delta \rho}\right) = \beta \nabla \log \Big( \frac{\rho}{\mu^*_t} \Big) - 2 \nabla \sigma(X_t, \theta) \ang{\mu^*_t, \sigma(X_t, \cdot)} + 2 \nabla \sigma(X_t, \theta) \ang{\rho, \sigma(X_t, \cdot)}. 
	\end{equation}
	
	For clarity, we define the following shorthand: $v := \nabla \log \left( \frac{\rho}{\mu^*_t} \right)$ and $G := \nabla \left(\frac{\delta F(\rho, Z_t)}{ \delta \rho}\right)$. Then  \eqref{eq:diff_grad1} becomes
	\begin{equation}
		G = \beta v + 2 J \nabla \sigma(X_t, \theta).
	\end{equation}
	
	To obtain a lower bound of the gradient norm $\|G\|_{L^2(\rho)}^2$, we apply the vector inequality $|A+B|^2 \ge (1 - \varepsilon)|A|^2 - \frac{1 - \varepsilon}{\varepsilon}|B|^2$ for any $\varepsilon \in (0,1)$.  It yields
	\begin{align*}
		\|G\|_{L^2(\rho)}^2 &= \| \beta v + 2 J \nabla \sigma(X_t, \theta) \|_{L^2(\rho)}^2 \ge (1 - \varepsilon) \beta^2 \|v\|_{L^2(\rho)}^2 - \frac{1 - \varepsilon}{\varepsilon} \| 2 J \nabla \sigma(X_t, \theta) \|_{L^2(\rho)}^2.
	\end{align*}
	
	Note that $\| v \|_{L^2(\rho)}^2 = I(\rho | \mu^*_t)$ is the relative Fisher information. Since Lemma \ref{lem:mu*LSI} shows that $\mu^*_t$ satisfies the $LSI(\alpha)$, we obtain
	\begin{equation*}
		I(\rho | \mu^*_t) \ge 2 \alpha D_{KL}(\rho \parallel \mu^*_t).
	\end{equation*}
	
	The second term satisfies
	\begin{align*}
		\| 2 J \nabla \sigma(X_t, \theta) \|_{L^2(\rho)}^2 = 4 J^2 \int |\nabla \sigma(X_t, \theta) |^2 \rho(\theta) d\theta \le 4 J^2 C_{1}^2 \le 8 C_{1}^2 C_\sigma^2 D_{KL}(\rho \| \mu^*_t),
	\end{align*}
	where we use the Pinsker's inequality to bound $J$ in the last inequality.
	
	Substitute these back, we obtain
	\begin{equation}\label{eq:Gbnd}
		\|G\|_{L^2(\rho)}^2 \geq (1 - \varepsilon) \left[ 2 \alpha \beta^2 - \frac{8 C_\sigma^2 C_{1}^2}{\varepsilon} \right] D_{KL}(\rho \| \mu^*_t).
	\end{equation}
	
	For the lower bound \eqref{eq:Gbnd} to be positive, we require 
	\begin{equation*}
		2 \alpha \beta^2 > \frac{8 C_\sigma^2 C_{1}^2}{\varepsilon}.
	\end{equation*}

	Combining \eqref{eq:F_KL1} with \eqref{eq:Gbnd}, we can choose the Polyak-Lojasiewicz constant as
	\begin{equation}
		C_{PL} = \frac{2 C_\sigma^2 + \beta}{(1 - \varepsilon) \left( 2 \alpha \beta^2 - \frac{8 C_\sigma^2 C_{1}^2}{\varepsilon} \right)}.
	\end{equation}
	
	One can choose $\varepsilon$ such that the denominator is maximized. But we simply set $\varepsilon = 1/2$ to obtain a simple constant:
	\begin{equation}
		C_{PL} = \frac{2 C_\sigma^2 + \beta}{\alpha \beta^2 - 8 C_\sigma^2 C_{1}^2}.
	\end{equation}
	This is valid if $\alpha \beta^2 > 8 C_\sigma^2 C_{1}^2$.
	
\end{proof}

\begin{proof}[{\bf Proof of Lemma \ref{lem:dFmu*}}]
	{\bf Step 1: Differentiability of $\mu^*(\theta, \cdot)$}. We first prove that $m(z)$ is differentiable. Consider the function $H$ parameterized by constant $m$:
	\begin{equation}
		H(\theta, z; m) := \frac{\lambda}{2}|\theta|^{2} + 2 \sigma(x,\theta) (m - y).
	\end{equation}
	Denote the parameterized probability measure by
	\begin{equation}
		\mu(\theta, z; m) := \frac{e^{-\frac{H(\theta, z; m)}{\beta}}}{\int e^{-\frac{H(\vartheta, z; m)}{\beta}} d\vartheta}.
	\end{equation}
	The mean $m(z)$ is the root of the following function
	\begin{equation}
		\Phi(z, m) := m - \int \sigma(x, \theta) \mu(\theta, z; m) d\theta = 0.
	\end{equation}
	The proof of Lemma \ref{lem:mu*LSI} ensures the existence and uniqueness of $m(z)$. To prove that $m(z)$ is differentiable, it is sufficient to use implicit function theorem to show the local differentiability. The function $\Phi$ is continuously differentiable on $(z, m)$. Moreover, 
	\begin{equation*}
		\frac{\partial \Phi}{\partial m} = 1 + \frac{2}{\beta} \text{Var}_{\theta \sim \mu} \left( \sigma(x, \theta) \right)
	\end{equation*}
	is non-zero. Therefore, $m(z)$ is differentiable. 
	
	We go back to the unparameterized case. The normalization constant in $\mu^*_t$ can be written as $C_t = C(z) := \int e^{-\frac{1}{\beta}H(\theta, z)}d\theta$. Then
	\begin{equation*}
		\mu^{*}(\theta, z) = \frac{1}{C(z)} e^{-\frac{1}{\beta}H(\theta, z)}.
	\end{equation*}
	Recall that $H(\theta, z)$ depends on $m(z)$. The unparameterized density $\mu^*(\theta, z)$ is also differentiable on $z$.
	
	{\bf Step 2: Gradients of $m(z)$}. Since  
	\begin{align*}
		D_z \mu^*(\theta, z) &= \mu^*(\theta, z) D_z \log \mu^*(\theta, z) \quad \text{ and } \quad \log \mu^*(\theta, z) = -\frac{1}{\beta} H(\theta, z) - \log C(z),
	\end{align*} 
	we obtain
	\begin{equation}\label{eq:Dzmu*}
		D_z \mu^*(\theta, z) =  \left( -\frac{1}{\beta}D_z H(\theta, z) + \frac{1}{\beta} \int D_z H(\theta, z) \mu^*(\theta, z)d\theta\right) \mu^*(\theta, z).
	\end{equation}
	The definition of $H$ indicates that
	\begin{equation}\label{eq:DzH}
		D_z H(\theta, z) = 2 D_z \sigma(x, \theta)  (m(z) - y) + 2 \sigma(x, \theta) (D_z  m(z) - D_z  y).
	\end{equation}
	Then
	\begin{align}
		D_{z}m(z) = & \int D_z \sigma(x, \theta) \mu^*(\theta, z) d \theta + \int \sigma(x, \theta) D_z \mu^*(\theta, z) d \theta \label{def:Dzm} \\
		= & \int D_z \sigma(x, \theta) \mu^*(\theta, z) d \theta - \frac{2}{\beta} e(z) \text{Cov}_{\theta \sim \mu^*} \big( \sigma(x, \theta),  D_z \sigma (x, \theta) \big) \nonumber \\
		& - \frac{2}{\beta} (D_z m(z) - D_z y) \text{Var}_{\theta \sim \mu^*}[\sigma(x, \theta)]. \nonumber
	\end{align}
	Rearranging the terms, we find
	\begin{align*}
		D_z m(z) = \frac{1}{1+\frac{2}{\beta} \text{Var}_{\theta \sim \mu^*}[ \sigma(x, \theta)]} \Big[ & \int D_z \sigma(x, \theta) \mu^*(\theta, z) d \theta \\ 
		& - \frac{2e(z)}{\beta} \text{Cov}_{\theta \sim \mu^*} \big( \sigma(x, \theta),  D_z \sigma (x, \theta) \big) \\
		& + \frac{2}{\beta} D_z y \text{Var}_{\theta \sim \mu^*} [\sigma(x, \theta)] \Big].
	\end{align*}
	Splitting it for $x$ and $y$, we obtain the claim as desired.
	
	{\bf Step 3: Formula of $\cS_z$}.
	By definition, for scalar-valued $f$,
	\begin{align*}
		\cS_z(z; f) =& \int f(z, \theta) D_{z}\mu^{*}(\theta, z)d\theta \\
		=& -\frac{1}{\beta} \int f(z, \theta) D_z H(\theta, z)\mu^*(\theta, z) d\theta \\
		& +\frac{1}{\beta}\int f(z, \theta) \mu^{*}(\theta, z)d\theta \int D_z H(\theta, z) \mu^*(\theta, z) d\theta \\
		=& - \frac{1}{\beta} \text{Cov}_{\theta \sim \mu^*} \big( f(z, \theta),  D_z H(\theta, z) \big).
	\end{align*}
	
	{\bf Step 4: Gradients in It\^o's lemma}. 
	To handle the dependence of $\mu^*_t$ on $Z_t$, we introduce 
	\begin{equation}
		G(Z_t) := F(\mu^*(\cdot, Z_t), Z_t) = F(\mu^*_t, Z_t).
	\end{equation}
	Recalling the proof of Lemma \ref{lem:F_decom}, we found that $F(\mu^*_t, Z_t) = - m(Z_t)^2 - \beta \log C(Z_t)$. Then
	\begin{equation*}
		D_z G(Z_t) = -2 m(Z_t) D_z m(Z_t) - \beta\frac{D_z C(Z_t)}{C(Z_t)}, 
	\end{equation*}
	with
	\begin{equation*}
		\frac{D_z C(Z_t)}{C(Z_t)} = -\frac{1}{\beta} \int D_z H(\theta, Z_t) \mu^*(\theta, Z_t) d\theta.
	\end{equation*}
	We can simplify
	\begin{align*}
		- \beta\frac{D_z C(Z_t)}{C(Z_t)}  &= 2 e(Z_t) \int D_z \sigma(X_t, \theta) \mu^*(\theta, Z_t) d \theta + 2 (D_z  m(Z_t) - D_z  Y_t) \int \sigma(X_t, \theta) \mu^*(\theta, Z_t) d \theta. 
	\end{align*}
	Combining terms for $D_z G$, we obtain
	\begin{align}
		D_z G(Z_t) &= -2m(Z_t) D_z m(Z_t) + 2 e(Z_t) \int D_z \sigma(X_t, \theta) \mu^*(\theta, Z_t) d \theta + 2 (D_z  m(Z_t) - D_z  Y_t) m(Z_t)  \nonumber \\
		&=  2 e(Z_t) \ang{ \mu^*(\cdot, Z_t),  D_z \sigma(X_t, \cdot)} - 2 D_z  Y_t m(Z_t). \label{eq:DzG}
	\end{align}
	If we split it for $x$ and $y$,
	\begin{align*}
		D_x G(Z_t) &= 2 (m(Z_t) - Y_t) \ang{ \mu^*(\cdot, Z_t),  D_x \sigma(X_t, \cdot)}, \\
		D_y G(Z_t)&= - 2 m(Z_t).
	\end{align*}
	They are equal to $F_x(\mu^*_t, Z_t)$ and $F_y(\mu^*_t, Z_t)$, respectively. It means the dependence of $\mu^*_t$ and $Z_t$ does not change the gradients.
	
	{\bf Step 5: Hessian in It\^o's lemma}.  Since $D_y \sigma = 0$, it follows from the definition of $\cS_z(\cdot; \sigma)$ that
	\begin{align*}
		D^2_{yy} G(Z_t) &= - 2 D_y m(Z_t) = - 2 \cS_y(Z_t; \sigma) = F_{yy}(\mu^*_t, Z_t) - 2 \cS_y(Z_t; \sigma), \\
		D^2_{xy}G(Z_t) & = - 2 D_x m(Z_t) = - 2 \ang{\mu^*_t, D_x \sigma(X_t, \cdot)} - 2 \cS_x(Z_t; \sigma) \\
		& = F_{xy}(\mu^*_t, Z_t) - 2 \cS_x(Z_t; \sigma),
	\end{align*}
	and
	\begin{align*}
		D^2_{xx} G(Z_t) = & 2 D_x m(Z_t) \ang{ \mu^*(\cdot, Z_t),  D_x \sigma(X_t, \cdot)}^\top + 2 (m(Z_t) - Y_t) \ang{ \mu^*(\cdot, Z_t),  D^2_{xx} \sigma(X_t, \cdot)} \\
		& +  2 (m(Z_t) - Y_t) \int D_x \sigma(X_t, \theta) [D_x \mu^*(\theta, Z_t)]^\top d \theta \\
		= & 2(\ang{\mu^*_t, D_x \sigma(X_t, \cdot)} + \cS_x(Z_t;\sigma)) \ang{ \mu^*(\cdot, Z_t),  D_x \sigma(X_t, \cdot)}^\top \\
		& + 2 (m(Z_t) - Y_t) \ang{ \mu^*(\cdot, Z_t),  D^2_{xx} \sigma(X_t, \cdot)} + +  2 (m(Z_t) - Y_t) \cS_x(Z_t; D_x\sigma) \\
		= & F_{xx}(\mu^*_t, Z_t) + 2 \cS_x(Z_t; \sigma) \ang{ \mu^*(\cdot, Z_t),  D_x \sigma(X_t, \cdot)}^\top + 2 (m(Z_t) - Y_t) \cS_x (Z_t; D_x\sigma).
	\end{align*}
	
	{\bf Step 6: Bounds on terms related to $\cS_z(Z_t ; f)$}. Since
	\begin{align*}
		& \| D_x m(Z_t) \|_{L^\infty} \\
		& = \Big\| \frac{\beta}{\beta + 2\text{Var}_{\theta \sim \mu^*_t} [ \sigma(X_t, \theta)]} \Big[ \int \sigma_x(X_t, \theta) \mu^*(\theta, Z_t) d \theta - \frac{2e(Z_t)}{\beta} \text{Cov}_{\theta \sim \mu^*_t} \big( \sigma(X_t, \theta), \sigma_x(X_t, \theta) \big) \Big] \Big\|_{L^\infty} \\
		& \leq  C_1 + \frac{2}{\beta}(C_\sigma + C_z) C_1 C_\sigma
	\end{align*}
	and
	\begin{align*}
		\| D_y m(Z_t) \|_{L^\infty} & = \Big\| \frac{2  \text{Var}_{\theta \sim \mu^*_t} [\sigma(X_t, \theta)]}{\beta + 2 \text{Var}_{\theta \sim \mu^*_t} [\sigma(X_t, \theta)]} \Big\|_{L^\infty}  \leq \frac{2 C^2_\sigma}{\beta},
	\end{align*}
	we obtain
	\begin{align*}
		\| D_x H(\theta, Z_t) \|_{L^\infty} = & \big\| 2 D_x \sigma(X_t, \theta)  (m(Z_t) - Y_t) + 2 \sigma(X_t, \theta) D_x  m(Z_t) \big\|_{L^\infty} \\
		\leq & 2 C_1 (C_\sigma + C_z) + 2 C_\sigma \Big( C_1 + \frac{2}{\beta}(C_\sigma + C_z) C_1 C_\sigma \Big)
	\end{align*}
	and
	\begin{align*}
		\| D_y H(\theta, Z_t) \|_{L^\infty} = & \big\| 2 \sigma(X_t, \theta) (D_y  m(Z_t) - 1) \big\|_{L^\infty} \\
		\leq & 2 C_\sigma \Big( \frac{2 C^2_\sigma}{\beta} + 1\Big).
	\end{align*}
	
	Therefore,
	\begin{align*}
		\| \cS_x(Z_t; \sigma) \|_{L^\infty} =& \frac{1}{\beta} \Big\| \text{Cov}_{\theta \sim \mu^*_t} \big( \sigma(X_t, \theta),  D_x H(\theta, Z_t) \big) \Big\|_{L^\infty}  \\
		\leq & \frac{C_\sigma}{\beta} \Big[ 2 C_1 (C_\sigma + C_z) + 2 C_\sigma \Big( C_1 + \frac{2}{\beta}(C_\sigma + C_z) C_1 C_\sigma \Big) \Big],
	\end{align*}
	\begin{align*}
		\| \cS_y(Z_t; \sigma) \|_{L^\infty} =& \frac{1}{\beta} \Big\| \text{Cov}_{\theta \sim \mu^*_t} \big( \sigma(X_t, \theta),  D_y H(\theta, Z_t) \big) \Big\|_{L^\infty} \leq  \frac{2 C^2_\sigma}{\beta}  \Big( \frac{2 C^2_\sigma}{\beta} + 1\Big),
	\end{align*}
	\begin{align*}
		\| \cS_x(Z_t; \sigma_x) \|_{L^\infty} =& \frac{1}{\beta} \Big\| \text{Cov}_{\theta \sim \mu^*_t} \big( D_x \sigma(X_t, \theta),  D_x H(\theta, Z_t) \big) \Big\|_{L^\infty}  \\
		\leq & \frac{C_1}{\beta} \Big[ 2 C_1 (C_\sigma + C_z) + 2 C_\sigma \Big( C_1 + \frac{2}{\beta}(C_\sigma + C_z) C_1 C_\sigma \Big) \Big].
	\end{align*}
\end{proof}

\begin{proof}[{\bf Proof of Theorem \ref{thm:PL-regret}}]
	The third property in Theorem \ref{thm:Lip-grad} ensures that we can apply the chain rule in \citet[Proposition 10.3.18]{ambrosio2008gradient} and the It\^o's formula to obtain
	\begin{equation}
		\begin{aligned}
			d F(\rho_t, Z_t) = \int \frac{\delta F(\rho_t, Z_t)}{\delta \rho} \partial_t \rho_t d\theta dt + \cL_z[F] (\rho_t, Z_t).
		\end{aligned}
	\end{equation} 
	With integration by parts, the first term is
	\begin{equation}
		\begin{aligned}
			\int \frac{\delta F(\rho_t, Z_t)}{\delta \rho} \partial_t \rho_t d\theta & = \int  \frac{\delta F(\rho_t, Z_t)}{\delta \rho} \nabla \cdot \left( \rho_t \nabla \left(\frac{\delta F(\rho_t, Z_t)}{ \delta \rho}\right) \right) d\theta  \\
			& = - \int \Big| \nabla \Big(\frac{\delta F(\rho_t, Z_t)}{ \delta \rho}\Big)  \Big|^2 \rho_t d\theta = - \Big\|\nabla \Big( \frac{\delta F(\rho_t, Z_t)}{\delta \rho} \Big) \Big\|_{L^2(\rho_t)}^2.
		\end{aligned}
	\end{equation}
	Lemma \ref{lem:dFmu*} obtained $dF(\mu^*_t, Z_t)$ in \eqref{eq:dFmu*}. 
	
	We have that $F^\top_x(\mu^*_t, Z_t) \Sigma_1 d W_t$, $F_y(\mu^*_t, Z_t) \Sigma_2 dW_t$, $F^\top_x(\rho_t, Z_t) \Sigma_1 d W_t$, and $F_y(\rho_t, Z_t) \Sigma_2 dW_t$ are true martingales. Then the dynamic of
	\begin{equation}
		E(t) := \E[F(\rho_t, Z_t) - F(\mu^*_t, Z_t)]
	\end{equation}
	is given by 
	\begin{equation}\label{eq:dE}
		\begin{aligned}
			d E(t) = & - \E \Big[ \Big\| \nabla \Big( \frac{\delta F(\rho_t, Z_t)}{\delta \rho} \Big) \Big\|_{L^2(\rho_t)}^2 \Big] dt + \E[\cA[F] (\rho_t, Z_t) - \cA[F] (\mu^*_t, Z_t)] dt \\
			& - \E \Big[\tr \big[\cS_x(Z_t; \sigma) \ang{\mu^*_t, \sigma_x(X_t, \cdot)}^\top \Sigma_1 \Sigma^\top_1 \big] \Big] dt \\
			& - \E \Big[ \tr \big[(\ang{\mu^*_t, \sigma(X_t, \cdot)}  - Y_t) \cS_x(Z_t; \sigma_x)  \Sigma_1 \Sigma^\top_1 \big] \Big]dt \\
			& + 2\E[\cS_x(Z_t; \sigma)^\top \Sigma_1 \Sigma^\top_2] dt + \E[\cS_y(Z_t; \sigma) \Sigma_2 \Sigma^\top_2] dt.
		\end{aligned}
	\end{equation}
	
	We calculate an upper bound of $\E[\cA[F] (\rho_t, Z_t) - \cA[F] (\mu^*_t, Z_t)]$ as follows:
	\begin{enumerate}
		\item The $b_1$ term in $\cA[F] (\rho_t, Z_t) - \cA[F] (\mu^*_t, Z_t)$.
		\begin{align*}
			& | (F_x(\rho_t, Z_t) - F_x(\mu^*_t, Z_t))^\top b_1 | \\
			& \leq  2 | b_1 |  | \ang{\rho_t, \sigma(X_t, \cdot)} \ang{\rho_t, \sigma_x (X_t, \cdot)}  - \ang{\mu^*_t, \sigma(X_t, \cdot)} \ang{\mu^*_t, \sigma_x (X_t, \cdot)} | \\
			& \quad + 2 C_z | b_1 | | \ang{\rho_t, \sigma_x(X_t, \cdot)} -  \ang{\mu^*_t, \sigma_x(X_t, \cdot)} |.
		\end{align*}
		Thanks to Assumption \ref{A:sigma} and the strong duality of Wasserstein distance of order 1, we have
		\begin{align*}
			& | \ang{\rho_t, \sigma(X_t, \cdot)} \ang{\rho_t, \sigma_x (X_t, \cdot)}  - \ang{\mu^*_t, \sigma(X_t, \cdot)} \ang{\mu^*_t, \sigma_x (X_t, \cdot)} | \\
			& \leq | \ang{\mu^*_t, \sigma(X_t, \cdot)} | | \ang{\rho_t, \sigma_x (X_t, \cdot)}  -  \ang{\mu^*_t, \sigma_x (X_t, \cdot)} | \\
			& \quad + | \ang{\rho_t, \sigma_x (X_t, \cdot)} | | \ang{\rho_t, \sigma(X_t, \cdot)} - \ang{\mu^*_t, \sigma(X_t, \cdot)} | \\
			& \leq C_\sigma C_2 \cW_1 (\rho_t, \mu^*_t) + C^2_1 \cW_1 (\rho_t, \mu^*_t)
		\end{align*}
		and
		\begin{align*}
			| \ang{\rho_t, \sigma_x(X_t, \cdot)} -  \ang{\mu^*_t, \sigma_x(X_t, \cdot)} | & \leq C_2 \cW_1 (\rho_t, \mu^*_t).
		\end{align*}
		
		\item The $b_2$ term in $\cA[F] (\rho_t, Z_t) - \cA[F] (\mu^*_t, Z_t)$. 
		
		We use $|\sigma_\theta| \leq C_1$ to obtain
		\begin{align*}
			| (F_y(\rho_t, Z_t) - F_y(\mu^*_t, Z_t)) b_2 | & \leq 2 | b_2 | \left| \ang{\rho_t, \sigma(X_t, \cdot)} - \ang{\mu^*_t, \sigma(X_t, \cdot)} \right| \\
			& \leq 2 | b_2 | C_1 \cW_1 (\rho_t, \mu^*_t).
		\end{align*}
		
		\item The $\Sigma_1 \Sigma^\top_1$ term in $\cA[F] (\rho_t, Z_t) - \cA[F] (\mu^*_t, Z_t)$.
		
		We need to bound $\frac{1}{2} \tr [ (F_{xx}(\rho_t, Z_t) - F_{xx}(\mu^*_t, Z_t)) \Sigma_1 \Sigma^\top_1 ]$. It suffices to bound the norm of the difference of the Hessian matrices.
		\begin{align*}
			& | F_{xx}(\rho_t, Z_t) - F_{xx}(\mu^*_t, Z_t) | \\
			& \leq 2 | \ang{\rho_t, \sigma} \ang{\rho_t, \sigma_{xx}} - \ang{\mu^*_t, \sigma} \ang{\mu^*_t, \sigma_{xx}} | \\
			& \quad + 2 | \ang{\rho_t, \sigma_x} \ang{\rho_t, \sigma_x}^\top - \ang{\mu^*_t, \sigma_x} \ang{\mu^*_t, \sigma_x}^\top | \\
			& \quad + 2 |Y_t| | \ang{\rho_t, \sigma_{xx}} - \ang{\mu^*_t, \sigma_{xx}} |.
		\end{align*}
		We handle these three parts separately. For the first part,
		\begin{align*}
			& | \ang{\rho_t, \sigma} \ang{\rho_t, \sigma_{xx}} - \ang{\mu^*_t, \sigma} \ang{\mu^*_t, \sigma_{xx}} | \\
			& \leq | \ang{\rho_t, \sigma} | | \ang{\rho_t, \sigma_{xx}} - \ang{\mu^*_t, \sigma_{xx}} | + | \ang{\mu^*_t, \sigma_{xx}} | | \ang{\rho_t, \sigma} - \ang{\mu^*_t, \sigma} | \\
			& \leq C_\sigma C_3 \cW_1 (\rho_t, \mu^*_t) +  C_1 C_2 \cW_1 (\rho_t, \mu^*_t).
		\end{align*}

		For the second part:
		\begin{align*}
			& | \ang{\rho_t, \sigma_x} \ang{\rho_t, \sigma_x}^\top - \ang{\mu^*_t, \sigma_x} \ang{\mu^*_t, \sigma_x}^\top | \\
			& \leq | \ang{\rho_t, \sigma_x} | | (\ang{\rho_t, \sigma_x} - \ang{\mu^*_t, \sigma_x})^\top | + | \ang{\rho_t, \sigma_x} - \ang{\mu^*_t, \sigma_x} | | \ang{\mu^*_t, \sigma_x}^\top | \\
			& \leq  2 C_1 C_2 \cW_1 (\rho_t, \mu^*_t).
		\end{align*}
		
		For the third part:
		\begin{align*}
			2 |Y_t| | \ang{\rho_t, \sigma_{xx}} - \ang{\mu^*_t, \sigma_{xx}} | \leq 2 C_z C_3 \cW_1 (\rho_t, \mu^*_t).
		\end{align*}
		
		Combining these, the trace term is bounded by
		\begin{align*}
			| \Sigma_1 \Sigma^\top_1 | (C_\sigma C_3 + 3 C_1 C_2 + C_z C_3)  \cW_1 (\rho_t, \mu^*_t).
		\end{align*}
		
		\item The $\Sigma_1 \Sigma^\top_2$ term in $\cA[F] (\rho_t, Z_t) - \cA[F] (\mu^*_t, Z_t)$.
		
		We have
		\begin{align*}
			& | \tr [ (F_{yx}(\rho_t, Z_t) - F_{yx}(\mu^*_t, Z_t)) \Sigma_1 \Sigma^\top_2 ] | \\
			& \leq | \Sigma_1 \Sigma^\top_2 | | F_{yx}(\rho_t, Z_t) - F_{yx}(\mu^*_t, Z_t) | \\
			& = 2 | \Sigma_1 \Sigma^\top_2 | | \ang{\rho_t, \sigma_x} - \ang{\mu^*_t, \sigma_x} | \\
			& \leq 2 | \Sigma_1 \Sigma^\top_2 | C_2 \cW_1 (\rho_t, \mu^*_t) \leq C_2 (| \Sigma_1 \Sigma^\top_1 | + | \Sigma_2 \Sigma^\top_2 |) \cW_1 (\rho_t, \mu^*_t).
		\end{align*}
		In the last inequality, we used $| \Sigma_1 \Sigma^\top_2 | \leq \frac{1}{2}(| \Sigma_1 \Sigma^\top_1 | + | \Sigma_2 \Sigma^\top_2 |)$.
	\end{enumerate}
	
	For notational simplicity, we define
	\begin{align*}
		A(t) & := P_1 |b_1| + 2 C_1 |b_2| + P_2 | \Sigma_1 \Sigma^\top_1 | + C_2 | \Sigma_2 \Sigma^\top_2 |, \\
		P_1 & := 2 (C_2  C_\sigma + C^2_1 + C_2  C_z), \\
		P_2 & := C_\sigma C_3 + 3 C_1 C_2 + C_z C_3 + C_2.
	\end{align*}
	Summing all terms up, we obtain
	\begin{align*}
		& \E[\cA[F] (\rho_t, Z_t) - \cA[F] (\mu^*_t, Z_t)] \\
		& \leq \E[ A(t) \cW_1 (\rho_t, \mu^*_t) ] \\
		& \leq \E[A(t)^2]^{1/2} \E[\cW_1 (\rho_t, \mu^*_t)^2]^{1/2} \\
		& \leq \big(P_1 \E[ | b_1 |^2 ]^{1/2}  + 2 C_1 \E[ | b_2 |^2 ]^{1/2} + P_2 \E[| \Sigma_1 \Sigma^\top_1 |^2]^{1/2} + C_2  \E[| \Sigma_2 \Sigma^\top_2 |^2]^{1/2} \big)  \E[\cW_1 (\rho_t, \mu^*_t)^2]^{1/2} \\
		& \leq (P_1 + 2 C_1 + P_2 + C_2) C_{b, \Sigma} \E[\cW_1 (\rho_t, \mu^*_t)^2]^{1/2} \\
		& = C_* C_{b, \Sigma} \E[\cW_1 (\rho_t, \mu^*_t)^2]^{1/2}.
	\end{align*}
	
	For the terms related to $\cS_z$, we have
	\begin{align*}
		& \Big| - \E \Big[\tr \big[\cS_x(Z_t; \sigma) \ang{\mu^*_t, \sigma_x(X_t, \cdot)}^\top \Sigma_1 \Sigma^\top_1 \big] \Big] - \E \Big[ \tr \big[(\ang{\mu^*_t, \sigma(X_t, \cdot)}  - Y_t) \cS_x(Z_t; \sigma_x)  \Sigma_1 \Sigma^\top_1 \big] \Big] \\
		& \quad + 2 \E[\cS_x(Z_t; \sigma)^\top \Sigma_1 \Sigma^\top_2] + \E[\cS_y(Z_t; \sigma) \Sigma_2 \Sigma^\top_2] \Big| \\
		& \leq  C_* C_{b, \Sigma} \Big(\frac{1}{\beta} + \frac{1}{\beta^2} \Big),
	\end{align*}
	where $C_*$ may be different from the previous one.
	
	Both $\rho_t$ and $\mu^*_t$ have finite moments of order $2$. By Lemma \ref{lem:mu*LSI} and \citet[Corollary 2.2]{otto2000generalization}, $\mu^*_t$ satisfies Talagrand inequality with the same constant $\alpha$. It leads to
	\begin{equation}
		\cW_2(\rho_t, \mu^*_t) \leq \sqrt{\frac{2 D_{KL}(\rho_t \parallel \mu^*_t)}{\alpha}}.
	\end{equation}
	
	Lemma \ref{lem:F_decom} ensures that
	\begin{equation*}
		F(\rho_t, Z_t) - F(\mu^*_t, Z_t) \geq \beta D_{KL}(\rho_t \parallel \mu^*_t).
	\end{equation*}
	
	Hence, 
	\begin{equation}\label{eq:Wass_F}
		\cW_1 (\rho_t, \mu^*_t)^2 \leq \cW_2 (\rho_t, \mu^*_t)^2 \leq \frac{2}{ \alpha \beta} [	F(\rho_t, Z_t) - F(\mu^*_t, Z_t)].
	\end{equation}
	It yields
	\begin{equation}
		\E[\cW_1 (\rho_t, \mu^*_t)^2] \leq \frac{2}{\alpha \beta} \E[ F(\rho_t, Z_t) - F(\mu^*_t, Z_t)] = \frac{2}{\alpha \beta} E(t).
	\end{equation}
	
	Moreover, Lemma \ref{lem:PL} shows that
	\begin{equation}
		\frac{F(\mu^*_t, Z_t) - F(\rho_t, Z_t)}{C_{PL}} \geq - \left\| \nabla \left( \frac{\delta F(\rho_t, Z_t)}{\delta \rho} \right) \right\|_{L^2(\rho_t)}^2.
	\end{equation}
	
	Combining with \eqref{eq:dE}, we obtain
	\begin{align*}
		d E(t) & \leq - \frac{E(t)}{C_{PL}} dt + C_* C_{b, \Sigma} \sqrt{\frac{2}{\alpha \beta} E(t)} dt + C_* C_{b, \Sigma} \Big(\frac{1}{\beta} + \frac{1}{\beta^2} \Big) dt.
	\end{align*}
	We apply Young's inequality $ab \leq \frac{a^2}{2 \varepsilon} + \frac{b^2 \varepsilon}{2}$ with $\varepsilon = C_{PL}$, $a = \sqrt{E(t)}$ and $b = C_* C_{b, \Sigma} \sqrt{2/(\alpha \beta)}$. Then
	\begin{align*}
		d E(t) & \leq - \frac{E(t)}{2C_{PL}} dt + \frac{C^2_* C^2_{b, \Sigma} C_{PL}  }{\alpha \beta} dt + C_* C_{b, \Sigma} \Big(\frac{1}{\beta} + \frac{1}{\beta^2} \Big) dt.
	\end{align*}
	Integrating on both sides, since $E(T) \geq 0$, we have
	\begin{align*}
		\int^T_0 E(t) dt & \leq 2 C_{PL} (E(0) - E(T)) + \frac{2 C^2_* C^2_{b, \Sigma} C^2_{PL}}{\alpha \beta} T + 2 C_* C_{b, \Sigma} C_{PL} \Big(\frac{1}{\beta} + \frac{1}{\beta^2} \Big) T \\
		& \leq 2 C_{PL} E(0) + \frac{C_* C^2_{b, \Sigma} C^2_{PL}}{\alpha \beta} T + C_* C_{b, \Sigma} C_{PL} \Big(\frac{1}{\beta} + \frac{1}{\beta^2} \Big) T,
	\end{align*}
	where we also replace $2C^2_*$ and $2C_*$ by a generic $C_*$. Gr\"onwall's inequality can yield a slightly better bound, while we omit here for simplicity.
\end{proof}

\subsection{The Malliavin approach for static regret}
\begin{proof}[{\bf Proof of Lemma \ref{lem:Drho_exist}}]
	The main idea is to apply implicit function theorem in Banach spaces \cite[Theorem 4.E]{zeidler1995vol2} and Malliavin chain rule in UMD Banach spaces \cite[Proposition 3.8]{pronk2014tools}.
	
	To satisfy the Banach space assumption in \cite{pronk2014tools}, we enlarge the space of $Z$ to $L^2([0, T]; \R^n \times \R)$. Assumption \ref{A:data} ensures that $\int^T_0 |X_t|^2 dt < \infty$ a.s. We do not assume $|Y_t| \leq C_z$ and only impose $ Y \in L^2([0, T]; \R)$ at the beginning. Consider a smooth and bounded truncation function $\chi(\cdot)$ such that $\chi(y) = y$ if $|y| \leq 2C_z$ and $\chi(y) = 2C_z$ when $|y| \geq 3 C_z$.
	
	For any path $Z \in L^2([0, T]; \R^n \times \R)$ and probability density $\rho$, we define a functional $\Phi$ that maps $(Z, \rho)$ to a probability density given by
	\begin{equation*}
		\Phi(Z, \rho)(\theta) := \frac{1}{A(Z, \rho)}  \exp \left[ - H(Z, \rho, \theta) \right],
	\end{equation*} 
	where 
	\begin{align*}
		H(Z, \rho, \theta) & := \frac{\lambda}{2 \beta} |\theta|^2 + \frac{2}{\beta T} \int^T_0  \left( \ang{\rho, \sigma(X_t, \cdot)} -  \chi(Y_t) \right) \sigma(X_t, \theta) dt, \\
		A(Z, \rho) & := \int \exp \left[ - H(Z, \rho, \theta) \right] d\theta.
	\end{align*}
	By an approximation argument, for each $Z \in L^2([0, T]; \R^n \times \R)$, we can extend the existence result in Lemma \ref{lem:rho*} and prove that there exists a solution $\rho^*$ to the self-consistent equation:
	\begin{equation*}
		\rho^* - \Phi(Z, \rho^*) = 0.
	\end{equation*}
	The uniqueness of $\rho^*$ follows in the same way like Lemma \ref{lem:rho*}.
	
	Denote $(\hat{Z}, \hat{\rho})$ as a reference point such that the self-consistent equation holds:
	\begin{equation*}
		0 = G(\hat{Z}, \hat{\rho}) := \hat{\rho} - \Phi(\hat{Z}, \hat{\rho}).
	\end{equation*}
	
	We introduce several functional spaces as follows:
	\begin{align*}
		\cB & := L^2([0, T]; \R^n \times \R), \quad E_0 := \{  \psi \in L^2(\R^d, \hat{\rho}): \int \psi(\theta) \hat{\rho}(\theta) d\theta  = 0\},
	\end{align*}
	where $L^2(\R^d, \hat{\rho}) := \{ \psi: \int \psi^2(\theta) \hat{\rho} (\theta) d\theta < \infty\}$.
	
	Consider the functional 
	\begin{equation}
		J(Z, \psi)(\theta) := \frac{G(Z, (1 + \psi)\hat{\rho})(\theta)}{\hat{\rho}(\theta)}.
	\end{equation}
	Since $\int J(Z, \psi)(\theta) \hat{\rho}(\theta) d\theta = 0$ and $\int J^2(Z, \psi)(\theta) \hat{\rho}(\theta) d\theta < \infty$, the functional $J$ is a mapping from $\cB \times E_0$ to $E_0$.
	
	Next, we calculate the Fr\'echet derivative of $J$ at $(\hat{Z}, 0)$, acting on $\psi \in E_0$. Note that
	\begin{align*}
		\partial_{\rho} H(\hat{Z}, \hat{\rho}, \theta)[h] = \frac{2}{\beta T} \int^T_0 \ang{h, \sigma(\hat{X}_s, \cdot)} \sigma(\hat{X}_s, \theta) ds =: \cK(h, \theta).
	\end{align*}
	Moreover, 
	\begin{align*}
		\partial_{\rho} A(\hat{Z}, \hat{\rho})[h] & = - \int A(\hat{Z}, \hat{\rho}) \frac{\exp[- H(\hat{Z}, \hat{\rho}, \theta)]}{A(\hat{Z}, \hat{\rho})} \cK(h, \theta) d\theta = - A(\hat{Z}, \hat{\rho}) \ang{\hat{\rho}, \cK(h, \cdot)} \\
		\partial_{\rho} \Phi(\hat{Z}, \hat{\rho})[h](\theta) & = - \hat{\rho}(\theta) \cK(h, \theta) + \hat{\rho}(\theta) \ang{\hat{\rho}, \cK(h, \cdot)}.
	\end{align*}
	Hence,
	\begin{equation}
		\partial_{\psi} J(\hat{Z}, 0)[\xi](\theta) = \xi(\theta) + \cK(\hat{\rho}\xi, \theta) - \ang{\hat{\rho}, \cK(\hat{\rho} \xi, \cdot)}.
	\end{equation}
	To show that it has a bounded inverse, we calculate the inner product under the Hilbert space $E_0$ with $\xi \in E_0$:
	\begin{align*}
		& \int \xi(\theta) \partial_{\psi} J(\hat{Z}, 0)[\xi](\theta) \hat{\rho}(\theta) d \theta \\
		& \quad = \int \xi^2(\theta) \hat{\rho}(\theta) d\theta + \int \xi(\theta) \cK(\hat{\rho} \xi, \theta) \hat{\rho}(\theta) d\theta - \int \xi(\theta) \ang{\hat{\rho}, \cK(\hat{\rho} \xi, \cdot)} \hat{\rho} (\theta) d\theta \\
		& \quad = \int \xi^2(\theta) \hat{\rho}(\theta) d\theta + \frac{2}{\beta T} \int^T_0 \ang{\hat{\rho} \xi, \sigma(\hat{X}_s, \cdot)}^2 ds \geq \int  \xi^2(\theta) \hat{\rho}(\theta) d\theta.
	\end{align*}
	Here, the third term in the second line is zero because $\xi \in E_0$. Besides, the associated bilinear form $a(\xi, \eta) := \int \eta(\theta) \partial_\psi J(\hat{Z}, 0)[\xi](\theta) \hat{\rho}(\theta) d\theta$ is continuous.
	
	It shows that the operator $\partial_{\psi} J(\hat{Z}, 0): E_0 \to E_0$ is strictly coercive \cite[Definition 13.28]{kress2014}. By Lax-Milgram theorem \cite[Theorem 13.29]{kress2014}, the operator $\partial_{\psi} J(\hat{Z}, 0)$ has a bounded inverse. 
	
	Moreover, it is direct to show that the functional $J$ is continuously Fr\'echet differentiable on a neighborhood of $(\hat{Z}, 0)$. Then \citet[Theorem 4.E]{zeidler1995vol2} ensures that, on a neighborhood of $(\hat{Z}, 0)$, there is exactly one $\psi(Z) \in E_0$ such that $J(Z, \psi(Z)) = 0$, which means
	\begin{equation*}
		\hat{\rho} + \hat{\rho} \psi(Z) = \Phi(Z, \hat{\rho} + \hat{\rho} \psi(Z)).
	\end{equation*}
	By the uniqueness of $\rho^*$, we must have
	\begin{equation}
		\rho^*(\theta, Z) = \hat{\rho}(\theta) + \hat{\rho}(\theta) \psi(Z)(\theta).
	\end{equation}
	In fact, the existence and uniqueness of $\rho^*$ ensure the global existence of $\psi(Z)$.

	\citet[Theorem 4.E (ii)]{zeidler1995vol2} shows that $\psi(Z)$ is also continuously Fr\'echet differentiable on a neighborhood of $\hat{Z}$. In particular, 
	\begin{equation}
		\partial_Z \psi(Z) = - [\partial_{\psi} J(Z, \psi(Z))]^{-1} \partial_Z J(Z, \psi(Z)).
	\end{equation}
	By a similar argument with Lax-Milgram theorem, we can show that $\partial_{\psi} J(Z, \psi(Z))$ has a bounded inverse, with a constant uniform in $Z$. It ensures the global differentiability of $\psi(Z)$. Moreover, the operator $\partial_Z \psi$ is bounded.
	
	Note that $E_0$ is a Hilbert space and hence a UMD Banach space. We can use \citet[Proposition 3.8]{pronk2014tools} to obtain the Malliavin derivative
	\begin{equation}
		\cD \psi(Z) = \partial_Z \psi(Z) \cD Z.
	\end{equation}

	Because $\hat{\rho}(\theta)$ is bounded, the multiplication operator $M_{\hat{\rho}}(\psi) := \hat{\rho}\psi$ is a bounded linear operator from $E_0$ to $L^2(\R^d)$. As a linear operator, its Fr\'echet derivative is simply itself, which is trivially bounded and continuous. Therefore, we can apply the chain rule from \citet[Proposition 3.8]{pronk2014tools} a second time to the affine transformation $\rho^*(\cdot, Z) = \hat{\rho}(\cdot) + M_{\hat{\rho}}(\psi(Z))$. Because the reference density $\hat{\rho}(\theta)$ is fixed and does not change with respect the chance parameter $\omega^W \in \Omega^W$, this guarantees that $\rho^* \in \bD^{1, 2}(L^2(\R^d))$ and
	\begin{equation*}
		\cD \rho^*(\theta) = M_{\hat{\rho}}(\cD \psi(Z)) = \hat{\rho}(\theta) \partial_{Z} \psi(Z) \cD Z = \partial_{Z} \rho^*(\theta, Z) \cD Z.
	\end{equation*}
	
	To recover the Malliavin derivative of the untruncated system, we recall that our data $|Y_t| \leq C_z$ by assumption. The truncation $\chi(Y_t) = Y_t$ a.s. By the local property of the Malliavin derivative, the derivative of the original system coincides with the derivative of the truncated system a.s., obtaining the existence of the Malliavin derivative.
\end{proof}

\begin{proof}[{\bf Proof of Lemma \ref{lem:psi}}]
	With $$H(Z, \rho^*, \theta) = \frac{\lambda}{2 \beta} |\theta|^2 + \frac{2}{\beta T} \int^T_0  \left( \ang{\rho^*, \sigma(X_t, \cdot)} -  Y_t \right) \sigma(X_t, \theta) dt, $$
	we have
	\begin{equation}
		\psi_t(\theta) = - \cD_t H + \ang{\rho^*, \cD_t H}.
	\end{equation}
	A further calculation obtains
	\begin{align*}
		\cD_t H = \cK(\psi_t, Z, \theta) + S_t(\theta),
	\end{align*}
	which leads to the Fredholm integral equation. The existence and uniqueness of a solution $\psi_t$ is also guaranteed by the Lax-Milgram theorem.
	
	Note that $\ang{\rho^*, \psi_t(\cdot)} = \cD_t \int \rho^*(\theta)d\theta = 0$ and 
	\begin{align*}
		\ang{\rho^*, \psi_t(\cdot) \cK(\psi_t, Z, \cdot)} & = \frac{2}{\beta T} \int^T_0 \ang{\rho^*, \psi_t(\cdot) \sigma(X_s, \cdot)}^2 ds \geq 0.
	\end{align*}
	We obtain 
	\begin{align*}
		\ang{\rho^*, \psi^2_t} & = -\ang{\rho^*, \psi_t(\cdot) \cK(\psi_t, Z, \cdot)} - \ang{\rho^*,  \psi_t(\cdot) S_t(\cdot)} + \ang{\rho^*, \psi_t(\cdot)} \ang{\rho^*, \cK(\psi_t, Z, \cdot) + S_t(\cdot)} \\
		& = -\ang{\rho^*, \psi_t(\cdot) \cK(\psi_t, Z, \cdot)} - \ang{\rho^*,  \psi_t(\cdot) S_t(\cdot)} \\
		& \leq - \ang{\rho^*,  \psi_t(\cdot) S_t(\cdot)}.
	\end{align*}
	By Cauchy-Schwarz inequality, the right-hand side can be replaced as
	\begin{align*}
		\int \psi^2_t(\theta) \rho^*(\theta) d\theta \leq \Big( \int S^2_t(\theta) \rho^*(\theta) d\theta \Big)^{1/2} \Big( \int \psi^2_t(\theta) \rho^*(\theta) d\theta \Big)^{1/2}.
	\end{align*}
	Therefore, 
	\begin{align*}
		\int | \cD_t \rho^*(\theta)| d\theta = \int |\psi_t(\theta)| \rho^*(\theta) d \theta \leq \Big( \int \psi^2_t(\theta) \rho^*(\theta) d\theta \Big)^{1/2} \leq  \Big( \int S^2_t(\theta) \rho^*(\theta) d\theta \Big)^{1/2}.
	\end{align*}
	With our assumptions on $\sigma$ and $Z_t$, it is direct to obtain
	\begin{align*}
		\sup_{\theta \in \R^d} |S_t(\theta)| \leq \frac{C_*}{\beta T} \int^T_t (|\cD_t X_s| + |\cD_t Y_s|) ds.
	\end{align*}
	Since $( \int S^2_t(\theta) \rho^*(\theta) d\theta)^{1/2} \leq \sup_{\theta \in \R^d} |S_t(\theta)|$, we have proved the desired result.
\end{proof}

\begin{proof}[{\bf Proof of Lemma \ref{lem:antic_Ito}}]	
	Let $t_i^n = \frac{i t}{2^n}$ for $i = 0, \dots, 2^n$ be a uniform partition of the interval $[0, t]$. By assumption, $Y_t$ is bounded, and $\sigma$ along with its derivatives are bounded. Consequently, $F$ and its spatial derivatives $F_x, F_y, F_{xx}, F_{xy}, F_{yy}$ are bounded random variables. By our previous derivations, $\rho^* \in \bD^{1,2}(L^2(\R^d))$. We apply Taylor's theorem up to the second order:
	\begin{equation}
		\begin{aligned}
			& F(\rho^*, Z_t) - F(\rho^*, Z_0) \\
			& \quad =  \sum_{i=0}^{2^n-1} \left( F(\rho^*, Z_{t_{i+1}^n}) - F(\rho^*, Z_{t_i^n}) \right) \\
			&\quad = \sum_{i=0}^{2^n-1} \left[ F_x(t_i^n) \Delta X_i^n + F_y(t_i^n) \Delta Y_i^n \right] \\
			& \qquad + \frac{1}{2} \sum_{i=0}^{2^n-1} \left[ F_{xx}(\overline{Z}_i^n) (\Delta X_i^n)^2 + 2 F_{xy}(\overline{Z}_i^n) \Delta X_i^n \Delta Y_i^n + F_{yy}(\overline{Z}_i^n) (\Delta Y_i^n)^2 \right],
		\end{aligned}
	\end{equation}
	where $\overline{Z}_i^n$ are random intermediate points, and $F_x(t_i^n)$ is shorthand for $F_x(\rho^*, Z_{t_i^n})$. 
	
	{\bf Step 1: The second-order terms}. As $n \to \infty$, the sum of the second-order terms converges in $L^1(\Omega)$ to the standard quadratic variation integrals. Because $Z_t$ is adapted, the presence of the fixed anticipative variable $\rho^*$ does not alter the classical cross-variation limits. Similar to Step 1 in \citet[Theorem 3.2.2]{nualart2006malliavin}, we have
	\begin{equation*}
		\frac{1}{2} \sum_{i=0}^{2^n-1} F_{xx}(\overline{Z}_i^n) (\Delta X_i^n)^2 \xrightarrow{L^1} \frac{1}{2} \int_0^t F_{xx}(\rho^*, Z_s) \Sigma_1(s) \Sigma_1^\top(s) ds.
	\end{equation*}
	Summing the analogous limits for the $xy$ and $yy$ terms yields exactly the second-order components of the operator $\mathcal{A}[F](\rho^*, Z_s) ds$.
	
	{\bf Step 2: The first-order drift terms}. The first-order increments driven by the drift terms $b_1 dt$ and $b_2 dt$ converge straightforwardly to their respective Lebesgue integrals:
	\begin{equation*}
		\sum_{i=0}^{2^n-1} \left[ F_x(t_i^n) \int_{t_i^n}^{t_{i+1}^n} b_1(s) ds + F_y(t_i^n) \int_{t_i^n}^{t_{i+1}^n} b_2(s) ds \right] \xrightarrow{L^1} \int_0^t \left( F_x(\rho^*, Z_s) b_1(s) + F_y(\rho^*, Z_s) b_2(s) \right) ds.
	\end{equation*}
	This provides the remaining first-order components of $\mathcal{A}[F]$.
	
	{\bf Step 3: The first-order diffusion terms}. Consider the $X$-diffusion term:
	\begin{equation*}
		\sum_{i=0}^{2^n-1} F^\top_x(\rho^*, Z_{t_i^n}) \int_{t_i^n}^{t_{i+1}^n} \Sigma_1(s) dW_s.
	\end{equation*}
	We invoke the integration-by-parts formula in \citet[Proposition 1.3.5]{nualart2006malliavin} or \citet[Lemma 4.9]{pronk2014tools} to obtain
	\begin{equation*}
		F^\top_x(t_i^n) \int_{t_i^n}^{t_{i+1}^n} \Sigma_1(s) dW_s = \int_{t_i^n}^{t_{i+1}^n} F^\top_x(t_i^n) \Sigma_1(s) \delta W_s + \int_{t_i^n}^{t_{i+1}^n} \tr\Big[ \cD_s [F^\top_x(\rho^*, Z_{t_i^n})] \Sigma_1(s) \Big] ds.
	\end{equation*}
	We evaluate $\cD_s [F_x(\rho^*, Z_{t_i^n})]$ for $s \in (t_i^n, t_{i+1}^n]$. Because $Z_t$ is adapted, its Malliavin derivative $\cD_s Z_{t_i^n} = 0$ for $s > t_i^n$. Therefore,
	\begin{equation*}
		\cD_s [F_x(\rho^*, Z_{t_i^n})] = \partial_{\rho} F_x(\rho^*, Z_{t_i^n}) [\cD_s \rho^*].
	\end{equation*}
	Note that 
	\begin{align*}
		& \partial_\rho F_x(\rho^*, Z_s)[\cD_s \rho^*] \\
		& = 2 \ang{\cD_s \rho^*, \sigma(X_s, \cdot)} \ang{\rho^*, \sigma_x(X_s, \cdot)} + 2 \ang{\rho^*, \sigma(X_s, \cdot)} \ang{\cD_s \rho^*, \sigma_x(X_s, \cdot)} - 2 Y_s \ang{\cD_s \rho^*, \sigma_x(X_s, \cdot)}.
	\end{align*}

	Because $\cD_s \rho^* = (\cD^- \rho^*)_s$, we can prove the following limit similarly like \citet[Equation (3.32)]{nualart2006malliavin}. As $n$ tends to infinity,
	\begin{align*}
		\sum_{i=0}^{2^n-1} \int_{t_i^n}^{t_{i+1}^n} \cD_s [F_x(\rho^*, Z_{t_i^n})] \Sigma_1(s) ds \xrightarrow{L^1} \int_0^t \partial_{\rho} F_x(\rho^*, Z_s) [\cD_s \rho^*] \Sigma_1(s) ds.
	\end{align*} 
	Summing the analogous limits for the $\Sigma_2$ term obtains the components of $\cM_t$.
	
	{\bf Step 4: Well-posedness of Skorohod integrals}. Like Step 4 in the proof of \citet[Theorem 3.2.2]{nualart2006malliavin}, we prove the integrand $F^\top_x(\rho^*, Z_t) \Sigma_1(t)$ belongs to the domain of the divergence operator, such that the Skorohod integral $\int^t_0 F^\top_x \Sigma_1(u) \delta W_u$ is well-defined. We need to take the Malliavin derivative of $F^\top_x(\rho^*, Z_t) \Sigma_1(t)$ and show it is square integrable.
	
	By the chain rule, we have
	\begin{align*}
		\cD_s[F^\top_x(\rho^*, Z_t) \Sigma_1(t)] = & \Big( F_{xx}(\rho^*, Z_t) \cD_s X_t + F_{xy}(\rho^*, Z_t) \cD_s Y_t + 2 \ang{\cD_s \rho^*, \sigma(X_t, \cdot)} \ang{\rho^*, \sigma_x(X_t, \cdot)} \\
		& \quad + 2 \ang{\rho^*, \sigma(X_t, \cdot)} \ang{\cD_s \rho^*, \sigma_x(X_t, \cdot)} - 2 Y_t \ang{\cD_s \rho^*, \sigma_x(X_t, \cdot)} \Big)^\top \Sigma_1(t) \\
		& + F^\top_x(\rho^*, Z_t) \cD_s \Sigma_1(t). 
	\end{align*}
	Thanks to the boundedness of $\sigma$ and its derivatives and the fact that $Y_t$ is bounded, we have
	\begin{align*}
		& \E \Big[\int^T_0 \int^T_0 |\cD_s[F^\top_x(\rho^*, Z_t) \Sigma_1(t)]|^2 dsdt \Big] \\
		& \leq C \E \Big[\int^T_0 \int^T_0 |\cD_s X_t |^4 dsdt \Big] + C \E \Big[\int^T_0 \int^T_0 |\cD_s Y_t |^4 dsdt \Big] + C \E \Big[\int^T_0 \int^T_0 |\Sigma_1(t) |^4 dsdt \Big] \\
		& \quad + C \E \Big[\int^T_0 \int^T_0 |\cD_s \Sigma_1(t) |^4 dsdt \Big] + C \E \Big[\int^T_0 \int^T_0 |\cD_s \rho^* |^4_{L^1} dsdt \Big].
	\end{align*}
	By Lemma \ref{lem:psi}, the last term here satisfies
	\begin{align*}
		\E \Big[\int^T_0 \int^T_0 |\cD_s \rho^* |^4_{L^1} dsdt \Big] \leq C \E \Big[ \int^T_0 |\cD_s \rho^* |^4_{L^1} ds \Big] \leq \frac{C}{\beta^4} \E \Big[ \int^T_0 \int^T_0 (|\cD_s X_t |^4 + |\cD_s Y_t|^4) ds dt \Big]. 
	\end{align*}
	We further bound these terms with conditions on $b_i$ and $\Sigma_i$, $i=1, 2$. Since
	\begin{equation*}
		\cD_s X_t = \cD_s X_0 + \Sigma_1(s) \mathbf{1}_{[0, t]}(s) + \int^t_0 \cD_s b_1(r) dr + \int^t_0 \cD_s \Sigma_1(r) d W_r,
	\end{equation*}
	we apply the Burkholder-Davis-Gundy inequality to obtain
	\begin{align*}
		\E \Big[\int^T_0 \int^T_0 |\cD_s X_t |^4 dsdt \Big]  & \leq C \E \Big[\int^T_0 \int^T_0 |\cD_s b_1(r) |^4 drds \Big]  + C \E \Big[\int^T_0 \int^T_0 |\cD_s \Sigma_1(r) |^4 drds \Big]  \\
		& \quad + C \E \Big[\int^T_0 |\Sigma_1(s) |^4 ds \Big].
	\end{align*}
	The $\cD_sY_t$ term and $\cD_s[F^\top_y(\rho^*, Z_t) \Sigma_2(t)]$ counterpart can be bounded similarly. 
	
	{\bf Step 5: The limit of Skorohod integrals}. Because the Skorohod integral is not continuous with respect to $L^2$ or $L^1$ norms, we cannot claim that the limit of a Riemann sum of Skorohod integrals equals the Skorohod integral of the limit directly. To bridge this gap, we use the same duality argument in Step 5 of \citet[Theorem 3.2.2]{nualart2006malliavin}. Then a direct application of \citet[Proposition 1.3.6]{nualart2006malliavin} shows that
	\begin{equation*}
		\sum_{i=0}^{2^n-1} \int_{t_i^n}^{t_{i+1}^n} F^\top_x(\rho^*, Z_{t_i^n}) \Sigma_1(s) \delta W_s \xrightarrow{L^1} \int_0^t F^\top_x(\rho^*, Z_s) \Sigma_1(s) \delta W_s, \quad n \to \infty.
	\end{equation*}
	The analogous derivation for the $Y$-diffusion term yields $\int_0^t F_y(\rho^*, Z_s) \Sigma_2(s) \delta W_s$.
\end{proof}

\begin{proof}[{\bf Proof of Theorem \ref{thm:Mallivian-regret}}]
	Denote
	\begin{equation}
		P(t) := \E[F(\rho_t, Z_t) - F(\rho^*, Z_t)].
	\end{equation}
	Crucially, the Skorohod integral satisfies 
	\begin{equation*}
		\E\Big[ \int^t_0 F^\top_x(\rho^*, Z_u) \Sigma_1(u) \delta W_u \Big] = \E \Big[ \int^t_0 F_y(\rho^*, Z_u) \Sigma_2(u) \delta W_u \Big] = 0.
	\end{equation*}
	Then
	\begin{equation}\label{eq:dP}
		\begin{aligned}
			d P(t) = & - \E \Big[ \Big\| \nabla \Big( \frac{\delta F(\rho_t, Z_t)}{\delta \rho} \Big) \Big\|_{L^2(\rho_t)}^2 \Big] dt + \E[\cA[F] (\rho_t, Z_t) - \cA[F] (\rho^*, Z_t)] dt - \E[\cM_t] dt.
		\end{aligned}
	\end{equation}
	For the first term, we rely on the optimizer $\mu^*_t$ at time $t$ and Lemma \ref{lem:PL} to obtain
	\begin{align*}
		- \E \Big[ \Big\|\nabla \left( \frac{\delta F(\rho_t, Z_t)}{\delta \rho} \right) \Big\|_{L^2(\rho_t)}^2 \Big] \leq & \frac{\E[F(\mu^*_t, Z_t) - F(\rho_t, Z_t)]}{C_{PL}} \leq \frac{\E[F(\rho^*, Z_t) - F(\rho_t, Z_t)]}{C_{PL}} = - \frac{P(t)}{C_{PL}}.
	\end{align*}
	We bound the second and third terms in a direct way. Thanks to Assumption \ref{A:sigma} and \ref{A:data}, 
	\begin{equation*}
		\E[|\cA[F] (\rho_t, Z_t)|] \leq C_* C_{b, \Sigma}, \quad \E[|\cA[F] (\rho^*, Z_t)|] \leq C_* C_{b, \Sigma}.
	\end{equation*}
	
	By Lemma \ref{lem:psi}, we obtain the following bound:
	\begin{equation}
		|\cM_t| \leq  \left( |\Sigma_1(t)| + |\Sigma_2(t)| \right) \left( \frac{C_*}{\beta T} \int_t^T (|\cD_t X_s| + |\cD_t Y_s|) ds \right).
	\end{equation}
	By Cauchy-Schwarz inequality, 
	\begin{align*}
		\E[|\cM_t|] &\le \frac{C_* C_{b, \Sigma}}{\beta T} \E \left[ \left( \int_t^T (|\cD_t X_s| + |\cD_t Y_s|) ds \right)^2 \right]^{1/2} \leq \frac{C_* C_{b, \Sigma} C_{a}}{\beta}.
	\end{align*}
	
	Combining together, we obtain
	\begin{align*}
		d P(t) & \leq - \frac{P(t)}{C_{PL}} dt + C_* C_{b, \Sigma} \Big(1 + \frac{C_a}{\beta} \Big)dt.
	\end{align*}
	For notational simplicity, we introduce $\tilde{C} = C_* C_{b, \Sigma} (1 + C_a/\beta )$. By Gr\"onwall's inequality, 
	\begin{align*}
		P(s) \leq (P(0) - \tilde{C} C_{PL}) e^{-s/C_{PL}} + \tilde{C} C_{PL}.
	\end{align*}
	Therefore,
	\begin{equation}
		\int^T_0 P(s) ds \leq (P(0) - \tilde{C} C_{PL}) C_{PL} (1 - e^{-T/C_{PL}}) + \tilde{C} C_{PL} T.
	\end{equation}
\end{proof}

\section{Proofs of Section \ref{sec:particle-regret}}

\subsection{Marginal-law dynamic regret}
\begin{proof}[{\bf Proof of Lemma \ref{lem:PoC}}]
	We verify \citet[Theorem 2.1 (1)]{lacker2023sharp} can be used here. For Assumption A.1 in \cite{lacker2023sharp}, our constant $1/(2\eta)$ is their constant $\eta$ in the LSI. For Assumption A.2, we note $b(t, \theta, \cdot) = -2 \nabla \sigma(X_t, \theta) \sigma(X_t, \cdot)$.  By Pinsker's inequality,
	\begin{align*}
		|\ang{\nu - \rho_t, -2 \nabla \sigma(X_t, \theta) \sigma(X_t, \cdot)}|  \leq 4 C_\sigma C_1 \| \nu - \rho_t \|_{TV} \leq 4 C_\sigma C_1 \sqrt{\frac{1}{2} D_{KL}(\nu \| \rho_t)}.
	\end{align*}
	We have $\gamma = 8 C^2_\sigma C^2_1$ in Assumption A.2 of \cite{lacker2023sharp}. Assumption A.3 is satisfied since our $b(t, \theta, \cdot)$ is bounded by a constant depending on $C_\sigma$ and $C_1$. Hence, to ensure $r_c > 1$ in \cite{lacker2023sharp}, we only need $\frac{4 \beta^2}{4 \gamma/(2\eta)} > 2$. It leads to our assumption $\eta \beta^2 > 8 C^2_\sigma C^2_1$.
\end{proof}

\begin{lemma}\label{lem:W2}
	Consider probability measures $\mu$ and $\nu$ with finite second moments. Assume $f$ is continuously differentiable and its gradient grows linearly: $|\nabla f(\theta)| \le L(1 + |\theta|)$. Then 
	\begin{equation*}
		\left| \int f d\mu - \int f d\nu \right| \leq L \cW_2(\mu, \nu) \left( 1 + \| \theta \|_{L^2(\mu)} + \| \theta \|_{L^2(\nu)} \right).
	\end{equation*}
\end{lemma}
\begin{proof}
	Let $\pi$ be the optimal coupling between $\mu$ and $\nu$ for the $\cW_2$ distance, such that $(X, Y) \sim \pi$. We write the difference as an expectation over the coupling:
	\begin{equation*}
		\int f d\mu - \int f d\nu = \E_{\pi}[f(X) - f(Y)].
	\end{equation*}
	By the mean value theorem, there exists a point $\xi_{X,Y}$ on the segment between $X$ and $Y$ such that
	\begin{equation*}
		f(X) - f(Y) = \nabla f(\xi_{X,Y}) \cdot (X - Y).
	\end{equation*}
	The growth condition on the gradient ensures that
	\begin{equation*}
		|f(X) - f(Y)| \le L(1 + |\xi_{X,Y}|) |X - Y|.
	\end{equation*}
	Since $\xi_{X,Y}$ is between $X$ and $Y$, $|\xi_{X,Y}| \le |X| + |Y|$. We apply Cauchy-Schwarz inequality to obtain
	\begin{align*}
		\E_\pi[|f(X) - f(Y)|] & \le L \E_\pi[(1 + |X| + |Y|)  |X - Y|] \\
		& \le L \sqrt{\E_\pi[(1 + |X| + |Y|)^2]} \sqrt{\E_\pi[|X - Y|^2]} \\
		& \leq L \cW_2(\mu, \nu) \left( 1 + \| \theta \|_{L^2(\mu)} + \| \theta \|_{L^2(\nu)} \right).
	\end{align*}
	In the last inequality, we use the optimality of $\pi$ to obtain $\sqrt{\E_\pi[|X - Y|^2]} = \cW_2(\mu, \nu)$.
\end{proof}

\begin{proof}[{\bf Proof of Theorem \ref{thm:margin-regret}}]
	The well-posedness results in Proposition \ref{prop:MkV} and \ref{prop:particle}  have obtained the finite second moment condition on $\rho_t$ and $\rho^{1, N}_t$, respectively. The system of $N$ particles is exchangeable. Therefore, the marginal distributions of all particles are identical and we simply use $\rho^{1,N}_t$. The difference $|F(\rho^{1, N}_t, Z_t) - F(\rho_t, Z_t)|$ satisfies
	\begin{align*}
		|F(\rho^{1, N}_t, Z_t) - F(\rho_t, Z_t)| \leq &  | \ang{\rho_t, \sigma(X_t, \cdot)}^2 - \ang{\rho^{1, N}_t, \sigma(X_t, \cdot)}^2 | \\
		& + 2 |Y_t| | \ang{\rho_t, \sigma(X_t, \cdot)} - \ang{\rho^{1, N}_t, \sigma(X_t, \cdot)} | \\
		& + \frac{\lambda}{2} \big| \ang{\rho_t, |\theta|^2} -\ang{\rho^{1, N}_t, |\theta|^2} \big| \\
		& + \beta \Big| \int \rho_t(\theta) \log \rho_t(\theta) d\theta - \int \rho^{1, N}_t(\theta) \log \rho^{1, N}_t(\theta) d\theta  \Big|.
	\end{align*}
	We bound each term separately. First, since $\sigma$ is bounded and Lipschitz in $\theta$,
	\begin{align*}
		| \ang{\rho_t, \sigma(X_t, \cdot)}^2 - \ang{\rho^{1, N}_t, \sigma(X_t, \cdot)}^2 | \leq 2 C_\sigma C_1 \cW_1(\rho^{1, N}_t, \rho_t).
	\end{align*}
	Similarly,
	\begin{align*}
		2 |Y_t| | \ang{\rho_t, \sigma(X_t, \cdot)} - \ang{\rho^{1, N}_t, \sigma(X_t, \cdot)} |  \leq 2 C_z C_1 \cW_1(\rho^{1, N}_t, \rho_t).
	\end{align*}
	Since the gradient of $|\theta|^2$ has a linear growth, Lemma \ref{lem:W2} shows that
	\begin{align*}
		\frac{\lambda}{2} \big| \ang{\rho_t, |\theta|^2} -\ang{\rho^{1, N}_t, |\theta|^2} \big|  & \leq \lambda \cW_2(\rho^{1, N}_t, \rho_t) \left( 1 + \| \theta \|_{L^2(\rho^{1, N}_t)} + \| \theta \|_{L^2(\rho_t)} \right) \\
		& \leq \lambda (1 + 2\sqrt{M_2(\beta, \lambda)})\cW_2(\rho^{1, N}_t, \rho_t).
	\end{align*}
	For the difference between entropy functional, 
	\begin{align*}
		& \beta \Big| \int \rho_t(\theta) \log \rho_t(\theta) d\theta - \int \rho^{1, N}_t(\theta) \log \rho^{1, N}_t(\theta) d\theta  \Big| \\
		& = \beta \Big| \int \rho^{1, N}_t(\theta) \log \rho^{1, N}_t(\theta) d\theta - \int \rho^{1, N}_t(\theta) \log \rho_t(\theta) d\theta + \int \rho^{1, N}_t(\theta) \log \rho_t(\theta) d\theta - \int \rho_t(\theta) \log \rho_t(\theta) d\theta\Big| \\
		& \leq \beta D_{KL}(\rho^{1, N}_t \parallel \rho_t) +    \beta \Big| \int \rho^{1, N}_t(\theta) \log \rho_t(\theta) d\theta - \int \rho_t(\theta) \log \rho_t(\theta) d\theta\Big|. 
	\end{align*}
	Theorem \ref{thm:Lip-grad} (1) ensures that we can apply Lemma \ref{lem:W2} to $\log(\rho_t)$:
	\begin{align*}
		\Big| \int \rho^{1, N}_t(\theta) \log \rho_t(\theta) d\theta - \int \rho_t(\theta) \log \rho_t(\theta) d\theta\Big| & \leq L_\rho \cW_2(\rho^{1, N}_t, \rho_t) \left( 1 + \| \theta \|_{L^2(\rho^{1, N}_t)} + \| \theta \|_{L^2(\rho_t)} \right) \\
		& \leq L_{\rho} (1 + 2 \sqrt{M_2(\beta, \lambda)}) \cW_2(\rho^{1, N}_t, \rho_t).
	\end{align*}
	By Lemma \ref{lem:rhotLSI}, $\rho_t$ satisfies the $LSI(\eta)$. \citet[Corollary 2.1]{conforti2020around} extends \citet[Theorem 1]{otto2000generalization} to the case when $- \log \rho_t$ does not necessarily have classical second derivatives. Our case satisfies their Assumption (ii). Hence, $\rho_t$ satisfies Talagrand inequality with constant $\eta$, that is,
	\begin{equation}
		\cW_2(\rho^{1, N}_t, \rho_t) \leq \sqrt{ \frac{2 D_{KL}(\rho^{1, N}_t \parallel \rho_t) }{\eta} }. 
	\end{equation}
	Since $\cW_1 \leq \cW_2$, summing all terms up, we obtain
	\begin{align*}
		|F(\rho^{1, N}_t, Z_t) - F(\rho_t, Z_t)| \leq & \big[ 2(C_\sigma + C_z) C_1 + (\lambda + \beta L_\rho)(1 + 2 \sqrt{M_2(\beta, \lambda)}) \big] \sqrt{ \frac{2 D_{KL}(\rho^{1, N}_t \parallel \rho_t) }{\eta} } \\
		& + \beta D_{KL}(\rho^{1, N}_t \parallel \rho_t).
	\end{align*}
	Using Lemma \ref{lem:PoC}, we conclude the proof. The last claim is from the triangle inequality.
\end{proof}

\subsection{Empirical dynamic regret without entropy}

\begin{proof}[{\bf Proof of Lemma \ref{lem:mu*_2m}}]
	We define the nonlinear interaction term in the exponent of $\mu^*_t$ by
	\begin{equation*}
		g_t(\theta) := - \frac{2}{\beta} \sigma(X_t, \theta) \left( \int_{\R^d} \sigma(X_t, \vartheta)\mu^*_t(d\vartheta) - Y_t \right).
	\end{equation*}
	Clearly, $g_t(\theta)$ is uniformly bounded by
	\begin{equation}\label{eq:gt_bound}
		|g_t(\theta)| \le \frac{2}{\beta} C_\sigma (C_\sigma + C_z) =: M.
	\end{equation}
	Let $\mu_0$ denote the unperturbed Gaussian probability measure:
	\begin{equation*}
		\mu_0(\theta) := \frac{1}{Z_0} \exp\left( - \frac{\lambda}{2 \beta} |\theta|^2 \right), \quad \text{where} \quad Z_0 := \left( \frac{2 \pi \beta}{\lambda} \right)^{d/2}.
	\end{equation*}
	We can lower-bound the normalization constant $C_t$ of the equilibrium measure as follows:
	\begin{align*}
		C_t &= \int_{\R^d} \exp\left( - \frac{\lambda}{2 \beta} |\theta|^2 + g_t(\theta) \right) d\theta \ge \int_{\R^d} \exp\left( - \frac{\lambda}{2 \beta} |\theta|^2 - M \right) d\theta = e^{-M} Z_0.
	\end{align*}
	Then a pointwise upper bound on $\mu^*_t(\theta)$ can be obtained as
	\begin{align*}
		\mu^*_t(\theta) &= \frac{1}{C_t} \exp\left( - \frac{\lambda}{2 \beta} |\theta|^2 + g_t(\theta) \right) \le \frac{1}{e^{-M} Z_0} \exp\left( - \frac{\lambda}{2 \beta} |\theta|^2 + M \right) = e^{2M} \mu_0(\theta).
	\end{align*}
	Therefore,
	\begin{equation*}
		\int_{\R^d} |\theta|^2 \mu^*_t(d\theta) \le \int_{\R^d} |\theta|^2 \left( e^{2M} \mu_0(\theta) \right) d\theta = e^{2M} \int_{\R^d} |\theta|^2 \mu_0(d\theta) = \frac{\beta d }{\lambda} e^{2M},
	\end{equation*}
	which completes the proof.
\end{proof}

Next, we introduce another bound of the difference between second moments.
\begin{lemma}\label{lem:W2_diff}
	Suppose both $\rho$ and $\nu$ have finite second moment. Then
	\begin{equation}
		\Big|\langle{\rho,|\theta|^2}\rangle-\langle{\nu, |\theta|^2}\rangle\Big|
		\le 
		\cW_2(\rho,\nu)^2
		+2\sqrt{\langle{\nu, |\theta|^2}\rangle} \cW_2(\rho,\nu).
	\end{equation}
\end{lemma}

\begin{proof}
	Let $\pi$ be an optimal coupling for the Wasserstein distance $W_2(\rho, \nu)$, defined on the product space with marginals $\theta \sim \rho$ and $\theta' \sim \nu$. We have
	\begin{align*}
		\langle{\rho, |\theta|^2}\rangle - \langle{\nu, |\theta|^2}\rangle 
		&= \int \big( |\theta |^2 - |\theta'|^2 \big) d\pi(\theta, \theta') \\
		&= \int \big( |\theta - \theta'|^2 + 2\theta' \cdot (\theta - \theta') \big) d\pi(\theta, \theta').
	\end{align*}
	We bound the cross-term by the Cauchy--Schwarz inequality. Then
	\begin{align*}
		\Big|\langle{\rho, |\theta|^2}\rangle - \langle{\nu, |\theta|^2}\rangle\Big| 
		&\le \int |\theta - \theta'|^2 d\pi + 2 \left( \int |\theta'|^2 d\pi \right)^{1/2} \left( \int |\theta - \theta'|^2 d\pi \right)^{1/2}.
	\end{align*}
	We have $\int |\theta'|^2 d\pi = \langle \nu, |\theta|^2 \rangle$. By the optimality of $\pi$, we have $\int |\theta - \theta'|^2 d\pi = \cW_2(\rho, \nu)^2$. Then the claim follows.
\end{proof}

\begin{proof}[{\bf Proof of Lemma \ref{lem:U_regret}}]
	Denote $\Delta_F(\rho_t, t) := F(\rho_t, Z_t) - F(\mu_t^*,Z_t)$. We first prove that
	\begin{equation}\label{eq:U_gap}
		\begin{aligned}
			& U(\rho_t, Z_t) - U(\mu_t^*, Z_t) \\
			& \quad \leq
			\frac{2C_1^2 + \lambda}{\alpha\beta}\Delta_F(\rho_t, t)
			+ \left( 2C_1(C_\sigma+C_z)+\lambda\sqrt{Q_*(\beta, \lambda)} \right) \sqrt{\frac{2}{\alpha\beta}}\sqrt{\Delta_F(\rho_t, t)}.
		\end{aligned}
	\end{equation}
	
	Indeed, it is direct to show that
	\begin{align*}
		U(\rho_t, Z_t) - U(\mu^*_t, Z_t) = &  \Big( \ang{\rho_t, \sigma(X_t, \cdot)} - \ang{\mu^*_t, \sigma(X_t, \cdot)} \Big)^2 \\
		& + 2 \Big( \ang{\mu^*_t, \sigma(X_t, \cdot)} - Y_t \Big) \Big( \ang{\rho_t, \sigma(X_t, \cdot)} - \ang{\mu^*_t, \sigma(X_t, \cdot)} \Big) \\
		& + \frac{\lambda}{2} \big( \ang{\rho_t, |\theta|^2} -\ang{\mu^*_t, |\theta|^2} \big).
	\end{align*}
	By the Kantorovich duality, the first term satisfies
	\begin{align*}
		\Big( \ang{\rho_t, \sigma(X_t, \cdot)} - \ang{\mu^*_t, \sigma(X_t, \cdot)} \Big)^2  \leq C^2_1 \cW_1(\rho_t, \mu^*_t)^2 \leq  C^2_1 \cW_2(\rho_t, \mu^*_t)^2.
	\end{align*}
	The second term is bounded as follows:
	\begin{align*}
		2 \Big| \ang{\mu^*_t, \sigma(X_t, \cdot)} - Y_t \Big| \Big| \ang{\rho_t, \sigma(X_t, \cdot)} - \ang{\mu^*_t, \sigma(X_t, \cdot)} \Big| \leq 2 (C_\sigma + C_z) C_1 \cW_2(\rho_t, \mu^*_t).
	\end{align*}
	For the third term, Lemma \ref{lem:mu*_2m} and \ref{lem:W2_diff} yield
	\begin{equation*}
		\frac{\lambda}{2} \big| \ang{\rho_t, |\theta|^2} -\ang{\mu^*_t, |\theta|^2} \big|
		\le \frac{\lambda}{2} \cW_2(\rho_t, \mu_t^*)^2 + \lambda \sqrt{Q_*(\beta, \lambda)} \cW_2(\rho_t, \mu_t^*).
	\end{equation*}
	Combining them together,
	\begin{equation*}
		U(\rho_t, Z_t) - U(\mu_t^*, Z_t) \le \left( C_1^2 + \frac{\lambda}{2} \right) \cW_2(\rho_t, \mu_t^*)^2 + \left( 2C_1(C_\sigma + C_z) + \lambda \sqrt{Q_*(\beta, \lambda)} \right) \cW_2(\rho_t, \mu_t^*).
	\end{equation*}
	Since \eqref{eq:Wass_F} showed that $$\cW_2(\rho_t, \mu_t^*) \le \sqrt{\frac{2}{\alpha\beta}} \sqrt{\Delta_F(\rho_t, t)},$$ 
	we have proved \eqref{eq:U_gap}.
	
	Next, we define $$E(t) := \E\big[\Delta_F(\rho_t,Z_t)\big] = \E\big[F(\rho_t, Z_t) - F(\mu^*_t, Z_t)\big] \geq 0.$$ 
	By the Cauchy-Schwarz inequality,
	\begin{align*}
		\E\Big[\sqrt{\Delta_F(\rho_t,Z_t)} \Big] & \le \sqrt{\E[\Delta_F(\rho_t,Z_t)]}=\sqrt{E(t)}, \\
		\E\Big[\sqrt{\ang{\mu_t^*, |\theta|^2}} \sqrt{\Delta_F(\rho_t,Z_t)}\Big]
		& \le \sqrt{\E[\ang{\mu_t^*, |\theta|^2}]}\sqrt{\E \Delta_F(\rho_t, Z_t)} \le \sqrt{Q_*(\beta, \lambda)} \sqrt{E(t)}.
	\end{align*}
	Taking the expectation in \eqref{eq:U_gap} and integrating over $[0, T]$ yield
 	\begin{align}
		\mathbb E\Big[\int_0^T\big(U(\rho_t,Z_t)-U(\mu_t^*,Z_t)\big)\,dt\Big]
		&\le \frac{2C_1^2 + \lambda}{\alpha\beta} \int_0^T E(t)\,dt \nonumber\\
		&\quad+ \left( 2C_1(C_\sigma+C_z)+\lambda\sqrt{Q_*(\beta, \lambda)} \right) \sqrt{\frac{2}{\alpha\beta}} \int_0^T \sqrt{E(t)}\,dt.
	\end{align}
	By the Cauchy-Schwarz inequality, $\int_0^T \sqrt{E(t)}\,dt \le  \sqrt{T \int_0^T E(t)\,dt}$. We use the definition of the dynamic regret $\cR_D(T)$ to obtain the claim.
\end{proof}

\begin{proof}[{\bf Proof of Lemma \ref{lem:UdiffmfN}}]
	By Lemma \ref{lem:rhotLSI}, the mean-field law $\rho_t$ satisfies the $LSI(\eta)$. By Talagrand's inequality and Lemma \ref{lem:PoC} on the propagation of chaos, the Wasserstein distance for the $k$-particle marginals satisfies
	$$ \cW_2(P^k_t, \rho_t^{\otimes k}) \le \sqrt{\frac{2}{\eta} D_{KL}(P^k_t \parallel \rho_t^{\otimes k})} \le \sqrt{\frac{2 C_{poc}}{\eta}} \frac{k}{N}. $$
	
	We decompose the conditional expected cost difference into three terms:
	\begin{equation}\label{eq:U3}
	\begin{aligned}
		\E_{B}\big[U(\hat{\rho}^N_t) - U(\rho_t)\big] &= \Big( \E_{B}[\ang{\hat{\rho}^N_t, \sigma(X_t, \cdot)}^2] - \ang{\rho_t, \sigma(X_t, \cdot)}^2 \Big) \\
		&\quad - 2Y_t \Big( \E_{B}[\ang{ \hat{\rho}^N_t, \sigma(X_t, \cdot)}] - \ang{\rho_t, \sigma(X_t, \cdot)} \Big) \\
		&\quad + \frac{\lambda}{2} \Big( \E_{B}[\ang{\hat{\rho}^N_t, |\theta|^2}] - \ang{\rho_t, |\theta|^2} \Big).
	\end{aligned}
	\end{equation}
	Here, we can pull $\ang{\rho_t, \sigma(X_t, \cdot)}$ and $Y_t$ out of the expectations because $Z_t$ is fixed under $\E_{B}$.
	
	Consider the first term in \eqref{eq:U3}. Expanding the square and utilizing the exchangeability yield
	\begin{equation*}
		\E_{B}[\ang{\hat{\rho}^N_t, \sigma(X_t, \cdot)}^2] = \frac{1}{N} \int \sigma^2(X_t, \theta) P_t^1(d\theta) + \Big(1 - \frac{1}{N}\Big) \int \sigma(X_t, \theta_1)\sigma(X_t, \theta_2) P_t^2(d\theta_1, d\theta_2).
	\end{equation*}
	The mean-field term satisfies 
	$$\ang{\rho_t, \sigma(X_t, \cdot)}^2 = \int \sigma(X_t, \theta_1)\sigma(X_t, \theta_2) \rho_t^{\otimes 2}(d\theta_1, d\theta_2) =: \int \sigma \otimes \sigma d\rho_t^{\otimes 2}.$$
	Then
	\begin{align*}
		& \E_{B}[\ang{\hat{\rho}^N_t, \sigma(X_t, \cdot)}^2] - \ang{\rho_t, \sigma(X_t, \cdot)}^2 \\
		& = \frac{1}{N} \Big( \int \sigma^2(X_t, \theta) P_t^1(d\theta) - \int \sigma(X_t, \theta_1)\sigma(X_t, \theta_2) \rho_t^{\otimes 2}(d\theta_1, d\theta_2) \Big) \\
		&\quad + \Big(1 - \frac{1}{N}\Big) \Big( \int \sigma \otimes \sigma dP_t^2 - \int \sigma \otimes \sigma d\rho_t^{\otimes 2} \Big).
	\end{align*}
	 The first element is bounded by $2C_\sigma^2 / N$. For the second element,  $\sigma(X_t, \theta_1) \sigma(X_t, \theta_2)$ has a Lipschitz constant bounded by $\sqrt{2} C_\sigma C_1$ on $\mathbb R^{2d}$ equipped with the $L_2$ norm. Applying the duality gives
	$$ \Big| \int \sigma \otimes \sigma dP_t^2 - \int \sigma \otimes \sigma d\rho_t^{\otimes 2} \Big| \le \sqrt{2} C_\sigma C_1 \cW_2(P_t^2, \rho_t^{\otimes 2}) \le \sqrt{2} C_\sigma C_1 \frac{2}{N} \sqrt{\frac{2 C_{poc}}{\eta}}. $$
	Combining these bounds yields
	$$ \Big| \E_{B}[\ang{\hat{\rho}^N_t, \sigma(X_t, \cdot)}^2] - \ang{\rho_t, \sigma(X_t, \cdot)}^2 \Big| \le \frac{2C_\sigma^2}{N} + \frac{2\sqrt{2} C_\sigma C_1}{N} \sqrt{\frac{2 C_{poc}}{\eta}}. $$
	
	 Consider the second term in \eqref{eq:U3}. By exchangeability of $N$ particles, $$\E_{B}[\ang{\hat{\rho}^N_t, \sigma(X_t, \cdot)}] = \int \sigma(X_t, \theta) P_t^1(d\theta).$$ 
	Then
	$$ 2|Y_t| \Big| \int \sigma(X_t, \theta) P_t^1(d\theta) - \int \sigma(X_t, \theta) \rho_t(d\theta) \Big| \le 2C_z C_1 \cW_1(P_t^1, \rho_t) \le 2C_z C_1 \frac{1}{N} \sqrt{\frac{2 C_{poc}}{\eta}}. $$
	
	For the last term in \eqref{eq:U3}, let $\pi$ be an optimal $\cW_2$-coupling between $P_t^1$ and $\rho_t$. Using the identity $|\theta_1|^2 - |\theta_2|^2 = (\theta_1 - \theta_2) \cdot (\theta_1 + \theta_2)$ and the Cauchy-Schwarz inequality, we have
	$$ \Big| \int |\theta|^2 P_t^1(d\theta) - \int |\theta|^2 \rho_t(d\theta) \Big| \le \|\theta_1 - \theta_2\|_{L^2(\pi)} \|\theta_1 + \theta_2\|_{L^2(\pi)}. $$
	With the uniform bound $M_2(\beta, \lambda)$, we obtain $$\|\theta_1 + \theta_2\|_{L^2(\pi)} \le 2\sqrt{M_2(\beta, \lambda)}.$$ 
	Thus,
	$$ \frac{\lambda}{2} \Big| \int |\theta|^2 P_t^1(d\theta) - \int |\theta|^2 \rho_t(d\theta) \Big| \le \frac{\lambda}{2} \Big( 2\sqrt{M_2(\beta, \lambda)} \cW_2(P_t^1, \rho_t) \Big) \le \lambda \sqrt{M_2(\beta, \lambda)} \frac{1}{N}\sqrt{\frac{2 C_{poc}}{\eta}}. $$
	
	Summing the bounds and integrating over $[0, T]$ yield the constant $C_U(\beta, \lambda, \eta)$ and concludes the proof.
\end{proof}

\end{document}